\documentclass{article}

\usepackage{authblk}

\usepackage[top=30truemm,bottom=30truemm,left=25truemm,right=25truemm]{geometry}
\usepackage[square,numbers]{natbib}

\usepackage{amssymb}
\usepackage{amsmath}
\usepackage{amsthm}

\usepackage{amsfonts}
\usepackage{bm}

\usepackage{stmaryrd}

\usepackage[subrefformat=parens]{subcaption}

\usepackage{algorithm}
\usepackage[noend]{algpseudocode}
\algdef{SE}[DOWHILE]{Do}{DoWhile}{\algorithmicdo}[1]{\algorithmicwhile\ #1}%

\usepackage{booktabs}

\usepackage{hyperref}
\usepackage{url}

\usepackage{multirow}

\usepackage{cancel}
\allowdisplaybreaks

\usepackage{algorithm}
\usepackage{algpseudocode}

\usepackage{graphicx}

\newtheorem{thm}{Theorem}
\newtheorem{proposition}[thm]{Proposition}

\newtheorem{dfn}{Definition}

\title{Residual Prior Diffusion: A Probabilistic Framework Integrating Coarse Latent Priors with Diffusion Models}
\author[$\ast$]{Takuro Kutsuna}
\affil[$\ast$]{\normalsize Toyota Central R\&D Labs., Inc.}
\date{}

\begin{document}
\maketitle
\begin{abstract}
    Diffusion models have become a central tool in deep generative modeling, but standard formulations rely on a single network and a single diffusion schedule to transform a simple prior, typically a standard normal distribution, into the target data distribution. As a result, the model must simultaneously represent the global structure of the distribution and its fine-scale local variations, which becomes difficult when these scales are strongly mismatched. This issue arises both in natural images, where coarse manifold-level structure and fine textures coexist, and in low-dimensional distributions with highly concentrated local structure.
    To address this issue, we propose Residual Prior Diffusion (RPD), a two-stage framework in which a coarse prior model first captures the large-scale structure of the data distribution, and a diffusion model is then trained to represent the residual between the prior and the target data distribution. We formulate RPD as an explicit probabilistic model with a tractable evidence lower bound, whose optimization reduces to the familiar objectives of noise prediction or velocity prediction.
    We further introduce auxiliary variables that leverage information from the prior model and theoretically analyze how they reduce the difficulty of the prediction problem in RPD.
    Experiments on synthetic datasets with fine-grained local structure show that standard diffusion models fail to capture local details, whereas RPD accurately captures fine-scale detail while preserving the large-scale structure of the distribution.
    On natural image generation tasks, RPD achieved generation quality that matched or exceeded that of representative diffusion-based baselines and it maintained strong performance even with a small number of inference steps.
\end{abstract}

\section{Introduction} \label{sec:introduction}

Diffusion models have established themselves as central tools of deep generative modeling and have achieved strong generative performance and stable training behavior across a wide range of domains, including images \citep{rombach2022high}, audio \citep{kongdiffwave}, and drug discovery \citep{xu2022geodiff}.
A standard diffusion model \citep{sohl2015deep,ho2020denoising,karras2022elucidating} constructs a reverse diffusion process that transforms a standard normal distribution into the target data distribution, and models this entire process using a single predictor---typically a deep neural network (DNN)---together with a single diffusion schedule.
Consequently, the model is required to capture both the global structure of the data distribution and the fine-scale local variations within a single framework.
However, when local structures exist at extremely small---or qualitatively different---scales relative to the global structure, this single-predictor, single-schedule design often faces challenges in capturing both scales simultaneously.

Such difficulties arise in many types of data distributions, including natural images. It is considered that natural images lie on a low-dimensional manifold embedded in a high-dimensional pixel space \citep{tenenbaum2000global,roweis2000nonlinear,bengio2013representation}. The manifold captures the large-scale structure of the image distribution, whereas textures, patterns, and edge details correspond to finer variations layered on top of it and require substantially more precise modeling. Designing a reverse diffusion process in which a single predictor and a single schedule must handle both the manifold-level structure and these fine-scale variations introduces fundamental representational challenges.
These issues are not limited to high-dimensional data such as images; they also arise in low-dimensional settings.

To address these challenges, we propose \emph{Residual Prior Diffusion} (RPD), a two-stage generative framework in which the coarse, large-scale structure of the data distribution is first captured by a prior model that is independent of the diffusion process. The diffusion model is then trained to represent the residual between the prior model and the true data distribution. By delegating the global variability of the distribution to the prior model, RPD allows the diffusion model to focus on learning the fine-scale local structure, thereby achieving an efficient separation between global organization and local detail.

As the prior model in RPD, we use a latent-variable Gaussian model, such as a variational autoencoder (VAE) \citep{kingmaauto} or one of its extensions \citep{higgins2017betavae,van2017neural}. While such latent-variable models often struggle to represent the fine details of a complex distribution, they can efficiently capture its coarse structure because they directly predict the mean of the distribution via latent variables. Moreover, these models typically allow one to infer the posterior distribution over the latent variable $z$ given an observed data $x_0$, a property that RPD exploits to improve the efficiency of training.

The diffusion model of the RPD framework is formulated as an explicit probabilistic model that incorporates the trained prior model as one of its components. We derive the evidence lower bound (ELBO) for this model and construct the associated training and inference procedures. Within this ELBO-based formulation, the training of the model reduces to predicting noise \citep{ho2020denoising} or velocity \citep{salimans2022progressive}, making it directly integrable with existing diffusion-model training frameworks.
We further introduce auxiliary variables that leverage information from the prior model during this prediction process. These variables are computed from the prior model, and we theoretically show that the more accurately the prior model reconstructs $x_0$, the more efficiently the prediction problem can be solved.

In our numerical experiments, we evaluated RPD by using both synthetic data and natural images.
For the synthetic experiments, we constructed two-dimensional datasets in which the local scale is significantly smaller than the global scale. Despite the simplicity of these distributions, we observed that conventional diffusion models and related approaches struggle to faithfully capture their local structures, whereas RPD accurately captures fine-scale details while preserving the global structure of the distribution.
We also evaluated RPD on image generation tasks using natural images. Across these tasks, RPD achieved a generation quality that matches or exceeds that of representative diffusion-based baselines. Notably, RPD was able to generate diverse and coherent images even when the number of inference steps is drastically reduced while avoiding the degradation typically observed in existing diffusion models.
Finally, we demonstrate that by employing a relatively lightweight prior model, the size of the DNN used to model the reverse diffusion process in RPD can be reduced to less than half the size of those used in standard diffusion models \citep{ho2020denoising} without observable degradation across standard distributional metrics.

The contributions of this paper are summarized as follows:
\begin{itemize}
  \item We propose RPD, a diffusion-based generative framework that transforms the distribution induced by a coarse prior model into the target data distribution (Section~\ref{sec:rpd_model}). The diffusion model is formulated as an explicit probabilistic model, and we derive its training and inference procedures from the ELBO under both noise-prediction modeling (Section~\ref{sec:rpd_denoising}) and velocity-prediction modeling (Section~\ref{sec:rpd_velocity}).
  \item We incorporate auxiliary variables in the prediction process and theoretically show that they improve its efficiency within the RPD framework (Section~\ref{sec:aux_var}).
  \item Through experiments on synthetic datasets and natural images, we demonstrate that RPD consistently achieves favorable performance and often surpasses diffusion-based baselines, highlighting the benefits of the proposed framework (Section~\ref{sec:experiment}).
\end{itemize}

\section{Related work} \label{sec:related_work}

\paragraph{Variational diffusion models}
Diffusion models generate data by learning the reverse of a forward diffusion process that gradually perturbs data into a standard normal distribution. Historically, they have been interpreted from several perspectives, including the variational view \citep{sohl2015deep,ho2020denoising}, score-based view \citep{songscore}, and flow-based view \citep{lipmanflow}, as summarized in \citep{lai2025principles}. RPD falls within the variational view.

Variational diffusion models \citep{sohl2015deep,ho2020denoising} are closely related to VAEs \citep{kingmaauto}.
A diffusion model can in fact be viewed as a multi-stage VAE whose encoder is fixed to a parameter-free diffusion process, enabling a more efficient training procedure via the conditioning trick \citep{luo2022understanding,lai2025principles}.
Despite this close relationship, only a limited number of methods directly integrate VAEs with diffusion models; DiffuseVAE \citep{pandey2022diffusevae}, discussed later, is one such example.
RPD can be viewed from the variational perspective as a framework that unifies a pretrained VAE with a diffusion model in a way that leverages their respective strengths.

\paragraph{Integrating diffusion models with other latent-variable models}
DiffuseVAE \citep{pandey2022diffusevae} is representative of the approach that integrates diffusion models with VAEs. It incorporates the VAE reconstruction into the diffusion process and is therefore methodologically related to RPD. However, our experiments showed that, because of its design constraints, the diffusion process in DiffuseVAE has difficulty capturing fine-scale details of the underlying distribution.
In particular, DiffuseVAE underperformed RPD on both synthetic examples and natural-image generation tasks. Moreover, DiffuseVAE does not take advantage of the auxiliary variables introduced in this work to improve prediction efficiency, and it is restricted to noise-prediction modeling \citep{ho2020denoising}, without providing a velocity-prediction formulation \citep{salimans2022progressive}.
Latent-variable priors other than VAEs---such as mixture-of-Gaussians models (MoGs)---have also been integrated with diffusion models, as shown by \citet{jia2024structured}. By comparison, RPD assumes a Gaussian latent-variable prior that includes both VAEs and MoGs as special cases; this makes it a more general framework for combining latent-variable models with diffusion processes.

\paragraph{Diffusion models in latent space}
The Latent Diffusion Model (LDM) \citep{rombach2022high} first trains an autoencoder, such as a VAE, and then performs diffusion in its latent space. While both RPD and LDM combine diffusion models with latent-variable models, their formulations differ fundamentally.
LDM applies a standard diffusion model in the latent space of an autoencoder, and the autoencoder and diffusion model are therefore not coupled through a single probabilistic objective during either training or inference.
In contrast, RPD does not merely run a diffusion model in a latent space: it integrates the latent-variable prior model into both its forward and reverse diffusion processes, resulting in a single ELBO that jointly couples the two models.
Moreover, LDM requires an autoencoder with sufficiently high reconstruction fidelity for its latent space to serve as an effective domain for diffusion. Training such a high-fidelity autoencoder is often technically challenging and computationally intensive, and hense it may become a bottleneck in the overall LDM pipeline. In contrast, although RPD also relies on a pretrained prior model, this prior is not required to capture the fine-scale structure of the target data distribution; it only needs to approximate its global structure, since the diffusion process of RPD is responsible for modeling fine-scale details. This reflects a fundamentally different design philosophy. Finally, the RPD framework could, in principle, be applied within the latent space of LDM, suggesting that the two approaches are not mutually exclusive and may in fact be complementary.

\paragraph{Residual modeling in diffusion-based super-resolution}
Diffusion-based super-resolution methods such as ResShift \citep{10.1109/TPAMI.2024.3461721} and Resfusion \citep{shi2024resfusion} model the residual between a high-resolution image $x_0$ and its corresponding low-resolution observation $y_0$.
These approaches therefore share with RPD the idea of ``modeling residuals through a diffusion process.’’ However, they require paired data $(x_0, y_0)$ during training and a low-resolution observation $y_0$ at inference time as a conditioning signal. As a result, they are not applicable to general-purpose generative modeling.
In contrast, RPD can automatically form such pairs by using the prior model’s reconstruction of $x_0$, thereby removing the need for $y_0$ entirely and enabling its use in generic generative modeling tasks.

\paragraph{Schr\"{o}dinger bridges}
Schr\"{o}dinger bridges provide a general framework for transforming probability distributions beyond the Gaussian setting, and their applications to generative modeling have been extensively explored \citep{de2021diffusion,shi2023diffusion,liu20232}.
They are tightly connected to entropy-regularized optimal transport \citep{leonard2013survey}, and several methods exist to extend classical optimal-transport algorithms such as the Sinkhorn iteration \citep{sinkhorn1964relationship} for use in generative modeling \citep{de2021diffusion,shi2023diffusion}.
The proposed RPD framework places the data distribution and prior model at the two endpoints of the transformation and learns the mapping between them as a diffusion process.
Because the posterior of the prior model provides an automatic coupling between data and latent variables, the resulting procedure is fundamentally different from approaches based on Schr\"{o}dinger bridges or optimal transport.

\section{Residual Prior Diffusion} \label{sec:rpd_model}
In this section, we describe the RPD framework in detail.
First, we introduce the mathematical notation used throughout the paper, followed by a description of the prior model assumed in RPD.
Then, we present the probabilistic formulation of the diffusion model.

\subsection{Notation}
Let $x_0 \in \mathbb{R}^n$ denote the random variable whose distribution is to be modeled, and let $p_{\mathrm{gt}}(x_0)$ represent its true underlying distribution. We assume that the training dataset $\mathcal{D}_{\mathrm{train}}$ consists of $N$ independent samples drawn from $p_{\mathrm{gt}}(x_0)$.
We use $\mathcal{N}(x \mid \mu, \Sigma)$ to denote the density of a multivariate normal distribution with mean vector $\mu$ and covariance matrix $\Sigma$. For a random variable $y$ sampled from $q(y)$, the expectation of a function $f(y)$ is written as $\mathbb{E}_{q(y)}[f(y)]$. The Kullback-Leibler (KL) divergence between distributions $p$ and $q$ is expressed as $D_{\mathrm{KL}}(p \,\|\, q)$.
We denote by $I_n$ the $n \times n$ identity matrix, and by $\mathbb{R}_{+}$ the set of positive real numbers. Finally, $\mathcal{U}(t \mid 1, T)$ denotes the uniform distribution supported on the discrete time indices $\{1, \ldots, T\}$.

\subsection{Prior model}
We define the prior model used in RPD as follows.
\begin{dfn}[Prior model]\label{def:prior_model}
  The prior model is specified by the joint distribution $\hat{p}(x_0, z)$ over the data variable $x_0$ and a latent variable $z \in \mathcal{Z}$:
  \begin{align}
    \hat{p}(x_0, z) = \hat{p}(x_0 \mid z)\, \hat{p}(z), \label{eq:prior_model}
  \end{align}
  where $\hat{p}(z)$ denotes an arbitrary prior distribution over $z$.
  The conditional distribution $\hat{p}(x_0 \mid z)$ is assumed to be a Gaussian of the form
  \begin{align}
    \hat{p}(x_0 \mid z)
    = \mathcal{N}\left(
    x_0 \,\middle|\, \hat{\mu}(z),\ \hat{\sigma}^2(z)\, I_n
    \right), \label{eq:prior_x0_z}
  \end{align}
  where $\hat{\mu} : \mathcal{Z} \to \mathbb{R}^n$ and $\hat{\sigma} : \mathcal{Z} \to \mathbb{R}$ denote functions mapping each latent variable $z$ to the mean and standard deviation of $x_0$, respectively.
  We further denote by $\hat{q}(z \mid x_0)$ a model that approximates the posterior distribution of $z$ given data $x_0$.
\end{dfn}

Typical examples of prior models that satisfy the above conditions include VAEs \citep{kingmaauto} and their extensions \citep{higgins2017betavae,van2017neural}.
In these models, the encoder corresponds to $\hat{q}(z \mid x_0)$, while the decoder corresponds to $\hat{p}(x_0 \mid z)$.
Other models, such as mixtures of Gaussians, can also serve as prior models.

Such prior models can flexibly capture the coarse structure of a data distribution, as the latent variable $z$ combined with the nonlinear mapping $\hat{\mu}(z)$ provides substantial flexibility in parameterizing the mean of the distribution.
However, because the conditional distribution in \eqref{eq:prior_x0_z} is Gaussian, these models generally struggle to capture fine-grained details \citep{dosovitskiy2016generating}.

\paragraph{Training of the prior model}
Given a training dataset $\mathcal{D}_\mathrm{train}$, the prior model is trained by maximizing the following ELBO with respect to $\hat{p}(x_0 \mid z)$ and $\hat{q}(z \mid x_0)$:
\begin{align}
  \mathbb{E}_{\hat{q}(z \mid x_0)}\!\left[\log \hat{p}(x_0 \mid z)\right]
  - D_\mathrm{KL}\!\left(\hat{q}(z \mid x_0)\, \|\, \hat{p}(z)\right).
  \label{eq:elbo_prior}
\end{align}
Such ELBO-based optimization, commonly referred to as variational inference, is widely used in models such as VAEs.

\subsection{Formulation of RPD}
We consider a setting in which a pretrained prior model satisfying the conditions described in the previous subsection is available.
RPD assumes discrete time steps $t = 0, \ldots, T$ (with $T$ denoting the final time step), and introduces a latent variable $x_t \in \mathbb{R}^n$ for each $t = 1, \ldots, T$.
Note that $x_0$ corresponds to the data variable to be modeled and is therefore not treated as a latent variable.
For a time interval $[t_s, t_e]$, we write $x_{t_s:t_e} = (x_{t_s}, \ldots, x_{t_e})$.
Under this setup, the probabilistic model defining RPD is given as follows.

\begin{dfn}[RPD model]\label{def:rpd_reverse}
  The probabilistic model of RPD is defined by
  \begin{align}
    p_\theta(x_{0:T}, z)
     & = \hat{p}(z)\,
    \hat{p}(x_T \mid z)
    \prod_{t=1}^T p_\theta(x_{t-1} \mid x_t, z),
    \label{eq:rev_1}         \\
    p_\theta(x_{t-1} \mid x_t, z)
     & = \mathcal{N}\!\left(
    x_{t-1}
    \,\middle|\,
    \mu_\theta(x_t, t, z),\,
    \sigma_t^2(z)\, I_n
    \right),
    \quad t = 1, \ldots, T,
    \label{eq:rev_2}
  \end{align}
  where $\hat{p}(z)$ is the distribution of the latent variable $z$ provided by the prior model, and
  $\hat{p}(x_T \mid z)$ is obtained by substituting $x_0$ with $x_T$ in~\eqref{eq:prior_x0_z}.
\end{dfn}

According to \eqref{eq:rev_1}, $x_T$ is assumed to follow the distribution predicted by the prior model and thus provides a coarse approximation to $x_0$.
In \eqref{eq:rev_2}, the function $\mu_\theta$ predicts the mean of $x_{t-1}$ based on the current latent state $x_t$, time step $t$, and latent variable $z$.
We parameterize $\mu_\theta$ using a neural network and denote its trainable parameters by~$\theta$.
The function $\sigma_t(z)$ specifies the standard deviation of $x_{t-1}$; it is not a trainable parameter, but instead admits a closed-form expression obtained from the ELBO maximization procedure of RPD (see Section~\ref{sec:rpd_denoising} for details).

\subsection{Posterior model}
To train RPD via variational inference, we introduce a posterior model for the latent variables $(x_{1:T}, z)$.
\begin{dfn}[Posterior model]\label{def:rpd_forward}
  The posterior model for RPD is defined as
  \begin{align}
    q(x_{1:T}, z \mid x_0)
     & = \hat{q}(z \mid x_0)
    \prod_{t=1}^T q(x_t \mid x_{t-1}, z),
    \label{eq:fwd_1}         \\
    q(x_t \mid x_{t-1}, z)
     & = \mathcal{N}\!\left(
    x_t
    \,\middle|\,
    \sqrt{\alpha_t}\, x_{t-1}
    + \left(1 - \sqrt{\alpha_t}\right)\hat{\mu}(z),\,
    \beta_t \hat{\sigma}^2(z)\, I_n
    \right),
    \quad t = 1, \ldots, T,
    \label{eq:fwd_2}
  \end{align}
  where $\beta_t \in \mathbb{R}_{+}$ $(t = 1, \ldots, T)$ denotes the noise-schedule parameters, and $\alpha_t$ is defined by $\alpha_t = 1 - \beta_t$.
\end{dfn}
As \eqref{eq:fwd_1} indicates, RPD directly employs the posterior model $\hat{q}(z \mid x_0)$ derived from the prior model.
The functions $\hat{\mu}$ and $\hat{\sigma}$ in \eqref{eq:fwd_2} are also derived from the prior model.

\subsection{Comparison with standard diffusion models}
The probabilistic model of RPD (Definition~\ref{def:rpd_reverse}) corresponds to the reverse diffusion process in standard discrete-time diffusion models \citep{ho2020denoising}.
In standard diffusion models, $x_T$ is sampled from a standard normal distribution, whereas in RPD, $x_T$ is sampled from the prior model.
Consequently, the reverse diffusion process in RPD starts from an $x_T$ drawn from the prior model and transforms it toward the data distribution through \eqref{eq:rev_2}.
If the prior model captures the global structure of the data distribution, the reverse process can focus more on correcting the discrepancy between the prior model and the true data distribution.
Another key distinction is that each reverse-diffusion step in RPD is conditioned on the latent variable $z$, so the entire reverse process adapts to $z$.
Furthermore, as shown in Section~\ref{sec:aux_var}, incorporating auxiliary variables that leverage information from the prior model can further improve the efficiency of the prediction performed by $\mu_\theta$.
Figure~\ref{fig:overview} presents an overview highlighting the differences between a standard diffusion model and RPD in an image generation task.
\begin{figure}[tbp]
  \centering
  \includegraphics[width=0.9\linewidth,clip]{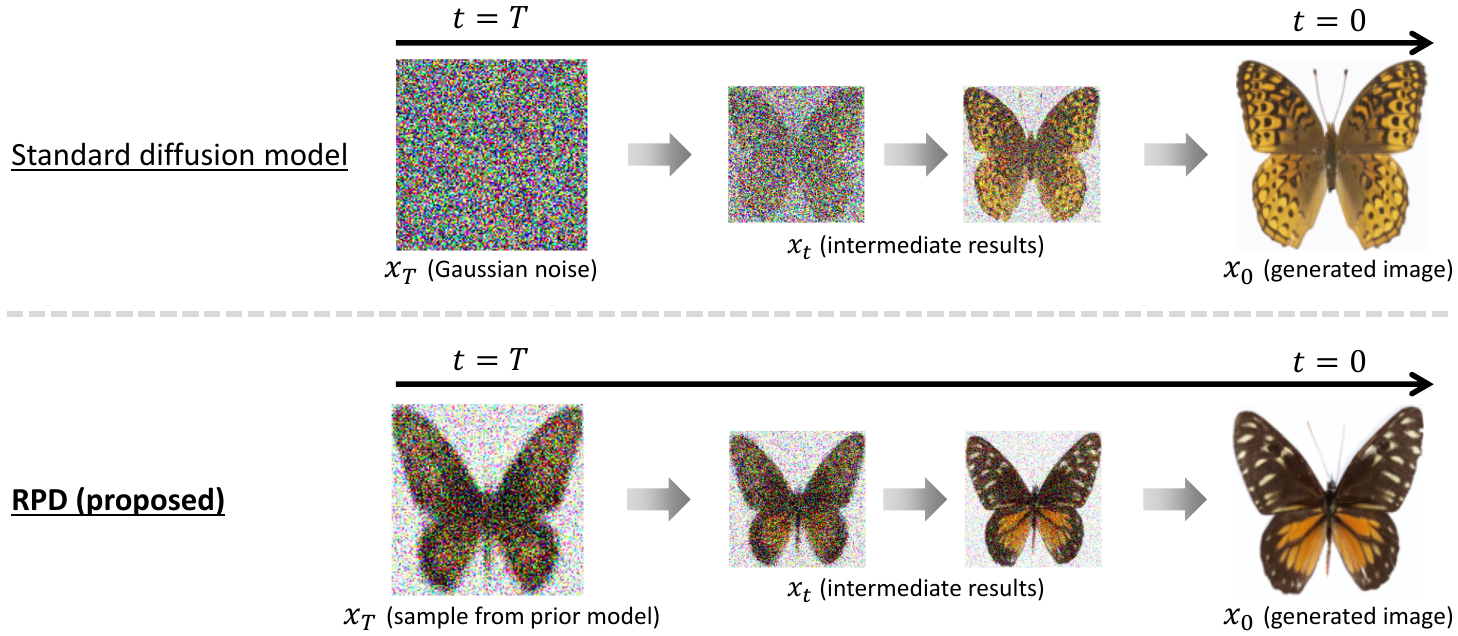}
  \caption{Overview comparing a standard diffusion model (top) with RPD (bottom).}
  \label{fig:overview}
\end{figure}

The posterior model of RPD (Definition~\ref{def:rpd_forward}) corresponds to the forward diffusion process in standard diffusion models.
In standard diffusion models, the forward process gradually transforms the input data $x_0$ into an $x_T$ that follows a standard normal distribution.
In contrast, RPD uses the posterior distribution $\hat{q}(z \mid x_0)$ provided by the prior model to infer the latent variable $z$ associated with $x_0$, and the entire diffusion process is conditioned on this~$z$.
As shown in Section~\ref{sec:rpd_denoising}, under the RPD diffusion process defined by \eqref{eq:fwd_2}, the variable $x_t$ evolves linearly (up to the noise term) from $x_0$ toward $\hat{\mu}(z)$ as $t$ increases, and at $t = T$ converges to a Gaussian distribution with mean $\hat{\mu}(z)$ and covariance $\hat{\sigma}^2(z) I_n$.
Notably, this terminal distribution coincides with the conditional prior model in \eqref{eq:prior_x0_z}.
Thus, the RPD forward process first infers the latent variable $z$ from $x_0$, and then transforms $x_0$ into an $x_T$ that follows the $z$-conditional prior model.
In Section~\ref{sec:rpd_denoising}, we justify this posterior model from the perspective of ELBO maximization.
Moreover, this forward process admits a simple normalization that makes the ``residual'' interpretation explicit.
Specifically, by centering $x_t$ at the prior mean $\hat{\mu}(z)$ and rescaling by the prior standard deviation $\hat{\sigma}(z)$, the resulting variable follows the standard forward diffusion form conditioned on $z$.
This provides a useful perspective in which RPD effectively operates in a prior-centered coordinate system; see Section~\ref{sec:residual_modeling} for details.

\subsection{Comparison with DiffuseVAE} \label{sec:compare_with_diffusevae}
DiffuseVAE \citep{pandey2022diffusevae} is methodologically close to RPD in that it leverages a pretrained VAE followed by diffusion-based generation. However, as demonstrated by the experimental results in Section~\ref{sec:experiment}, a substantial performance gap exists between the two approaches. In this subsection, we clarify the primary sources of this discrepancy.

The central difference lies in the definition of the forward diffusion process (i.e., the posterior model).
In DiffuseVAE, letting $\hat{x}_0$ denote the VAE reconstruction of $x_0$, the forward diffusion process is designed such that $x_t$ converges, as $t \to T$, to a Gaussian distribution whose mean is $x_0 + \hat{x}_0$ (see \citet{pandey2022diffusevae} for details).
Consequently, the terminal distribution of the forward process does not coincide with the distribution that the reverse process assumes as its initial state.
This stands in sharp contrast to RPD, where these two distributions are shown to coincide by construction.
To compensate for this mismatch, DiffuseVAE subtracts the contribution corresponding to $\hat{x}_0$ from the generated sample at the final inference step and treats the adjusted value as the final output.
From a variational perspective, this structural inconsistency is expected to degrade the ELBO, particularly the term $\mathcal{J}_2$ introduced in Section~\ref{sec:rpd_denoising}, thereby leading to inferior generative performance.

Another important distinction is that RPD introduces auxiliary variables that explicitly exploit information from the prior model to improve prediction efficiency, as discussed in Section~\ref{sec:aux_var}.
Furthermore, it admits both noise-prediction and velocity-prediction parameterizations within a unified probabilistic framework.
In contrast, DiffuseVAE neither incorporates such auxiliary variables nor provides a velocity-prediction formulation.

\section{ELBO-based noise-prediction modeling for RPD} \label{sec:rpd_denoising}

This section derives an ELBO for RPD and formulates a training objective under the noise-prediction parameterization.
While the algebraic manipulations largely follow the standard variational analysis of diffusion models \citep{luo2022understanding}, the contribution here is to establish an ELBO tailored to RPD and to show that its posterior model (forward process) is variationally well justified.
Building on this foundation, we obtain a concrete learning algorithm in which the interaction between the prior model and the diffusion process is made explicit and mathematically consistent.

\subsection{Derivation of the ELBO}

Given the probabilistic model in Definition~\ref{def:rpd_reverse} and the posterior model in Definition~\ref{def:rpd_forward}, the ELBO for RPD can be derived as follows:
\begin{align}
  \log p_\theta(x_0) & \geq \mathbb{E}_{q(x_{1:T}, z \mid x_0)}\left[\log \frac{p_\theta(x_{0:T}, z)}{q(x_{1:T}, z \mid x_0)}\right]                                    \notag                                                                                                             \\
                     & = \underbrace{- D_{\mathrm{KL}} \left(\hat{q}(z \mid x_0) \ || \ \hat{p}(z)\right)}_{\mathcal{J}_1} \underbrace{- \mathbb{E}_{\hat{q}(z \mid x_0)}\left[ D_{\mathrm{KL}} \left(q(x_T \mid x_0, z) \ || \ \hat{p}(x_T \mid z)\right) \right]}_{\mathcal{J}_2} \notag \\
                     & \qquad  + \underbrace{\mathbb{E}_{\hat{q}(z \mid x_0) q(x_1 \mid x_0, z)}\left[\log p_\theta(x_0 \mid x_1, z)\right] }_{\mathcal{J}_3}   \notag                                                                                                                     \\
                     & \qquad \underbrace{- \sum_{t=2}^T \mathbb{E}_{\hat{q}(z \mid x_0) q(x_{t} \mid x_0, z)} \left[ D_{\mathrm{KL}}\left(q(x_{t-1} \mid x_{t}, x_0, z) \ || \ p_\theta(x_{t-1} \mid x_t, z) \right)\right] }_{\mathcal{J}_4}. \label{eq:ELBO}
\end{align}
The detailed derivation is provided in \ref{apdx:elbo}.
Maximizing this ELBO with respect to $\theta$ yields the training objective.
In what follows, we examine each term $\mathcal{J}_i \ (i = 1,2,3,4)$ and discuss how each contributes to the optimization of the ELBO.

\subsection{Maximizing $\mathcal{J}_1$ and $\mathcal{J}_2$}
The terms $\mathcal{J}_1$ and $\mathcal{J}_2$ in the ELBO~\eqref{eq:ELBO} are independent of the learnable parameters $\theta$ and therefore do not directly affect the optimization of the diffusion model.
However, in terms of minimizing the variational gap---the discrepancy between the approximate posterior $q(x_{1:T}, z \mid x_0)$ and the true posterior $p_\theta(x_{1:T}, z \mid x_0)$---it is still desirable for these $\theta$-independent terms to be as large as possible.
The term $\mathcal{J}_1$ appears in the ELBO of the prior model~\eqref{eq:elbo_prior} and is expected to have already been sufficiently maximized during the prior-model training stage.
Furthermore, Proposition~\ref{prop:D_kl_J2} shows that, under mild assumptions, $\mathcal{J}_2$ is automatically maximized as a consequence of the structural design of the diffusion process in the RPD framework.

\subsection{Maximizing $\mathcal{J}_3$ and $\mathcal{J}_4$}
We begin with the maximization of $\mathcal{J}_4$.
Below, we first derive in detail the distributions $q(x_t \mid x_0, z)$, which appear inside the expectation in $\mathcal{J}_4$, and $q(x_{t-1} \mid x_t, x_0, z)$, which appear in the KL-divergence term.
We then derive an explicit expression for maximizing $\mathcal{J}_4$ under the noise-prediction parameterization.

\subsubsection{Deriving $q(x_t \mid x_0, z)$}
The variable $x_t$ appearing in the expectation term of $\mathcal{J}_4$ follows the distribution $q(x_t \mid x_0, z)$ and can be expressed using $\epsilon_0 \sim \mathcal{N}(\epsilon_0 \mid 0, I_n)$ as follows (see \ref{apdx:q_xt_x0_z} for the derivation):
\begin{align}
  x_t & = \sqrt{\bar{\alpha}_t} x_0 + \left(1 - \sqrt{\bar{\alpha}_t}\right) \hat{\mu}(z) + \sqrt{\left(1 - \bar{\alpha}_t\right)}\hat{\sigma}(z) \epsilon_0, \label{eq:xt_x0_ep0_f}
\end{align}
where $\bar{\alpha}_t := \prod_{i=1}^t \alpha_i$.
It follows directly from \eqref{eq:xt_x0_ep0_f} that $q(x_t \mid x_0, z)$ is given by
\begin{align}
  q(x_t \mid x_0, z) & = \mathcal{N}\left(x_t \ \middle| \ \sqrt{\bar{\alpha}_t} x_0  + \left(1 - \sqrt{\bar{\alpha}_t}\right) \hat{\mu}(z), \left(1 - \bar{\alpha}_t\right) \hat{\sigma}^2(z) I_n\right).  \label{eq:q_xt_x0_z_norm}
\end{align}
From this expression, we obtain the following proposition.
\begin{proposition} \label{prop:D_kl_J2}
  Suppose that $\bar{\alpha}_T \approx 0$. Then the following relation holds:
  \begin{align}
    \mathcal{J}_2 \approx 0. \notag
  \end{align}
\end{proposition}
\begin{proof}
  From \eqref{eq:q_xt_x0_z_norm}, we have
  \[
    q(x_T \mid x_0, z) \approx \mathcal{N}\!\left(x_T \mid \hat{\mu}(z),\, \hat{\sigma}^2(z) I_n\right).
  \]
  Therefore, by combining this with \eqref{eq:prior_x0_z}, we obtain
  \begin{align*}
    \mathcal{J}_2
     & = - \mathbb{E}_{\hat{q}(z \mid x_0)}\!\left[
      D_{\mathrm{KL}}\!\left(
      q(x_T \mid x_0, z)
      \ \middle\|\
      \hat{p}(x_T \mid z)
      \right)
    \right]                                               \\
     & \approx - \mathbb{E}_{\hat{q}(z \mid x_0)}\!\left[
      D_{\mathrm{KL}}\!\left(
      \mathcal{N}\!\left(x_T \mid \hat{\mu}(z),\, \hat{\sigma}^2(z) I_n\right)
      \ \middle\|\
      \mathcal{N}\!\left(x_T \mid \hat{\mu}(z),\, \hat{\sigma}^2(z) I_n\right)
      \right)
      \right] = 0.
  \end{align*}
\end{proof}
Since the KL divergence is non-negative, this proposition implies that, under a mild assumption, the $\mathcal{J}_2$ term in the ELBO is maximized by construction in RPD.

\subsubsection{Deriving $q(x_{t-1} \mid x_t, x_0, z)$}
The distribution $q(x_{t-1} \mid x_t, x_0, z)$ appearing in the KL-divergence term of $\mathcal{J}_4$ can be rewritten as follows for $t = 2, \ldots, T$.
\begin{align}
  q(x_{t-1} \mid x_{t}, x_0, z) & = \mathcal{N}\left(x_{t-1} \ \middle| \ \tilde{\mu}(x_t, x_0, z), \tilde{\beta}_t(z) I \right), \label{eq:q_xt-1_xt_x0_z}
\end{align}
where $\tilde{\mu}(x_t, x_0, z)$ and $\tilde{\beta}_t(z)$ are given by the following:
\begin{align}
  \tilde{\mu}(x_t, x_0, z) & = \frac{\sqrt{\alpha_t} (1-\bar{\alpha}_{t-1})}{1 - \bar{\alpha}_t} x_t + \frac{(1 - \alpha_t)\sqrt{\bar{\alpha}_{t-1}}}{1 - \bar{\alpha}_t} x_0 + \nu_t \hat{\mu}(z),      \label{eq:tilde_mu_x0} \\
  \tilde{\beta}_t(z)       & = \frac{(1 - \alpha_t)(1 - \bar{\alpha}_{t-1})}{1 - \bar{\alpha}_t} \hat{\sigma}^2(z), \label{eq:tilde_beta}
\end{align}
where
\begin{align}
  \nu_t & = \frac{\left(1- \sqrt{\alpha_t}\right) \left(1 - \sqrt{\bar{\alpha}_{t-1}}\right)}{1 +  \sqrt{\bar{\alpha}_t}}.
\end{align}
Derivations of these expressions can be found in \ref{apdx:q_xt-1_xt_x0_z}.
Using \eqref{eq:q_xt-1_xt_x0_z} and minimizing the KL-divergence term in $\mathcal{J}_4$, we set $\sigma_t^2(z)$ in \eqref{eq:rev_2} to $\sigma_t(z) = \sqrt{\tilde{\beta}_t(z)}$.

\subsubsection{Derivation of $\mathcal{J}_4$ under noise-prediction modeling}
From \eqref{eq:q_xt-1_xt_x0_z} and \eqref{eq:rev_2}, the KL-divergence term in $\mathcal{J}_4$ becomes the KL divergence between two Gaussian distributions, and it can therefore be rewritten as
\begin{align}
  \mathcal{J}_4 & = -\sum_{t=2}^T \mathbb{E}_{\hat{q}(z \mid x_0) q(x_{t} \mid x_0, z)}\left[D_{\text{KL}}\left(q(x_{t-1} \mid x_{t}, x_0, z) \ || \ p_\theta(x_{t-1}|x_t, z) \right) \right]                                  \\
                & = -\sum_{t=2}^T \mathbb{E}_{\hat{q}(z \mid x_0) q(x_{t} \mid x_0, z)}\left[\frac{1}{2 \tilde{\beta}_t(z)} \left\|\tilde{\mu}(x_t, x_0, z) - \mu_\theta(x_t, t, z)\right\|^2 \right] + C_1,  \label{eq:J4_mu}
\end{align}
where $C_1$ is a constant independent of~$\theta$.
Maximizing $\mathcal{J}_4$ therefore reduces to learning $\mu_\theta(x_t, t, z)$ so that it approximates $\tilde{\mu}(x_t, x_0, z)$.

Although it is possible to directly learn $\mu_\theta(x_t, t, z)$ from the above expression, we follow \citep{ho2020denoising} and reformulate the problem by rewriting the mean $\tilde{\mu}(x_t, x_0, z)$ and reparameterizing $\mu_\theta(x_t, t, z)$ accordingly, so that the objective becomes predicting the noise component contained in $\tilde{\mu}(x_t, x_0, z)$.
From \eqref{eq:xt_x0_ep0_f}, we obtain
\begin{align}
  x_0 = \frac{1}{\sqrt{\bar{\alpha}_t}}\left(x_t - \left(1 - \sqrt{\bar{\alpha}_t}\right) \hat{\mu}(z) - \sqrt{\left(1 - \bar{\alpha}_t\right) \hat{\sigma}^2(z)} \epsilon_0\right),
\end{align}
which can be substituted into \eqref{eq:tilde_mu_x0} to yield
\begin{align}
  \tilde{\mu}(x_t, x_0, z) & = \frac{1}{\sqrt{\alpha_t}} x_t - \frac{1 - \sqrt{\alpha_t}}{\sqrt{\alpha_t}} \hat{\mu}(z) - \frac{1 - \alpha_t}{\sqrt{\left(1 - \bar{\alpha}_t\right) \alpha_t}} \sqrt{\hat{\sigma}^2(z)} \epsilon_0. \label{eq:mu_x0_to_e}
\end{align}
Motivated by \eqref{eq:mu_x0_to_e}, we reparameterize $\mu_\theta(x_t, t, z)$ using $\epsilon_\theta(x_t, t, z)$ as
\begin{align}
  \mu_\theta(x_t,t, z) = \frac{1}{\sqrt{\alpha_t}} x_t - \frac{1 - \sqrt{\alpha_t}}{\sqrt{\alpha_t}} \hat{\mu}(z) - \frac{1 - \alpha_t}{\sqrt{\left(1 - \bar{\alpha}_t\right) \alpha_t}} \sqrt{\hat{\sigma}^2(z)} \epsilon_\theta(x_t, t, z). \label{eq:mu_theta}
\end{align}
Note that this reparameterization changes the learning target from $\mu_\theta(x_t, t, z)$ in \eqref{eq:rev_2} to $\epsilon_\theta(x_t, t, z)$.
Under this reparameterization, $\mathcal{J}_4$ can be rewritten as
\begin{align}
  \mathcal{J}_4 & = - \sum_{t=2}^T \mathbb{E}_{\hat{q}(z \mid x_0) \mathcal{N}(\epsilon_0 \mid 0, I)}\left[\frac{(1-\alpha_t)^2 \hat{\sigma}^2(z)}{2 \tilde{\beta}_t(z) \alpha_t (1-\bar{\alpha}_t)}  \left\|\epsilon_0 - \epsilon_\theta(x_t, t, z)\right\|^2 \right] + C_2,  \label{eq:J4_epsilon}
\end{align}
where $C_2$ is a constant independent of~$\theta$.
Maximizing $\mathcal{J}_4$ therefore reduces to learning $\epsilon_\theta(x_t, t, z)$ so that it predicts the noise component $\epsilon_0$ contained in $\tilde{\mu}(x_t, x_0, z)$.

\subsubsection{$\mathcal{J}_3$ under noise-prediction modeling}
The term $\mathcal{J}_3$ in the ELBO can be rewritten under noise-prediction modeling as follows (see \ref{apdx:J3_epsilon} for the derivation):
\begin{align}
  \mathcal{J}_3 & = - \mathbb{E}_{\hat{q}(z \mid x_0) \mathcal{N}(\epsilon_0 \mid 0, I)}\left[\frac{\left(1 - \alpha_1\right) \hat{\sigma}^2(z)}{2 \tilde{\beta}_1(z) \alpha_1} \left\|\epsilon_0 - \epsilon_\theta (x_1, 1, z) \right\|^2 \right] + C_3,  \label{eq:J3_epsilon}
\end{align}
where $C_3$ is a constant independent of~$\theta$, and we set $\sigma_1^2(z) = \tilde{\beta}_1(z) = \sigma_{\min}^2$ with a fixed small constant $\sigma_{\min} > 0$ (e.g., $\sigma_{\min}=10^{-3}$) to avoid a degenerate variance at $t=1$.

\subsection{Loss function of RPD under noise-prediction modeling}
From \eqref{eq:J4_epsilon} and \eqref{eq:J3_epsilon}, the loss function for maximizing the ELBO in \eqref{eq:ELBO} with respect to $\theta$ is given by
\begin{align}
  \ell^\epsilon(\theta; x_0) & = \mathbb{E}_{\mathcal{U}(t \mid 1,T) \hat{q}(z \mid x_0) \mathcal{N}(\epsilon_0 \mid 0, I)} \left[\frac{(1-\alpha_t)^2 \hat{\sigma}^2(z)}{2 \tilde{\beta}_t(z) \alpha_t (1-\bar{\alpha}_t)} \left\|\epsilon_0 - \epsilon_\theta(x_t, t, z)\right\|^2 \right], \label{eq:ell_proposed}
\end{align}
where $x_t$ is computed from $(t, x_0, z, \epsilon_0)$ according to \eqref{eq:xt_x0_ep0_f}.
Following \citet{ho2020denoising}, we simplify the above loss by setting all coefficients to one, yielding
\begin{align}
  \ell^\epsilon_\mathrm{simple}(\theta; x_0) & = \mathbb{E}_{\mathcal{U}(t \mid 1,T) \hat{q}(z \mid x_0) \mathcal{N}(\epsilon_0 \mid 0, I)} \left[ \left\|\epsilon_0 - \epsilon_\theta(x_t, t, z)\right\|^2 \right]. \label{eq:ell_simple}
\end{align}
A theoretical analysis of the effect of coefficient simplification in diffusion models is provided by \citet{kingma2023understanding}.

\subsection{Residual modeling induced by centering at the prior} \label{sec:residual_modeling}
A useful perspective on RPD is obtained by centering and normalizing the intermediate diffusion state with respect to the prior model.
From \eqref{eq:xt_x0_ep0_f}, let us define
\begin{align*}
  y_t := \frac{x_t - \hat{\mu}(z)}{\hat{\sigma}(z)}.
\end{align*}
Then,
\begin{align*}
  y_t = \sqrt{\bar{\alpha}_t}\, y_0 + \sqrt{1 - \bar{\alpha}_t}\, \epsilon_0,
  \qquad
  y_0 := \frac{x_0 - \hat{\mu}(z)}{\hat{\sigma}(z)},
\end{align*}
which has exactly the same form as the standard forward diffusion process, conditioned on $z$.
Therefore, under noise-prediction modeling, the network learns to predict noise in this prior-centered coordinate system.
As a result, the diffusion component primarily models the residual discrepancy between $x_0$ and the coarse structure captured by the prior model.
Importantly, while the forward process becomes the standard one in the $y_t$ coordinates,
the reverse process remains conditioned on $z$ through the prior-dependent parameters,
so this transformation does not reduce RPD to a conventional diffusion model.
This prior-centered normalization also motivates the auxiliary variables introduced in Section~\ref{sec:aux_var},
which further exploit this coordinate system to accelerate prediction.

\subsection{Algorithms} \label{sec:alg_epsilon}
On the basis of the above derivation, Algorithm~\ref{alg:rpd_training} summarizes the training procedure of RPD under noise-prediction modeling, while Algorithm~\ref{alg:rpd_inference} presents the corresponding inference (generation) procedure using the trained model.
Details regarding $\omega_t^\epsilon$, which appears in these algorithms, are provided in Section~\ref{sec:aux_var}.
\begin{figure}[htbp]
  \begin{algorithm}[H]
    \caption{RPD training procedure under noise prediction modeling}
    \label{alg:rpd_training}
    \begin{algorithmic}[1]
      \Require Training dataset~$\mathcal{D}_\mathrm{train}$
      \Require $\hat{q}(z \mid x_0)$, $\hat{\mu}(z)$, and $\hat{\sigma}(z)$ from the prior model
      \Require Hyperparameters $\left\{\alpha_t \ \middle| \ t=1,\ldots,T\right\}$
      \Ensure Trained model $\epsilon_\theta$
      \Repeat
      \State $x_0 \sim \mathcal{D}_\mathrm{train}$, $t \sim \mathcal{U}(t \mid 1,T)$, $\epsilon_0 \sim \mathcal{N}(\epsilon_0 \mid 0, I)$
      \State $z \sim \hat{q}(z \mid x_0)$
      \State $x_t = \sqrt{\bar{\alpha}_t} x_0 + \left(1 - \sqrt{\bar{\alpha}_t}\right) \hat{\mu}(z) + \sqrt{\left(1 - \bar{\alpha}_t\right) \hat{\sigma}^2(z)} \epsilon_0$ \Comment{Eq.~\eqref{eq:xt_x0_ep0_f}}
      \State $\omega_t^\epsilon = \frac{1}{\sqrt{1 - \bar{\alpha}_t}} \frac{x_t - \hat{\mu}(z)}{\hat{\sigma}(z)}$ \Comment{Eq.~\eqref{eq:omega_t_epsilon}}
      \State Take a gradient-descent step on:
      $\nabla_\theta \left\|\epsilon_0 - \epsilon_\theta(x_t, t, z, \omega_t^\epsilon)\right\|^2$
      \Until{convergence}
    \end{algorithmic}
  \end{algorithm}
  \begin{algorithm}[H]
    \caption{RPD inference procedure under noise prediction modeling}
    \label{alg:rpd_inference}
    \begin{algorithmic}[1]
      \Require $\hat{p}(z)$, $\hat{\mu}(z)$, and $\hat{\sigma}(z)$ from the prior model
      \Require Hyperparameters $\left\{\alpha_t \ \middle| \ t=1,\ldots,T\right\}$
      \Require Trained model $\epsilon_\theta$
      \Ensure Generated sample $x_0$
      \State $z \sim \hat{p}(z)$
      \State $x_T \sim \mathcal{N}\left(x_T \ \middle| \ \hat{\mu}(z), \hat{\sigma}^2(z) I \right)$
      \For{$t = T$ to $1$}
      \State $\omega_t^\epsilon = \frac{1}{\sqrt{1 - \bar{\alpha}_t}} \frac{x_t - \hat{\mu}(z)}{\hat{\sigma}(z)}$ \Comment{Eq.~\eqref{eq:omega_t_epsilon}}
      \State $\mu_\theta(x_t, t, z) = \frac{1}{\sqrt{\alpha_t}} x_t - \frac{1 - \sqrt{\alpha_t}}{\sqrt{\alpha_t}} \hat{\mu}(z) - \frac{\left(1 - \alpha_t\right) \sqrt{\hat{\sigma}^2(z)}}{\sqrt{\left(1 - \bar{\alpha}_t\right) \alpha_t}} \epsilon_\theta(x_t, t, z, \omega_t^\epsilon)$ \Comment{Eq.~\eqref{eq:mu_theta}}
      \State $\sigma_t(z) = \sqrt{\frac{(1 - \alpha_t)(1 - \bar{\alpha}_{t-1})}{1 - \bar{\alpha}_t}} \hat{\sigma}(z)$ if $t \geq 2$ else $\sigma_t(z) = \sigma_{\min}$  \Comment{Eq.~(\ref{eq:tilde_beta})}
      \State $x_{t-1} \sim \mathcal{N}\left(x_{t-1} \ \middle| \ \mu_\theta(x_t, t, z), \sigma_t^2(z) I \right)$
      \EndFor
    \end{algorithmic}
  \end{algorithm}
\end{figure}

\paragraph{Inference with a reduced number of steps}
In diffusion models, the number of inference steps does not need to match the number of steps used during training, a practice commonly adopted in libraries such as \citep{von-platen-etal-2022-diffusers}. The same applies to RPD.
In the RPD inference procedure (Algorithm~\ref{alg:rpd_inference}), we may use a reduced number of steps $S < T$ by selecting a subsequence of time indices,
\[
  1 = \tau_1 < \tau_2 < \cdots < \tau_S = T .
\]
An effective noise schedule on this subsequence is obtained by defining
\[
  \alpha'_s = \frac{\bar{\alpha}_{\tau_s}}{\bar{\alpha}_{\tau_{s-1}}}, \qquad s = 2,\dots,S,
\]
which ensures that the reduced schedule remains consistent with the cumulative noise level imposed during training.

\section{Velocity prediction modeling for RPD} \label{sec:rpd_velocity}
In the previous section, we derived a loss function for maximizing the ELBO of RPD under noise prediction modeling. An alternative approach is to parameterize $\mu_\theta$ based on the notion of velocity \citep{salimans2022progressive}.
It is discussed in \citep{albergo2023stochastic, zhouinductive} that velocity-based formulations are closely connected to flow models \citep{lipmanflow, liuflow}.
Therefore, deriving the training procedure for RPD under velocity prediction modeling may lead to a broader range of applications. In this section, we derive a loss function for maximizing the ELBO of RPD under velocity prediction modeling. Because RPD uses an explicit probabilistic model, it naturally admits equivalent formulations under various modeling frameworks.

\subsection{Introducing the velocity and its modification}
As preparation for velocity prediction modeling, let $w_t = \sqrt{\bar{\alpha}_t}$. Then, from \eqref{eq:xt_x0_ep0_f}, we obtain the following expression:
\begin{align}
  x_t & = w_t x_0 + \left(1 - w_t\right) \hat{\mu}(z) + \sqrt{\left(1 - w_t^2\right) \hat{\sigma}^2(z)} \epsilon_0. \label{eq:xt_wt}
\end{align}
Defining the derivative of $x_t$ with respect to $w_t$ as the velocity $v_t$ yields
\begin{align}
  v_t := \frac{d x_t}{d w_t} & = x_0 - \hat{\mu}(z) - \frac{w_t}{\sqrt{1 - w_t^2}} \hat{\sigma}(z) \epsilon_0                         \notag             \\
                             & = x_0 - \hat{\mu}(z) - \frac{\sqrt{\bar{\alpha}_t}}{\sqrt{1 - \bar{\alpha}_t}} \hat{\sigma}(z) \epsilon_0. \label{eq:v_t}
\end{align}
From \eqref{eq:v_t}, it follows that in \eqref{eq:xt_wt}, $x_t$ varies at a constant velocity $x_0 - \hat{\mu}(z)$ with respect to an infinitesimal change in $w_t$, ignoring the noise term.

Using \eqref{eq:xt_x0_ep0_f} to eliminate $\epsilon_0$ from \eqref{eq:v_t} and rearranging the expression with respect to $x_0$ yields
\begin{align}
  x_0 = \left(1-\bar{\alpha}_t\right) v_t + \sqrt{\bar{\alpha}_t} x_t + \left(1-\sqrt{\bar{\alpha}_t}\right) \hat{\mu}(z). \label{eq:x0_vpred}
\end{align}
Substituting this expression into \eqref{eq:tilde_mu_x0} and simplifying gives the following result.
\begin{align}
  \tilde{\mu}(x_t, x_0, z) & = \sqrt{\alpha_t} x_t + (1 - \sqrt{\alpha_t}) \hat{\mu}(z) + \left(1 - \alpha_t \right) \sqrt{\bar{\alpha}_{t-1}} v_t.  \label{eq:mu_tilde_flow}
\end{align}
Instead of directly predicting $v_t$ in \eqref{eq:mu_tilde_flow}, we introduce the following modified velocity $\hat{v}_t$:
\begin{align}
  \hat{v}_t := - \frac{\sqrt{1 - \bar{\alpha}_t}\, v_t}{\hat{\sigma}(z)}
  = \sqrt{\bar{\alpha}_t}\,\epsilon_0
  - \sqrt{1 - \bar{\alpha}_t}\,\frac{x_0 - \hat{\mu}(z)}{\hat{\sigma}(z)}.
  \label{eq:v_t_hat}
\end{align}
This corresponds to a normalized version of the standard velocity used in conventional velocity-prediction modeling \citep{salimans2022progressive, lin2024common}, $v_t^{\mathrm{std}} = \sqrt{\bar{\alpha}_t}\,\epsilon_0 - \sqrt{1 - \bar{\alpha}_t}\, x_0$, where the $x_0$ term is normalized by the reconstruction $\hat{\mu}(z)$ and the variance estimate $\hat{\sigma}(z)$ from the prior model. If an appropriate prior model is used, the local scale variations of $x_0$ are expected to be absorbed by this normalization with $\hat{\mu}(z)$ and $\hat{\sigma}(z)$, making regression with $\hat{v}_t$ the prediction target less sensitive to scale variation.

\subsection{Derivation of the loss function}
By replacing $v_t$ in \eqref{eq:mu_tilde_flow} with $\hat{v}_t$, we obtain
\begin{align}
  \tilde{\mu}(x_t, x_0, z) & = \sqrt{\alpha_t} x_t + (1 - \sqrt{\alpha_t}) \hat{\mu}(z) - \frac{\left(1 - \alpha_t \right) \sqrt{\bar{\alpha}_{t-1}}}{\sqrt{1 - \bar{\alpha}_{t}}} \hat{\sigma}(z) \hat{v}_t.  \label{eq:mu_tilde_flow_hat}
\end{align}
Following the structure of this expression, we reparameterize $\mu_\theta(x_t, t, z)$ using $v_\theta(x_t,t, z)$ as
\begin{align}
  \mu_\theta(x_t,t, z) = \sqrt{\alpha_t} x_t + (1 - \sqrt{\alpha_t}) \hat{\mu}(z) - \frac{\left(1 - \alpha_t \right) \sqrt{\bar{\alpha}_{t-1}}}{\sqrt{1 - \bar{\alpha}_{t}}} \hat{\sigma}(z) v_\theta(x_t, t, z). \label{eq:mu_theta_flow}
\end{align}
Under this parameterization, $\mathcal{J}_4$ in the ELBO can be rewritten as
\begin{align}
  \mathcal{J}_4 & = - \sum_{t=2}^T \mathbb{E}_{\hat{q}(z \mid x_0) \mathcal{N}(\epsilon_0 \mid 0, I)}\left[\frac{\left(1 - \alpha_t\right)^2 \bar{\alpha}_{t-1} \hat{\sigma}^2(z)}{2 \tilde{\beta}_t(z) \left(1 - \bar{\alpha}_{t}\right)} \left\|\hat{v}_t - v_\theta(x_t, t, z)\right\|^2 \right] + C_4,
\end{align}
where $C_4$ is a constant independent of $\theta$.

By applying a similar transformation to $\mathcal{J}_3$, we finally obtain the following loss function for maximizing the ELBO of RPD under velocity prediction modeling (the loss coefficient is fixed to 1 following \citep{ho2020denoising}):
\begin{align}
  \ell_\mathrm{simple}^v(\theta; x_0) & = \mathbb{E}_{\mathcal{U}(t \mid 1,T) \hat{q}(z \mid x_0) \mathcal{N}(\epsilon_0 \mid 0, I)} \left[ \left\|\hat{v}_t - v_\theta(x_t, t, z)\right\|^2 \right]. \label{eq:ell_simple_vpred}
\end{align}

\subsection{Algorithms}
The training algorithm for RPD under velocity prediction modeling is summarized in Algorithm~\ref{alg:rpd_training_vpred}, and the inference (generation) procedure using the trained model is given in Algorithm~\ref{alg:rpd_inference_vpred}. The auxiliary variable $\omega_t^v$ appearing in these algorithms is described in detail in Section~\ref{sec:aux_var}.
\begin{figure}[htbp]
  \begin{algorithm}[H]
    \caption{RPD training under velocity prediction modeling}
    \label{alg:rpd_training_vpred}
    \begin{algorithmic}[1]
      \Require Training dataset $\mathcal{D}_\mathrm{train}$
      \Require $\hat{q}(z \mid x_0)$, $\hat{\mu}(z)$, and $\hat{\sigma}(z)$ from the prior model
      \Require Hyperparameters $\left\{\alpha_t \ \middle| \ t=1,\ldots,T\right\}$
      \Ensure Trained model $v_\theta$
      \Repeat
      \State $x_0 \sim \mathcal{D}_\mathrm{train}$, $t \sim \mathcal{U}(t \mid 1,T)$, $\epsilon_0 \sim \mathcal{N}(\epsilon_0 \mid 0, I)$
      \State $z \sim \hat{q}(z \mid x_0)$
      \State $x_t = \sqrt{\bar{\alpha}_t} x_0 + \left(1 - \sqrt{\bar{\alpha}_t}\right) \hat{\mu}(z) + \sqrt{1 - \bar{\alpha}_t} \hat{\sigma}(z) \epsilon_0$ \Comment{Eq.~\eqref{eq:xt_x0_ep0_f}}
      \State $\omega_t^v = \frac{\sqrt{\bar{\alpha}_t}}{\sqrt{1 - \bar{\alpha}_t}} \frac{x_t - \hat{\mu}(z)}{\hat{\sigma}(z)}$ \Comment{Eq.~\eqref{eq:omega_t_v_z}}
      \State $\hat{v}_t = \sqrt{\bar{\alpha}_t} \epsilon_0 - \sqrt{1 - \bar{\alpha}_t} \frac{x_0 - \hat{\mu}(z)}{\hat{\sigma}(z)}$ \Comment{Eq.~\eqref{eq:v_t_hat}}
      \State Take a gradient-descent step on:
      $\nabla_\theta \left\|\hat{v}_t - v_\theta(x_t, t, z, \omega_t^v)\right\|^2$
      \Until{convergence}
    \end{algorithmic}
  \end{algorithm}
  \begin{algorithm}[H]
    \caption{RPD inference under velocity prediction modeling}
    \label{alg:rpd_inference_vpred}
    \begin{algorithmic}[1]
      \Require $\hat{p}(z)$, $\hat{\mu}(z)$, and $\hat{\sigma}(z)$ from the prior model
      \Require Hyperparameters $\left\{\alpha_t \ \middle| \ t=1,\ldots,T\right\}$
      \Require Trained model $v_\theta$
      \Ensure Generated sample $x_0$
      \State $z \sim \hat{p}(z)$
      \State $x_T \sim \mathcal{N}\left(x_T \ \middle| \ \hat{\mu}(z), \hat{\sigma}^2(z) I \right)$
      \For{$t = T$ to $1$}
      \State $\omega_t^v = \frac{\sqrt{\bar{\alpha}_t}}{\sqrt{1 - \bar{\alpha}_t}} \frac{x_t - \hat{\mu}(z)}{\hat{\sigma}(z)}$ \Comment{Eq.~\eqref{eq:omega_t_v_z}}
      \State $\mu_\theta(x_t,t, z) = \sqrt{\alpha_t} x_t + (1 - \sqrt{\alpha_t}) \hat{\mu}(z) - \frac{\left(1 - \alpha_t \right) \sqrt{\bar{\alpha}_{t-1}}}{\sqrt{1 - \bar{\alpha}_{t}}} \hat{\sigma}(z) v_\theta(x_t, t, z, \omega_t^v)$ \Comment{Eq.~\eqref{eq:mu_theta_flow}}
      \State $\sigma_t(z) = \sqrt{\frac{(1 - \alpha_t)(1 - \bar{\alpha}_{t-1})}{1 - \bar{\alpha}_t}} \hat{\sigma}(z)$ if $t \geq 2$ else $\sigma_t(z) = \sigma_{\min}$  \Comment{Eq.~\eqref{eq:tilde_beta}}
      \State $x_{t-1} \sim \mathcal{N}\left(x_{t-1} \ \middle| \ \mu_\theta(x_t, t, z), \sigma_t^2(z) I \right)$
      \EndFor
    \end{algorithmic}
  \end{algorithm}
\end{figure}

\section{Accelerating prediction in RPD using auxiliary variables} \label{sec:aux_var}
In noise-prediction modeling, the objective is to estimate $\epsilon_0$, whereas in velocity-prediction modeling, the objective is to estimate the modified velocity $\hat{v}_t$, both conditioned on $x_t$, $t$, and $z$. This section presents a method for improving the efficiency of these prediction tasks by introducing auxiliary variables---computed from $x_t$ and $z$ together with the prior model---as additional inputs to the predictor. We theoretically demonstrate that these auxiliary variables converge toward their respective prediction targets as the reconstruction accuracy of the prior model increases, thereby facilitating more efficient prediction.

\subsection{Auxiliary variables for noise-prediction modeling}
We introduce the following composite variable $\omega_t^\epsilon$ and incorporate it as an additional input for predicting $\epsilon_0$ in noise-prediction modeling.
\begin{align}
  \omega_t^\epsilon := \frac{1}{\sqrt{1 - \bar{\alpha}_t}} \frac{x_t - \hat{\mu}(z)}{\hat{\sigma}(z)}. \label{eq:omega_t_epsilon}
\end{align}
The following relationship holds between the composite variable $\omega_t^\epsilon$ and the prediction target $\epsilon_0$.
\begin{proposition} \label{prop:eps_omega_diff}
  Given $(x_0, z, t) \sim p_\mathrm{gt}(x_0)\, \hat{q}(z \mid x_0)\, \mathcal{U}(t \mid 1, T)$, the following holds:
  \begin{align}
    \mathbb{E}_{\mathcal{N}(\epsilon_0 \mid 0, I)}\left[\left\|\epsilon_0 - \omega_t^\epsilon\right\|^2\right] = \frac{\bar{\alpha}_t \left\|x_0 - \hat{\mu}(z)\right\|^2}{\left(1 - \bar{\alpha}_t\right) \hat{\sigma}^2(z)}. \label{eq:eps_omega_diff}
  \end{align}
\end{proposition}
\begin{proof}
  The claim in \eqref{eq:eps_omega_diff} follows immediately from the following calculation:
  \begin{align*}
    \epsilon_0 - \omega_t^\epsilon & = \epsilon_0 - \frac{x_t - \hat{\mu}(z)}{\sqrt{1 - \bar{\alpha}_t} \hat{\sigma}(z)}                                                                                                                                                 \\
                                   & = \epsilon_0 - \frac{\sqrt{\bar{\alpha}_t} x_0 + \left(1 - \sqrt{\bar{\alpha}_t}\right) \hat{\mu}(z) + \sqrt{\left(1 - \bar{\alpha}_t\right)} \hat{\sigma}(z) \epsilon_0 - \hat{\mu}(z)}{\sqrt{1 - \bar{\alpha}_t} \hat{\sigma}(z)} \\
                                   & = - \frac{\sqrt{\bar{\alpha}_t}\left(x_0 - \hat{\mu}(z)\right)}{\sqrt{1 - \bar{\alpha}_t} \hat{\sigma}(z)}.
  \end{align*}
\end{proof}
Because $\hat{\mu}(z)$ (with $z \sim \hat{q}(z \mid x_0)$) can be interpreted as a reconstruction of $x_0$ produced by the prior model, Proposition~\ref{prop:eps_omega_diff} indicates that $\omega_t^\epsilon$ converges toward $\epsilon_0$ as the reconstruction accuracy improves. This suggests that exploiting the reconstruction capability of the prior model can enhance the efficiency of noise prediction in RPD when $\omega_t^\epsilon$ is used as an additional input.
Note that $\hat{\mu}(z)$ and $\hat{\sigma}^2(z)$ are determined solely by maximizing the ELBO of the prior model \eqref{eq:elbo_prior} and are not adjusted to bring $\omega_t^\epsilon$ closer to $\epsilon_0$.

While, in principle, the expressive power of neural networks allows $\epsilon_\theta$ to approximate an equivalent function without explicitly providing $\omega_t^\epsilon$, the experimental results in Section~\ref{sec:experiment} show that explicitly including $\omega_t^\epsilon$ as an input to $\epsilon_\theta$ substantially accelerates training, even when $\epsilon_\theta$ is implemented as a neural network.

\subsection{Auxiliary variables for velocity-prediction modeling}
For the prediction task in RPD under velocity-prediction modeling, we introduce the following composite variable $\omega_t^v$ as an additional input to the predictor.
\begin{align}
  \omega_t^v := \frac{\sqrt{\bar{\alpha}_t}}{\sqrt{1 - \bar{\alpha}_t}} \frac{x_t - \hat{\mu}(z)}{\hat{\sigma}(z)}. \label{eq:omega_t_v_z}
\end{align}
The following relationship holds between the composite variable $\omega_t^v$ and the prediction target $\hat{v}_t$.
\begin{proposition} \label{prop:v_omega_diff}
  Given $(x_0, z, t) \sim p_\mathrm{gt}(x_0)\, \hat{q}(z \mid x_0)\, \mathcal{U}(t \mid 1, T)$, the following holds:
  \begin{align}
    \mathbb{E}_{\mathcal{N}(\epsilon_0 \mid 0, I)}\left[\left\|\hat{v}_t - \omega_t^v\right\|^2\right] & = \frac{\left\|x_0 - \hat{\mu}(z)\right\|^2}{\left(1 - \bar{\alpha}_t\right) \hat{\sigma}^2(z)}. \label{eq:v_omega_diff}
  \end{align}
\end{proposition}
\begin{proof}
  The claim in \eqref{eq:v_omega_diff} follows immediately from the calculation below:
  \begin{align*}
    \hat{v}_t - \omega_t^v & = \sqrt{\bar{\alpha}_t} \epsilon_0 - \sqrt{1 - \bar{\alpha}_t} \frac{x_0 - \hat{\mu}(z)}{\hat{\sigma}(z)} - \frac{\sqrt{\bar{\alpha}_t}}{\sqrt{1 - \bar{\alpha}_t}} \frac{x_t - \hat{\mu}(z)}{\hat{\sigma}(z)}                                 \\
                           & = \sqrt{\bar{\alpha}_t} \epsilon_0 - \sqrt{1 - \bar{\alpha}_t} \frac{x_0 - \hat{\mu}(z)}{\hat{\sigma}(z)}                                                                                                                                      \\
                           & \qquad - \frac{\sqrt{\bar{\alpha}_t}}{\sqrt{1 - \bar{\alpha}_t}} \frac{\sqrt{\bar{\alpha}_t} x_0 + \left(1 - \sqrt{\bar{\alpha}_t}\right) \hat{\mu}(z) + \sqrt{1 - \bar{\alpha}_t} \hat{\sigma}(z) \epsilon_0 - \hat{\mu}(z)}{\hat{\sigma}(z)} \\
                           & = - \frac{1}{\sqrt{1 - \bar{\alpha}_t}} \frac{x_0 - \hat{\mu}(z)}{\hat{\sigma}(z)}.
  \end{align*}
\end{proof}
Proposition~\ref{prop:v_omega_diff} indicates that, as the reconstruction error of $x_0$ under the prior model decreases, $\omega_t^v$ converges toward $\hat{v}_t$, thereby making the prediction of $\hat{v}_t$ easier when $\omega_t^v$ is used as an input. This suggests that, in velocity-prediction modeling as well, introducing an appropriate auxiliary variable can enhance the efficiency of prediction within RPD.

\subsection{Incorporation into the algorithms}
The specific procedures for using the auxiliary variables $\omega_t^\epsilon$ and $\omega_t^v$ are provided in Algorithms~\ref{alg:rpd_training}--\ref{alg:rpd_inference_vpred}. To avoid numerical instability arising from division by $\sqrt{1 - \bar{\alpha}_t}$ when computing $\omega_t^\epsilon$ and $\omega_t^v$, we replace $\sqrt{1 - \bar{\alpha}_t}$ with $\sqrt{1 - \bar{\alpha}_t} + \delta$ in the implementation (with $\delta = 0.01$ in our experiments).

\section{Experiments} \label{sec:experiment}
This section summarizes numerical experiments comparing RPD with related existing methods. First, we present results on two-dimensional hetero-scale synthetic examples in which the global and local scales of the distribution differ. Then, we report an evaluation on generative modeling tasks involving natural images.

\subsection{Experiments on two-dimensional hetero-scale data}
\subsubsection{Dataset}
As two-dimensional data, we used a processed version of the Datasaurus dataset \citep{datasaurus}. The Datasaurus dataset contains 13 distinct two-dimensional distributions, all of which share the same mean and standard deviation along both the $x$- and $y$-axes. Because the original data lie roughly within the range $[0,100]$ on both axes, we normalized them by subtracting $50$ from each coordinate and dividing by $50$, yielding data roughly within the range $[-1,1]$. We then constructed the following two datasets based on these normalized data.

\paragraph{Datasaurus-Grid (scale $s$)}
We selected nine of the 13 distributions in the Datasaurus dataset and arranged them on an evenly spaced grid with spacing $\Delta = 3$. Before placing them on the grid, each distribution was scaled by a factor $s$. When $s$ was small relative to $\Delta$, each grid cell contained a highly localized structure relative to the global layout. An example with $s = 0.1$ is shown in Fig.~\ref{fig:datasaurus-grid-0.1-GT}; the top row shows the full distribution, and the bottom row shows zoomed-in views of the distributions placed at each grid location.

\paragraph{Datasaurus-Grid (hetero-scale)}
Here, we arranged nine distributions on the grid as above but scaled each of them by a different factor. This produced a more complex dataset in which the local scales vary across grid cells. An example of the Datasaurus-Grid (hetero-scale) dataset is shown in Fig.~\ref{fig:datasaurus-grid-heteroscale-GT}.

\subsubsection{Compared models} \label{sec:compared_models}
The experiments using two-dimensional data compared the following models: Denoising Diffusion Probabilistic Models (DDPM) \citep{ho2020denoising}, $v$-prediction \citep{salimans2022progressive, lin2024common}, DiffuseVAE \citep{pandey2022diffusevae}, and Rectified Flow \citep{liuflow}.\footnote{We excluded an LDM-style baseline for Datasaurus-Grid, as we found it difficult to obtain a sufficiently accurate autoencoder for the dataset.}
For DiffuseVAE, we used the formulation based on Formulation~2 proposed in \citet{pandey2022diffusevae}.
We refer to the models proposed in this work as RPD and RPD\_vpred, corresponding to noise-prediction and velocity-prediction modeling, respectively.

We chose vector quantized VAE (VQVAE) \citep{van2017neural} as the VAE model and prior model for both DiffuseVAE and RPD.
The details of the VQVAE training are provided in \ref{apdx:vqvae}. For the comparison between RPD and DiffuseVAE, we used the same VQVAE model trained under an identical procedure. However, because DiffuseVAE does not perform variance estimation in its decoder, we employed a variant in which $\hat{\sigma}(z)$ in \eqref{eq:prior_x0_z} is fixed to $1$.\footnote{We also experimented with an extended version of DiffuseVAE that incorporates variance estimation through $\hat{\sigma}(z)$, but the resulting models produced unstable generations. Therefore, following the original formulation, we used the variant with the variance fixed to $1$.}

\paragraph{DNN model configurations}
For each model, the network used to predict the diffusion dynamics was implemented as a multilayer perceptron (MLP) augmented to take the timestep $t$ as an additional input. In our implementation, we used the \texttt{TimeInputMLP} module\footnote{\url{https://github.com/yuanchenyang/smalldiffusion}
  (accessed 2025-12-08).\label{fn:smalldiffusion}} to construct this network. The MLP consisted of five hidden layers with 64, 128, 256, 126, and 64 units, respectively.

For the VQVAE used in DiffuseVAE and RPD, both the encoder and decoder were implemented as three-layer MLPs with 16, 32, and 16 hidden dimensions, respectively, together with GELU activation \citep{hendrycks2016gelu}. Note that the VQVAE architecture was substantially smaller than that employed for the diffusion process.

\paragraph{Diffusion schedule configuration}
The maximum number of diffusion steps was set to $T=200$, and the noise schedule $\beta_t$ followed a log-linear scheme \citep{pmlr-v235-permenter24a}. In our implementation, we used the \texttt{ScheduleLogLinear} module\footref{fn:smalldiffusion} to generate the schedule. Unless otherwise specified, the parameters of the log-linear schedule were set to \texttt{sigma\_min}=0.01 and \texttt{sigma\_max}=100.

\paragraph{Training configurations}
We used the Adam optimizer \citep{adam} with a learning rate of $0.001$. The mini-batch size was set to $1278$, matching the size of the Datasaurus-Grid dataset. Because DiffuseVAE and RPD require pretraining of the VQVAE, the number of training iterations was adjusted accordingly. Specifically, DDPM, $v$-prediction, and Rectified Flow were trained for $120000$ iterations, whereas DiffuseVAE and RPD were trained for $60000$ iterations. Since the VQVAE itself was pretrained for $15000$ iterations (see \ref{apdx:vqvae}), the total number of training iterations for DiffuseVAE and RPD was smaller overall.

\subsubsection{Evaluation metrics for 2D generation}
For the experiments using the Datasaurus--Grid datasets, we generated $6390$ samples\footnote{This corresponds to five times the dataset size of $1278$.} with each trained model and evaluated their distributional proximity to the ground-truth (GT) dataset.\footnote{Because the Datasaurus dataset contains too few samples to allow for a meaningful split, we constructed the GT dataset by simply replicating the training data five times.} As the primary measure of distributional discrepancy, we employed the 1-Wasserstein distance \citep{villani2008optimal} (1WD). However, the standard 1WD does not sufficiently capture local discrepancies within the Datasaurus--Grid dataset. Therefore, we additionally used the region-wise 1WD (RW-1WD). Specifically, we partitioned the two-dimensional plane into three intervals along each axis, $(-\infty, -\Delta/2)$, $[-\Delta/2, \Delta/2)$, and $[\Delta/2, \infty)$, forming nine regions in total. For each region, we computed the 1WD between the generated samples and the GT dataset and defined RW-1WD as the average of these nine values.

To compute these metrics, we repeated each experiment five times with different random seeds and averaged the results. The components primarily affected by the random seeds were the initial weights of the DNN models, the Gaussian noise used during sample generation in the diffusion models, and the VQVAE models pretrained in DiffuseVAE and RPD. Note that for each seed, the same pretrained VQVAE model was shared between DiffuseVAE and RPD, except that $\hat{\sigma}(z)$ was fixed to $1$ in DiffuseVAE, as mentioned above in Section~\ref{sec:compared_models}.

\subsubsection{Qualitative evaluation of 2D generated samples}
\paragraph{Qualitative results for Datasaurus-Grid (scale=0.1)}
Fig.~\ref{fig:datasaurus-grid-0.1} presents the samples generated by DDPM and RPD trained on the Datasaurus-Grid (scale=0.1) dataset. For reference, we also include samples generated by the VQVAE used as the prior model in RPD. The top row of the figure illustrates the overall distribution, whereas the bottom row provides magnified views of each grid location (see the axis labels of each subplot for the visualization range).
\begin{figure}[tbp]
  \centering
  \begin{minipage}[b]{0.24\linewidth}
    \centering
    \includegraphics[width=\linewidth,clip]{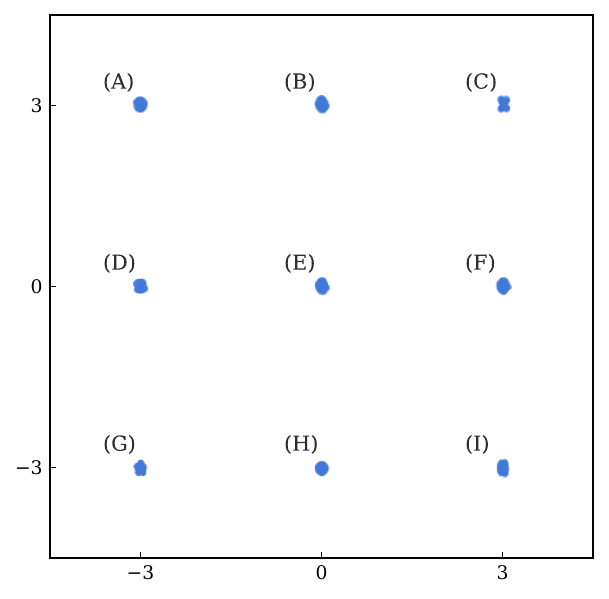}
    \\
    \includegraphics[width=\linewidth,clip]{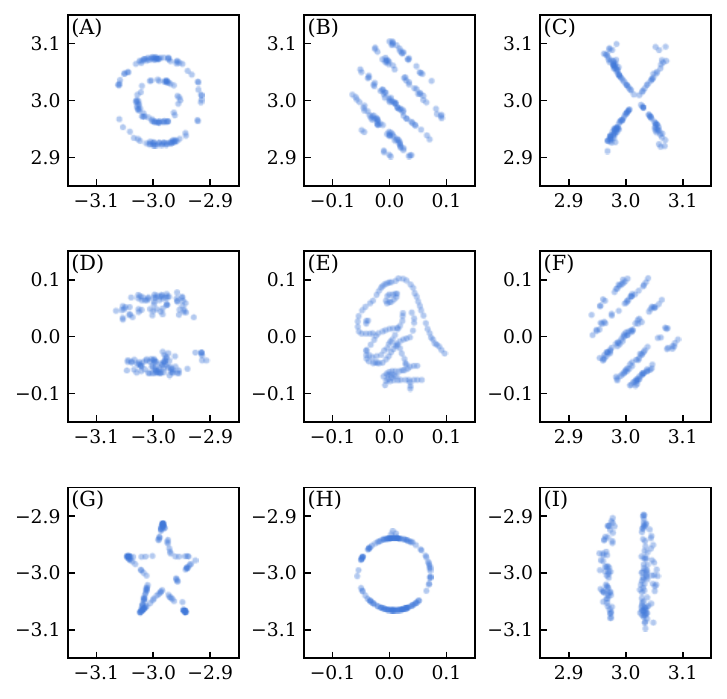}
    \subcaption{GT}
    \label{fig:datasaurus-grid-0.1-GT}
  \end{minipage} \hfill
  \begin{minipage}[b]{0.24\linewidth}
    \centering
    \includegraphics[width=\linewidth,clip]{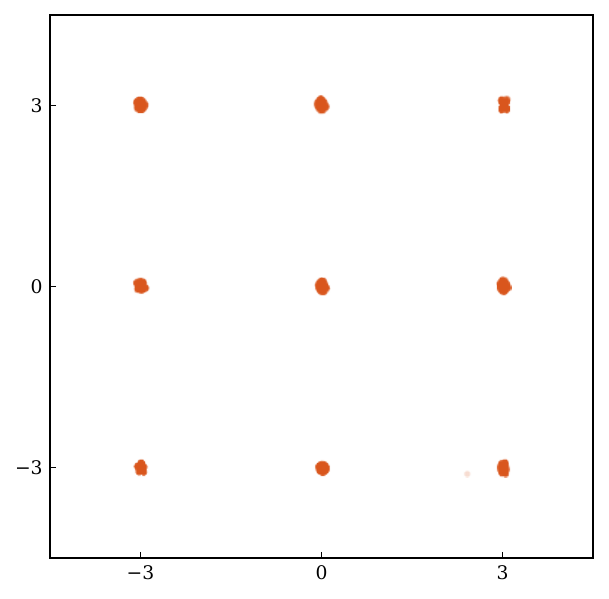}
    \\
    \includegraphics[width=\linewidth,clip]{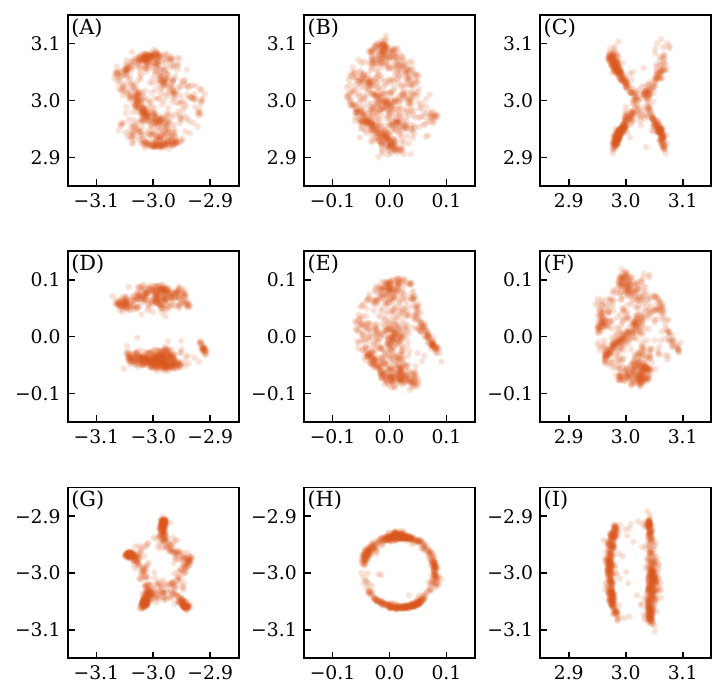}
    \subcaption{DDPM}
    \label{fig:datasaurus-grid-0.1-DDPM}
  \end{minipage} \hfill
  \begin{minipage}[b]{0.24\linewidth}
    \centering
    \includegraphics[width=\linewidth,clip]{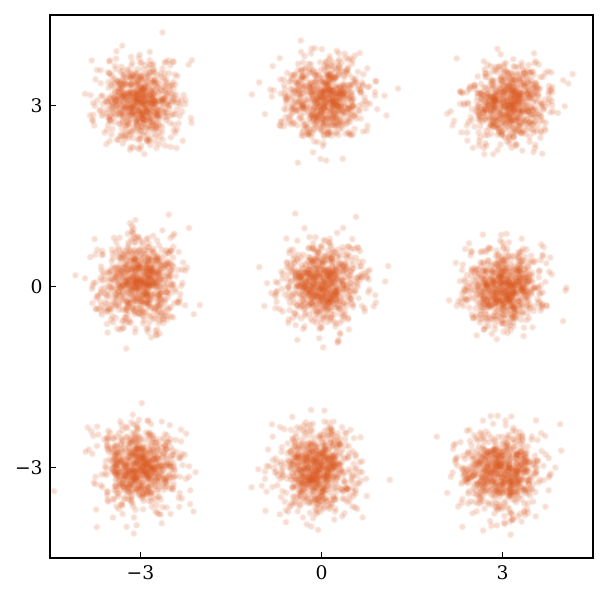}
    \\
    \includegraphics[width=\linewidth,clip]{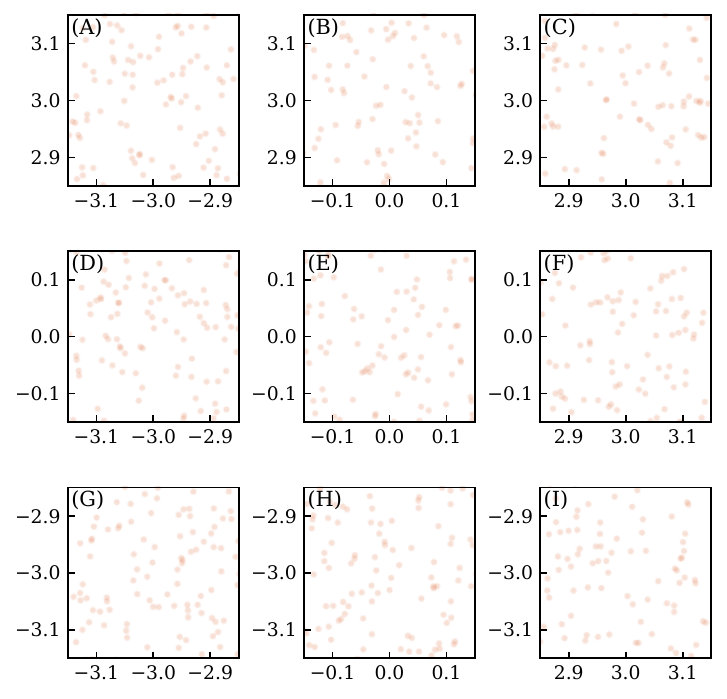}
    \subcaption{VQVAE (prior model)}
    \label{fig:datasaurus-grid-0.1-VQVAE}
  \end{minipage} \hfill
  \begin{minipage}[b]{0.24\linewidth}
    \centering
    \includegraphics[width=\linewidth,clip]{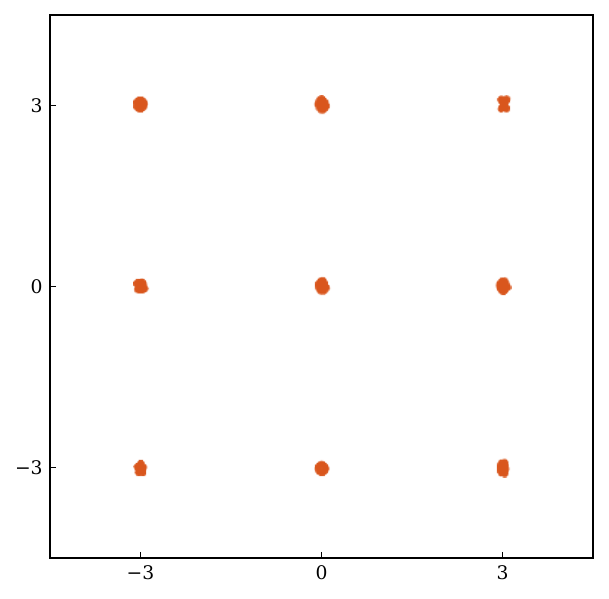}
    \\
    \includegraphics[width=\linewidth,clip]{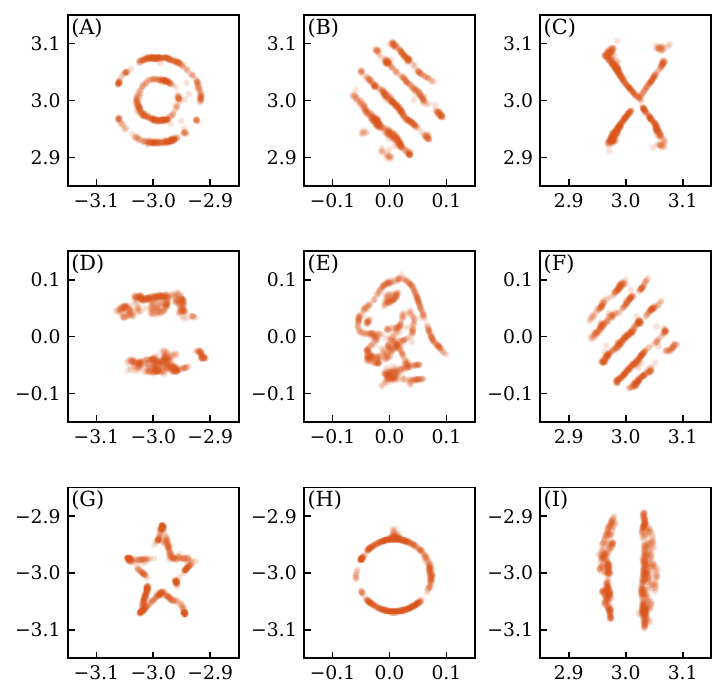}
    \subcaption{RPD (proposed)}
    \label{fig:datasaurus-grid-0.1-RPD}
  \end{minipage}
  \caption{Datasaurus-Grid (scale=0.1) dataset (a) and samples generated by each model (b--d). The top row shows the full distribution, and the bottom row shows zoomed-in views of the distributions at each grid location (the correspondence between rows follows indices A--I in (a)). RPD uses the VQVAE model in (c) as its prior.}
  \label{fig:datasaurus-grid-0.1}
\end{figure}
As shown in Fig.~\ref{fig:datasaurus-grid-0.1-DDPM}, DDPM captures the global trend of the distribution reasonably well. However, as highlighted in the magnified views in the bottom row, it fails to capture most of the fine-grained structures present in the GT distribution, except for simple shapes such as circles. The VQVAE results in Fig.~\ref{fig:datasaurus-grid-0.1-VQVAE} provide only a coarse approximation of the global structure. Because the VQVAE relies on a Gaussian reconstruction model, it is unable to capture the detailed structure of the GT distribution.
In contrast, RPD, which uses the VQVAE as its prior model, successfully preserves the global structure while also generating local patterns that closely match the GT distribution, as shown in Fig.~\ref{fig:datasaurus-grid-0.1-RPD}. These results demonstrate that even when the prior model offers only a coarse approximation of the distribution, RPD can learn to capture fine-grained local structures with high fidelity.

\paragraph{Qualitative results for Datasaurus-Grid (scale=1.0)}
The evaluation results for the Datasaurus-Grid (scale=1.0) dataset are presented in Fig.~\ref{fig:datasaurus-grid-1.0}. Note that magnified views for each grid location are omitted in this figure. Compared with the scale=0.1 setting, the DDPM results in Fig.~\ref{fig:datasaurus-grid-1.0-DDPM} show that DDPM is able to express a considerable amount of detail in each shape. This indicates that, when the global and local scales of the data distribution are not substantially different, DDPM can capture the original distribution with relatively high fidelity.
The VQVAE results in Fig.~\ref{fig:datasaurus-grid-1.0-VQVAE} reveal that, for example, around $(-3,-3)$, the model estimates a Gaussian mixture with multiple means. As discussed in \ref{apdx:vqvae}, this behavior arises from the choice of a codebook size of $16$ in the VQVAE and indicates that the Gaussian mean can take one of sixteen possible vectors. Nevertheless, as shown in Fig.~\ref{fig:datasaurus-grid-1.0-RPD}, RPD remains capable of stable distribution learning even when such a VQVAE is used as the prior model.

\begin{figure}[tbp]
  \centering
  \begin{minipage}[b]{0.24\linewidth}
    \centering
    \includegraphics[width=\linewidth,clip]{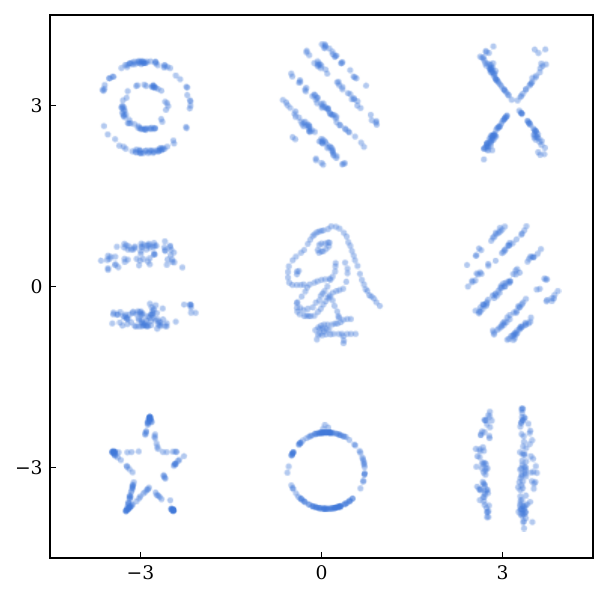}
    \subcaption{GT}
    \label{fig:datasaurus-grid-1.0-GT}
  \end{minipage} \hfill
  \begin{minipage}[b]{0.24\linewidth}
    \centering
    \includegraphics[width=\linewidth,clip]{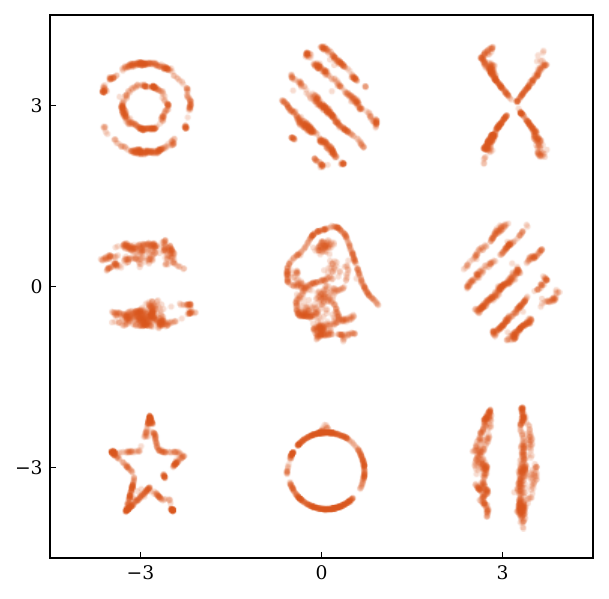}
    \subcaption{DDPM}
    \label{fig:datasaurus-grid-1.0-DDPM}
  \end{minipage} \hfill
  \begin{minipage}[b]{0.24\linewidth}
    \centering
    \includegraphics[width=\linewidth,clip]{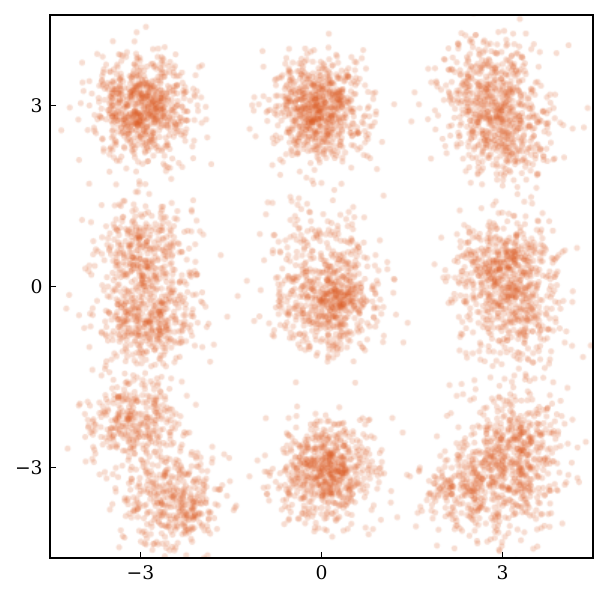}
    \subcaption{VQVAE (prior model)}
    \label{fig:datasaurus-grid-1.0-VQVAE}
  \end{minipage} \hfill
  \begin{minipage}[b]{0.24\linewidth}
    \centering
    \includegraphics[width=\linewidth,clip]{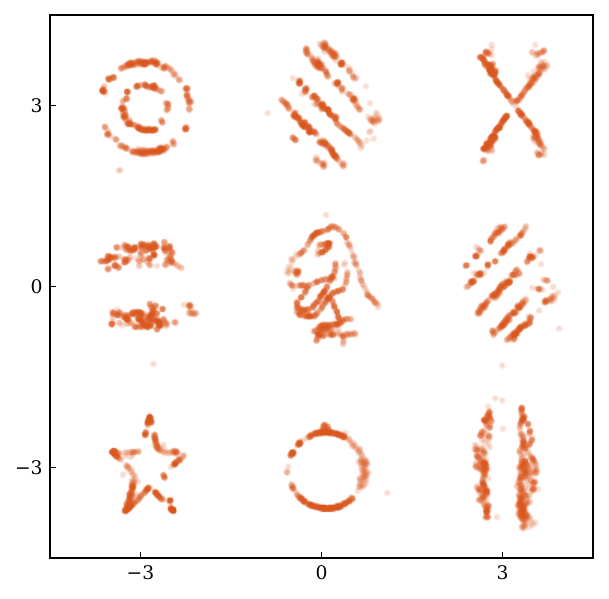}
    \subcaption{RPD (proposed)}
    \label{fig:datasaurus-grid-1.0-RPD}
  \end{minipage}
  \caption{Datasaurus-Grid (scale=1.0) dataset (a) and samples generated by each model (b--d). RPD uses the VQVAE model in (c) as its prior.}
  \label{fig:datasaurus-grid-1.0}
\end{figure}

\paragraph{Qualitative results for Datasaurus-Grid (hetero-scale)}
The generated samples for the Datasaurus-Grid (hetero-scale) dataset are shown in Fig.~\ref{fig:datasaurus-grid-multiscale}. As can be seen in Fig.~\ref{fig:datasaurus-grid-heteroscale-DDPM}, DDPM is able to express the GT distribution reasonably well in regions with larger local scales. However, in regions with smaller local scales, DDPM struggles to capture fine-grained details of the distribution, except for simple shapes such as circles.
In contrast, as shown in Fig.~\ref{fig:datasaurus-grid-heteroscale-RPD}, RPD is capable of accurately capturing detailed structures even when the local scales vary substantially across regions. This demonstrates that RPD maintains high fidelity in settings where the data exhibit heterogeneous local scales.

\begin{figure}[tbp]
  \centering
  \begin{minipage}[b]{0.24\linewidth}
    \centering
    \includegraphics[width=\linewidth,clip]{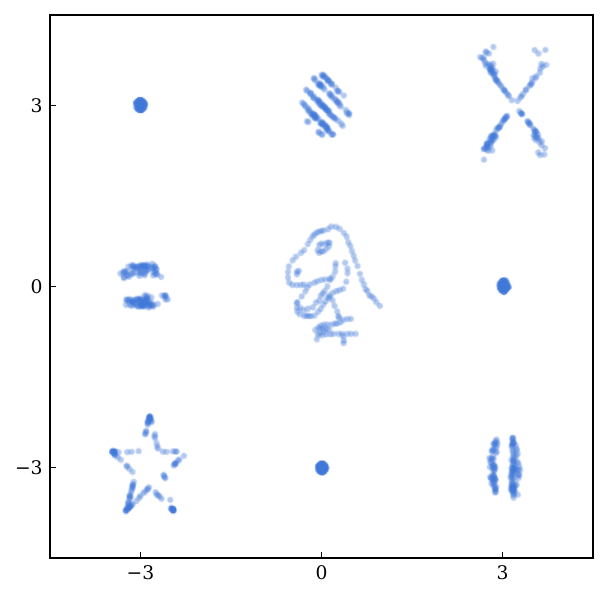}
    \\
    \includegraphics[width=\linewidth,clip]{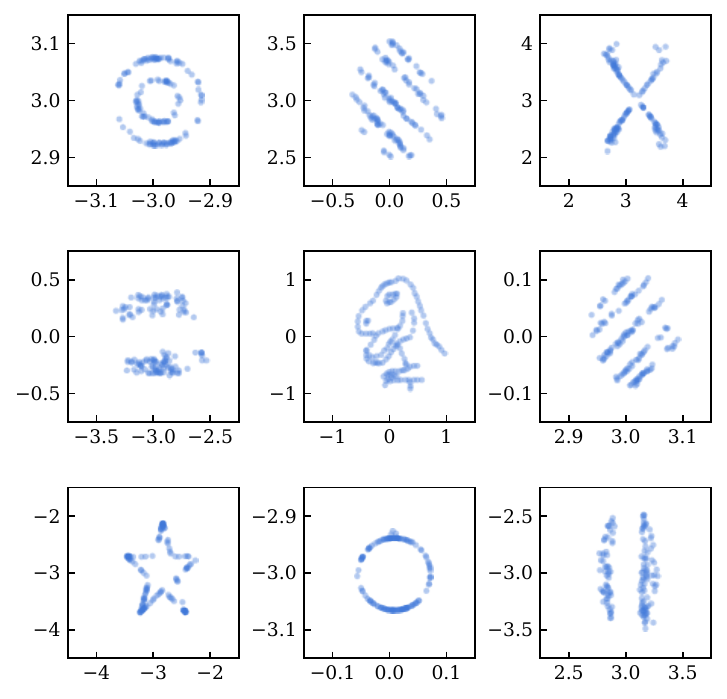}
    \subcaption{GT}
    \label{fig:datasaurus-grid-heteroscale-GT}
  \end{minipage} \hfill
  \begin{minipage}[b]{0.24\linewidth}
    \centering
    \includegraphics[width=\linewidth,clip]{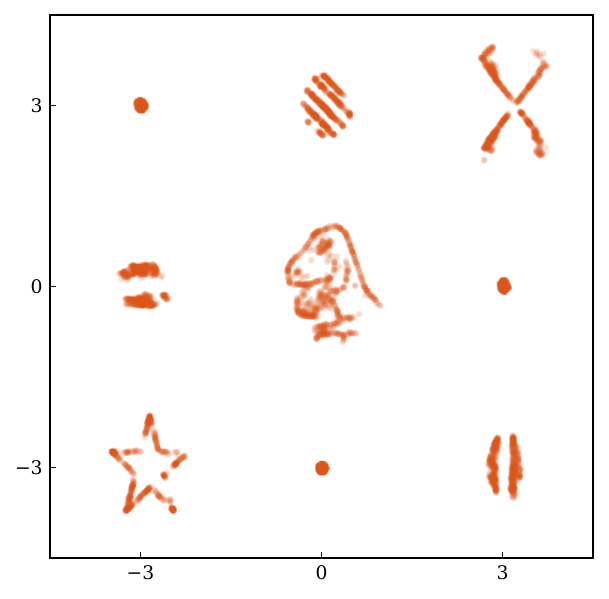}
    \\
    \includegraphics[width=\linewidth,clip]{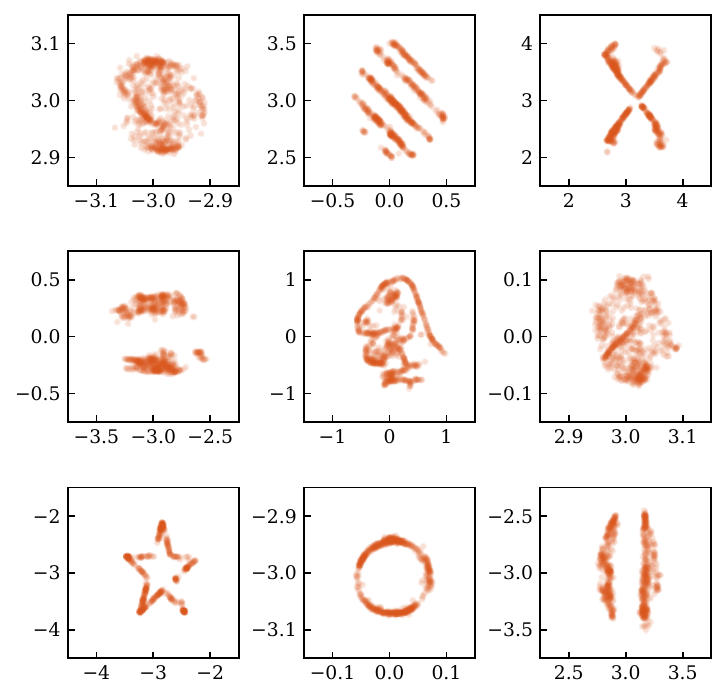}
    \subcaption{DDPM}
    \label{fig:datasaurus-grid-heteroscale-DDPM}
  \end{minipage} \hfill
  \begin{minipage}[b]{0.24\linewidth}
    \centering
    \includegraphics[width=\linewidth,clip]{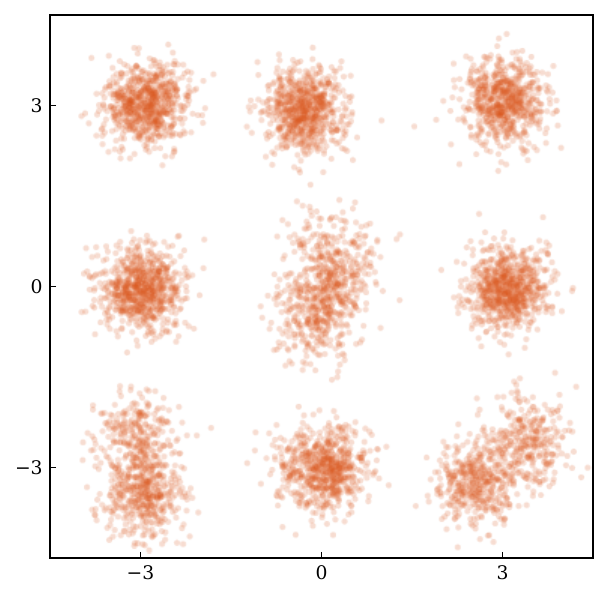}
    \\
    \includegraphics[width=\linewidth,clip]{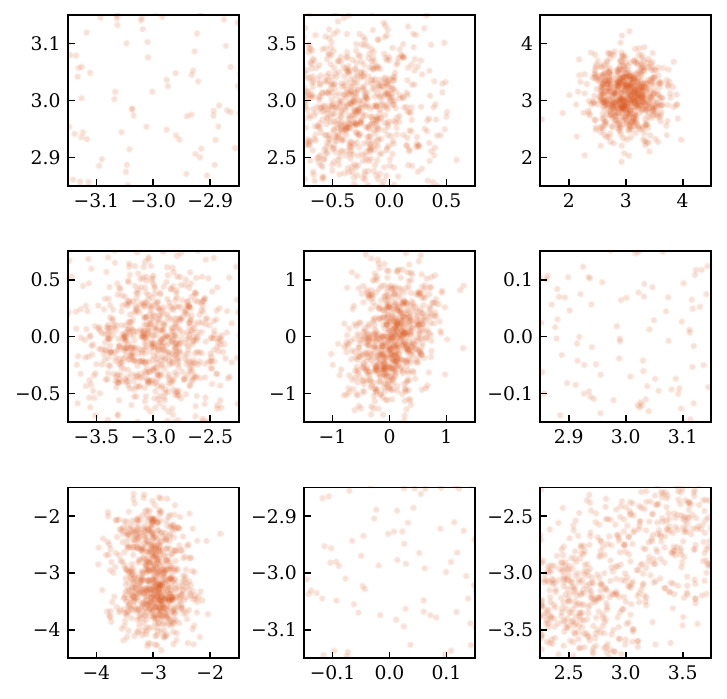}
    \subcaption{VQVAE (prior model)}
    \label{fig:datasaurus-grid-heteroscale-VQVAE}
  \end{minipage} \hfill
  \begin{minipage}[b]{0.24\linewidth}
    \centering
    \includegraphics[width=\linewidth,clip]{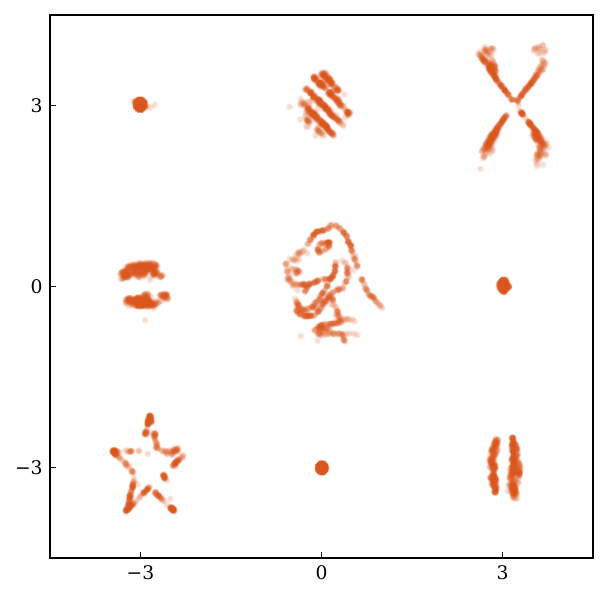}
    \\
    \includegraphics[width=\linewidth,clip]{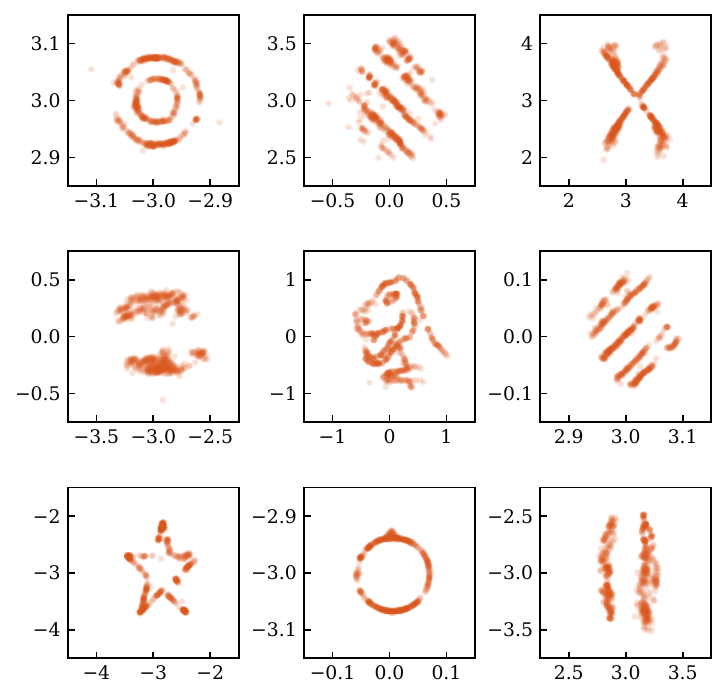}
    \subcaption{RPD (proposed)}
    \label{fig:datasaurus-grid-heteroscale-RPD}
  \end{minipage}
  \caption{Datasaurus-Grid (hetero-scale) dataset (a) and samples generated by each model (b--d). The top row shows the global distribution, and the bottom row shows zoomed-in views at each grid location (note that the scale differs across zoomed-in panels). RPD uses the VQVAE model in (c) as its prior.}
  \label{fig:datasaurus-grid-multiscale}
\end{figure}

\paragraph{Qualitative comparison with other models}
Qualitative results obtained using DiffuseVAE and Rectified Flow are shown in Figs.~\ref{fig:datasaurus-grid-diffusevae} and \ref{fig:datasaurus-grid-rf} in \ref{apdx:2d_sample}, respectively, for the Datasaurus-Grid datasets (scale=0.1, 1.0, and hetero-scale).
As can be seen, the generation characteristics of DiffuseVAE and Rectified Flow are similar to those of DDPM: their generation accuracy in regions with small local scales is limited compared with regions where the local scale is larger. Despite the methodological similarity between RPD and DiffuseVAE, the performance gap between the two models can be attributed to the differences discussed in Section~\ref{sec:compare_with_diffusevae}.

\paragraph{Visualization of the reverse diffusion process}
Fig.~\ref{fig:datasaurus-0.1-reverse} visualizes the reverse diffusion trajectories for each model trained on the Datasaurus-Grid (scale=0.1) dataset. In the figure, the light gray points represent the distribution of $x_t$ at $t=200$ ($=T$), corresponding to the initial distribution of the reverse diffusion process, while the dark gray points represent the distribution of $x_t$ at $t=0$, i.e., the estimated $x_0$ distribution produced by each model. The colored curves denote the reverse-time trajectories from $t=200$ to $t=0$ for individual generated samples. For DiffuseVAE (Fig.~\ref{fig:datasaurus-0.1-reverse-diffusevae}), the red curve corresponds to the post-processing step in which the decoded mean $\hat{\mu}(z)$ of the latent variable is subtracted at the end of generation (see Section~\ref{sec:compare_with_diffusevae} and \citep{pandey2022diffusevae} for details).

\begin{figure}[tbp]
  \centering
  \begin{minipage}[t]{0.48\linewidth}
    \centering
    \includegraphics[height=\linewidth,clip]{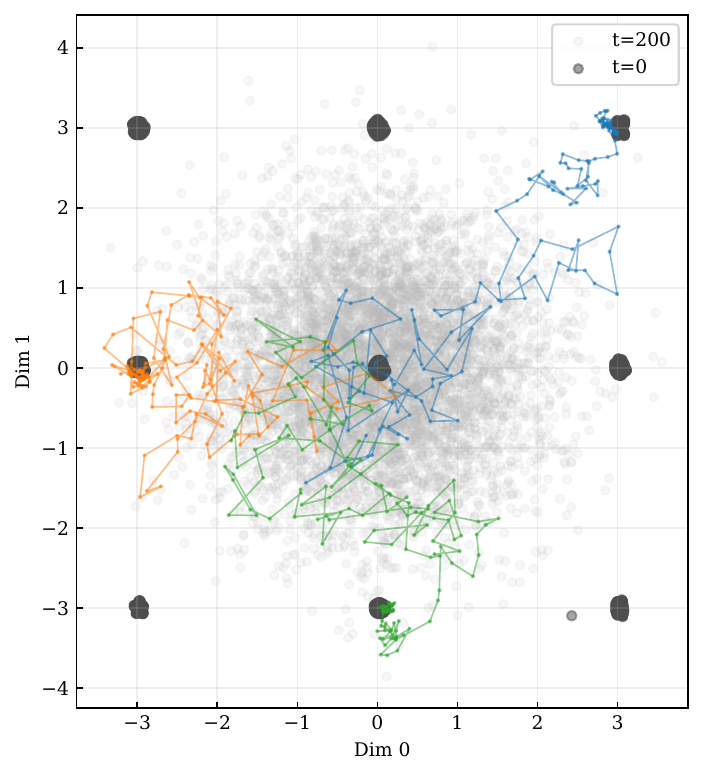}
    \subcaption{DDPM}
    \label{fig:datasaurus-0.1-reverse-ddpm}
  \end{minipage} \hfill
  \begin{minipage}[t]{0.48\linewidth}
    \centering
    \includegraphics[height=\linewidth,clip]{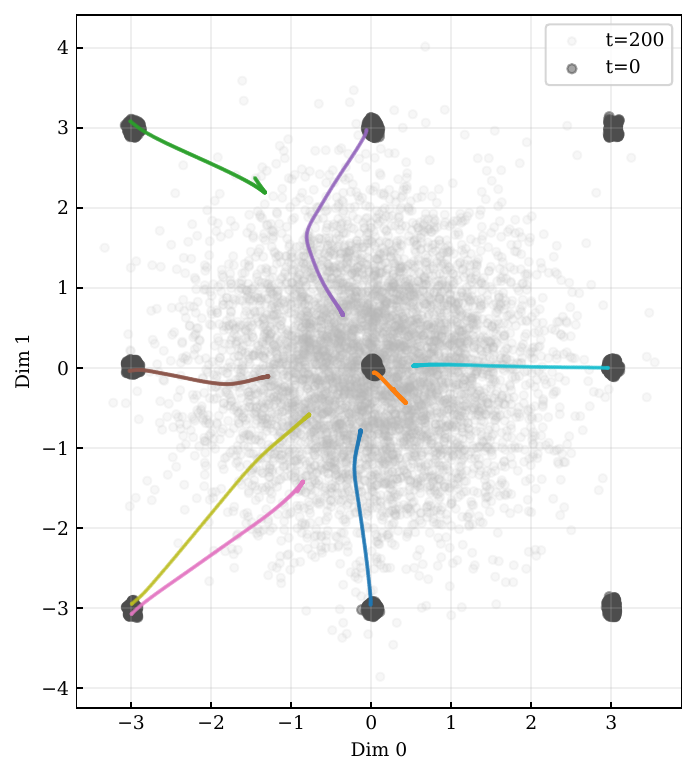}
    \subcaption{Rectified Flow}
    \label{fig:datasaurus-0.1-reverse-rectified-flow}
  \end{minipage}
  \begin{minipage}[t]{0.48\linewidth}
    \centering
    \includegraphics[height=\linewidth,clip]{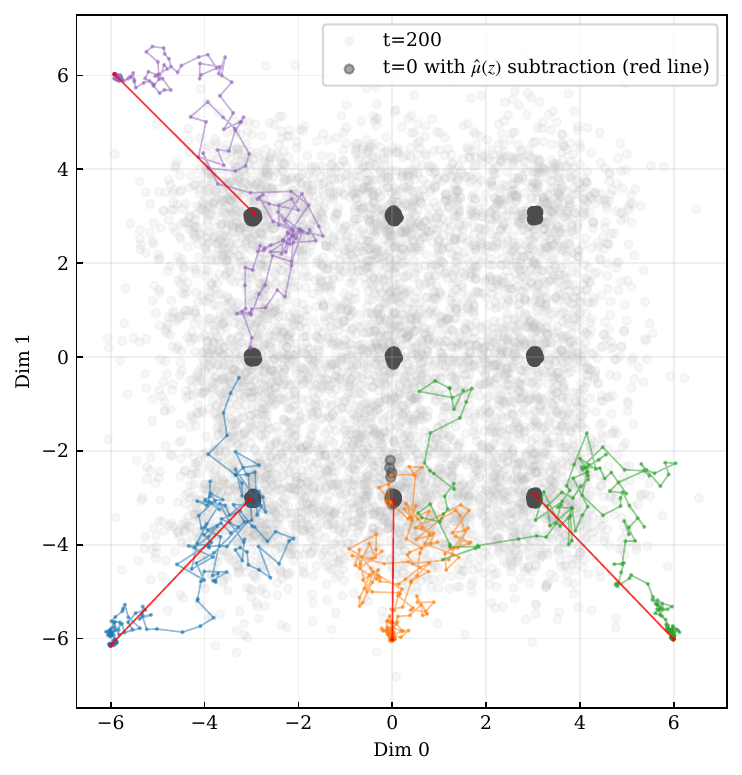}
    \subcaption{DiffuseVAE}
    \label{fig:datasaurus-0.1-reverse-diffusevae}
  \end{minipage} \hfill
  \begin{minipage}[t]{0.48\linewidth}
    \centering
    \includegraphics[height=\linewidth,clip]{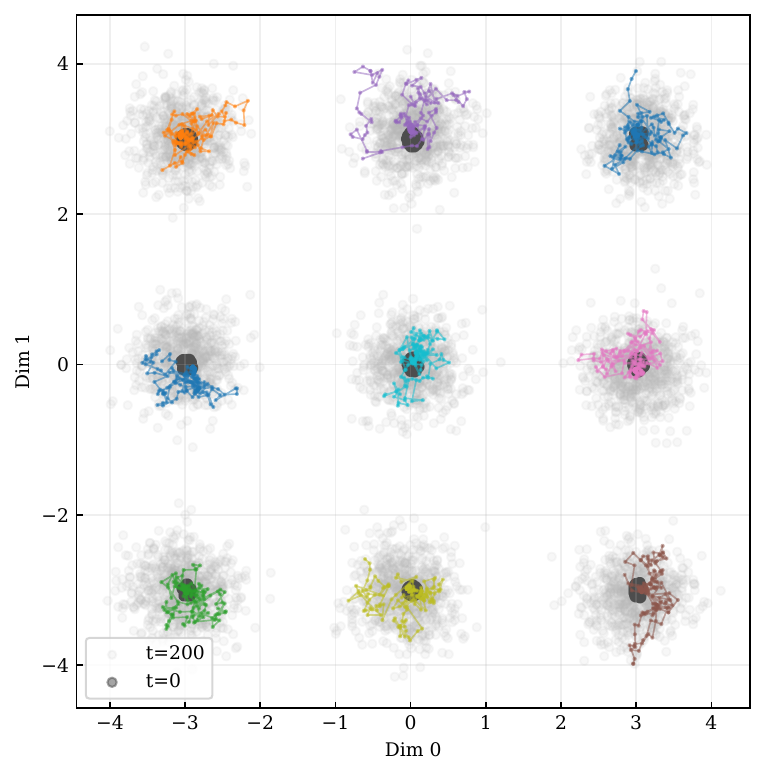}
    \subcaption{RPD (proposed)}
    \label{fig:datasaurus-0.1-reverse-rpd}
  \end{minipage}
  \caption{Reverse diffusion trajectories produced by each model on the Datasaurus-Grid (scale=0.1) dataset.}
  \label{fig:datasaurus-0.1-reverse}
\end{figure}

As shown in the figure, both DDPM and Rectified Flow exhibit trajectories that move from a standard Gaussian distribution at $t=200$ toward the distribution represented by the Datasaurus-Grid (scale=0.1) dataset. DDPM displays stochastic behavior, whereas Rectified Flow produces nearly straight deterministic trajectories.
The reverse trajectories of DiffuseVAE deviate from the ground-truth distribution as $t \to 0$, yet subtracting $\hat{\mu}(z)$ at the final step produces samples that remain close to the ground-truth distribution.
In contrast, RPD exhibits a fundamentally different reverse diffusion behavior: the estimated trajectories converge locally within clusters determined by the latent variable $z$.

\subsubsection{Quantitative evaluation with Datasaurus-Grid datasets}
Using the models during training, we periodically generated samples and evaluated their discrepancies from the GT dataset using 1WD and RW-1WD. The results are shown in Fig.~\ref{fig:datasaurus-grid-plot-1wd-all}. In the figure, the solid and dashed lines denote the median values of each metric, and the error bars represent the 10th and 90th percentiles. As described in the experimental setup, DDPM, $v$-prediction, and Rectified Flow were trained for $120000$ iterations, whereas RPD and DiffuseVAE were trained for $60000$ iterations.

\begin{figure}[htb]
  \begin{center}
    \includegraphics[width=1.0\linewidth,clip]{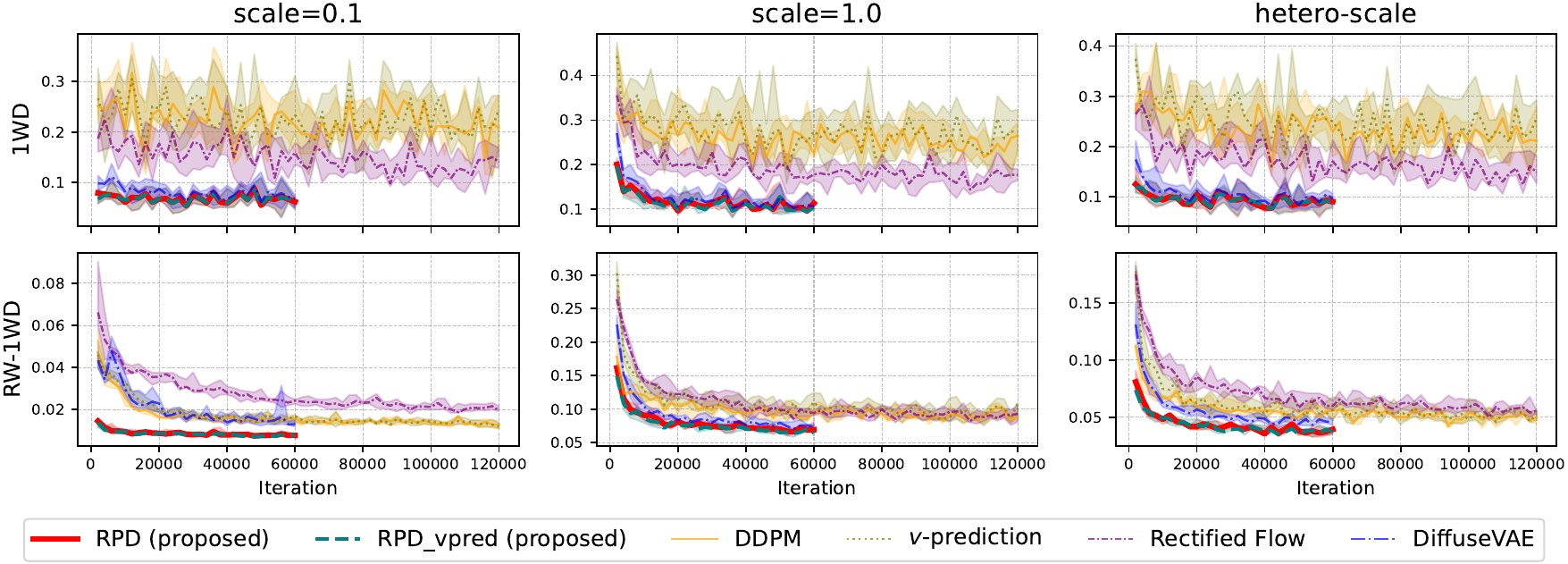}
    \caption{Evolution of 1WD and RW-1WD on the Datasaurus-Grid datasets (left: scale=0.1, center: scale=1.0,   right: hetero-scale).}
    \label{fig:datasaurus-grid-plot-1wd-all}
  \end{center}
\end{figure}

The figure shows that, despite using only half as many training iterations, RPD consistently achieves favorable metric values. The trends of the evaluation metrics for RPD and RPD\_vpred are similar, indicating that both prediction parameterizations yield comparably strong performance.
Comparing DiffuseVAE and RPD, we observe that RPD produces distributions closer to the GT dataset in terms of RW-1WD, particularly for the scale=0.1 and hetero-scale settings. In contrast, the difference in 1WD between the two methods is small. Since 1WD primarily reflects the global characteristics of the distribution, this suggests that the overall distributional trends generated by DiffuseVAE and RPD are similar. However, RW-1WD captures local agreement at each grid point, and RPD clearly outperforms DiffuseVAE in this regard. These observations are consistent with the qualitative results shown in Figs.~\ref{fig:datasaurus-grid-0.1}--\ref{fig:datasaurus-grid-multiscale}.
Furthermore, DDPM, $v$-prediction, and Rectified Flow exhibit comparatively larger 1WD values. This is attributed to the strong dependence of 1WD on the density distribution across the nine grid locations: the generated samples produced by these models deviate from the GT densities at the grid level, leading to degraded 1WD. In contrast, both RPD and DiffuseVAE, which utilize a pretrained VQVAE prior, produce density distributions that are closer to those of the GT dataset, enabling them to achieve lower 1WD values.

\subsubsection{Ablation study on auxiliary variables}
To evaluate the effect of the auxiliary variables introduced in Section~\ref{sec:aux_var},
we performed an ablation study on RPD under the noise-prediction parameterization.
Specifically, we compared two configurations: one in which the auxiliary variable $\omega_t^\epsilon$,
defined in \eqref{eq:omega_t_epsilon}, was provided as an additional input to $\epsilon_\theta$,
and another in which $\hat{\mu}(z)$ was used instead.
The latter corresponds to a setting where the mean of the prior model is directly used as the auxiliary input.
The results of this comparison are shown in Fig.~\ref{fig:datasaurus-grid-plot-rpd-ablation}.
The top row shows the training loss (in log scale), while the middle and bottom rows show the evolution of 1WD and RW-1WD, respectively, during training.

\begin{figure}[tbp]
  \begin{center}
    \includegraphics[width=1.0\linewidth,clip]{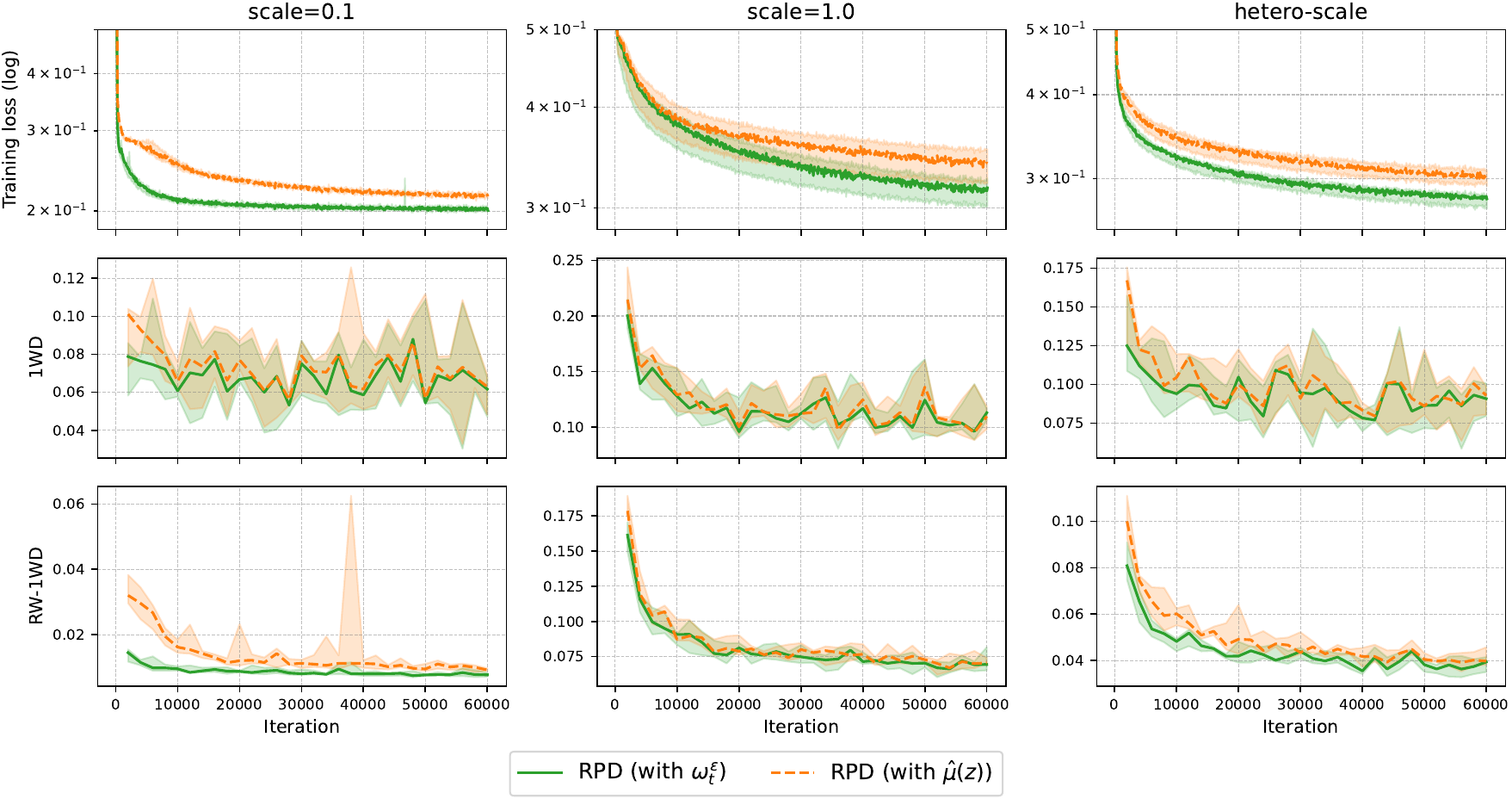}
    \caption{Comparison of training curves with and without the auxiliary variable $\omega_t^\epsilon$ in RPD.}
    \label{fig:datasaurus-grid-plot-rpd-ablation}
  \end{center}
\end{figure}

Across all datasets, the model using $\omega_t^\epsilon$ as an input exhibits a faster reduction in the training loss than the model using $\hat{\mu}(z)$,
demonstrating the effectiveness of the proposed auxiliary variable.
While the difference between the two configurations in terms of 1WD is small,
the RW-1WD results reveal that $\omega_t^\epsilon$ leads to a noticeably faster decrease for the scale=0.1 and hetero-scale settings.
These findings indicate that incorporating $\omega_t^\epsilon$ as an auxiliary input improves the model's ability to capture fine-grained local structure.

\subsubsection{Comparison with DDPM under different hyperparameter configurations}
Diffusion models involve several hyperparameters, including the beta schedule, the capacity of the prediction model, and the total number of diffusion steps. To assess how these design choices affect performance, we compared RPD with multiple DDPM variants in which individual hyperparameters were modified. The results are summarized in \ref{apdx:ddpm_various_hparams}.
The experiments revealed that, even after various hyperparameter adjustments, standard DDPM still struggles to accurately model distributions such as Datasaurus-Grid (scale=0.1), where local structures occur at scales far smaller than the global distribution. Moreover, the results suggest that the advantages of RPD remain robust, underscoring its effectiveness in scenarios characterized by pronounced scale heterogeneity.

\subsection{Experiments on natural image generation}

\subsubsection{Natural image datasets}
For the natural image generation tasks, we used the following two benchmark datasets.

\paragraph{Smithsonian Butterflies Subset (the Butterflies dataset)}\footnote{\url{https://huggingface.co/datasets/huggan/smithsonian_butterflies_subset} (Accessed 2025-12-08)}
This dataset is a preprocessed subset constructed from the Smithsonian Butterflies dataset,\footnote{\url{https://huggingface.co/datasets/ceyda/smithsonian_butterflies} (Accessed 2025-12-08)}
which contains butterfly images sourced from Smithsonian Open Access.\footnote{\url{https://www.si.edu/openaccess} (Accessed 2025-12-08)}
In this subset, images of butterflies belonging to the same species have been removed and background regions have been filtered out to produce clean color images.
We randomly split the dataset into training, validation, and test sets with a 7:1:2 ratio.
All images were resized to $128 \times 128$ for training and evaluation.

\paragraph{Fashion-MNIST (FMNIST) dataset \citep{xiao2017fashion}}
This is a grayscale image dataset of clothing items, annotated with ten class labels (e.g., T-shirt, dress).
For our experiments, we divided the original training set into training and validation splits with a 9:1 ratio.
Although the original image size is $28 \times 28$, we used upsampled images of size $32 \times 32$ to match the architecture of the DNN models described in Section~\ref{sec:compared_models_image}.

Fig.~\ref{fig:gt_test} shows examples randomly selected from the test split of each dataset. For FMNIST, each row contains images belonging to the same class, and the leftmost column displays their corresponding class labels (0--9).
\begin{figure}[tbp]
  \centering
  \begin{minipage}[b]{0.4\linewidth}
    \centering
    \includegraphics[width=\linewidth,clip]{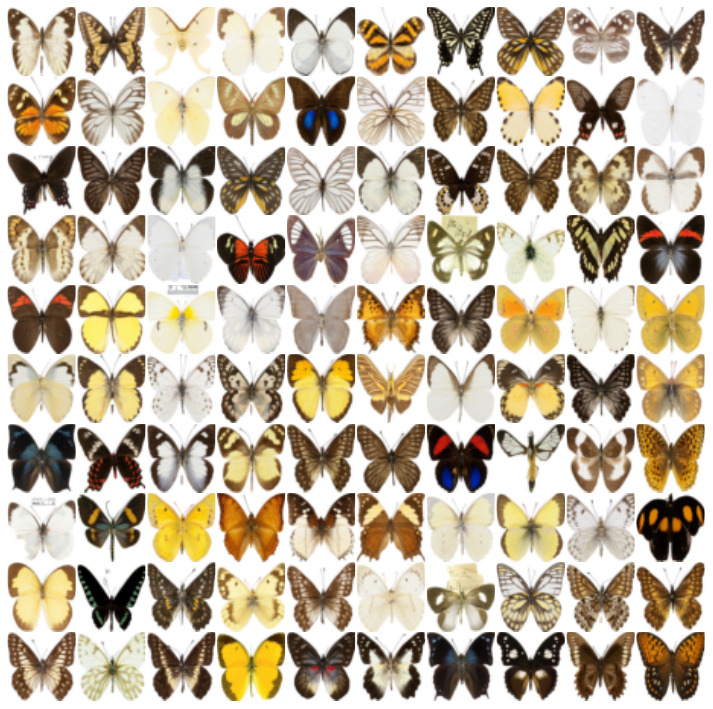}
    \subcaption{Butterflies}
    \label{fig:butterflies_gt_test}
  \end{minipage}%
  \hspace{10mm}
  \begin{minipage}[b]{0.25\linewidth}
    \centering
    \includegraphics[width=\linewidth,clip]{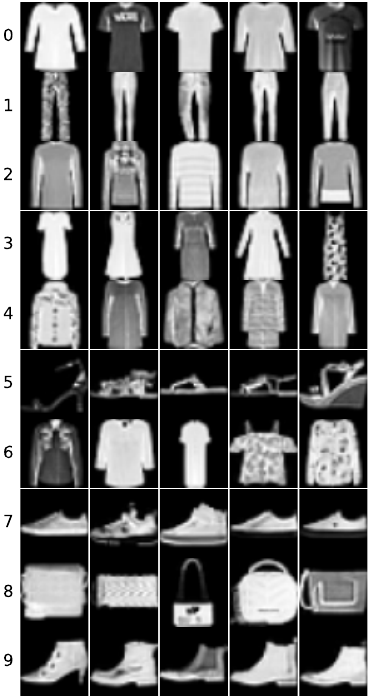}
    \subcaption{FMNIST}
    \label{fig:fmnist_gt_test}
  \end{minipage}%
  \caption{Examples from the test split of the datasets used in our image generation experiments. For FMNIST, the numbers shown in the leftmost column denote the corresponding class labels (0--9).}
  \label{fig:gt_test}
\end{figure}

We performed unconditional image generation for the Butterflies dataset, whereas for the FMNIST dataset we conducted conditional image generation by providing class labels as conditioning information.
As in standard diffusion models, conditioning was incorporated by supplying the class label as an additional input to the prediction model used in the diffusion process. For example, under noise-prediction modeling, we used $\epsilon_\theta(x_t, t, z, c)$ in place of $\epsilon_\theta(x_t, t, z)$, where $c$ denotes the class label. In both RPD and DiffuseVAE, the VAE decoder in the prior model was also conditioned on $c$. Specifically, the functions $\hat{\mu}(z)$ and $\hat{\sigma}(z)$ in \eqref{eq:prior_x0_z} were replaced with $\hat{\mu}(z, c)$ and $\hat{\sigma}(z, c)$, respectively.

\subsubsection{Compared models in image generation tasks} \label{sec:compared_models_image}
We focused our image-generation experiments on pixel-space diffusion baselines, as our aim was to evaluate residual prior diffusion in a setting where the prior and diffusion components are explicitly defined in data space.
In addition to the diffusion baselines considered in Section~\ref{sec:compared_models} (DDPM, $v$-prediction, DiffuseVAE, and Rectified Flow), we also included Denoising Diffusion Implicit Models (DDIM) \citep{songdenoising} and Inductive Moment Matching (IMM) \citep{zhouinductive}.
DDIM extends DDPM by introducing a non-Markovian inference procedure, and it often improves sample quality when the number of inference steps is reduced.
Since DDIM uses the same training objective as DDPM, we reused the DDPM-trained checkpoints and applied DDIM only at inference time.
IMM generalizes consistency models (CMs) \citep{song2023consistency}, a representative approach for few-step generation.
CMs can suffer from training instability and often rely on distillation; IMM was proposed as a way to alleviate these issues and enable stable training without such additional procedures \citep{zhouinductive}.

\paragraph{DNN configurations}
For all models, we used the \texttt{UNet2DModel} implementation provided in the Diffusers library \citep{von-platen-etal-2022-diffusers} to predict the diffusion dynamics. The configurations of \texttt{UNet2DModel} for the Butterflies and FMNIST datasets are summarized in Table~\ref{tab:unet_configs}.
\begin{table}[tbp]
  \centering
  \caption{DNN configurations for the image generation tasks}
  \label{tab:unet_configs}
  \scalebox{0.9}{
    \begin{tabular}{lll}
      \toprule
      \textbf{Parameter}                                                   & \textbf{Butterflies}           & \textbf{FMNIST}      \\
      \midrule
      \texttt{class}                                                       & \texttt{UNet2DModel}           & \texttt{UNet2DModel} \\
      \texttt{sample\_size}                                                & 128                            & 32                   \\
      \texttt{in\_channels}                                                & 3                              & 1                    \\
      \texttt{out\_channels}                                               & 3                              & 1                    \\
      \texttt{layers\_per\_block}                                          & 2                              & 1                    \\
      \texttt{norm\_num\_groups}                                           & 32                             & 16                   \\
      \texttt{block\_out\_channels}                                        & (128, 128, 256, 256, 512, 512) & (16, 32, 32, 64)     \\
      \texttt{down\_block\_types}                                          &
      \begin{tabular}[t]{@{}l@{}}
        \texttt{DownBlock2D}, \texttt{DownBlock2D}, \texttt{DownBlock2D}, \\
        \texttt{DownBlock2D}, \texttt{AttnDownBlock2D}, \texttt{DownBlock2D}
      \end{tabular} &
      \begin{tabular}[t]{@{}l@{}}
        \texttt{DownBlock2D}, \texttt{DownBlock2D}, \\
        \texttt{AttnDownBlock2D}, \texttt{DownBlock2D}
      \end{tabular}                                                                                \\
      \texttt{up\_block\_types}                                            &
      \begin{tabular}[t]{@{}l@{}}
        \texttt{UpBlock2D}, \texttt{AttnUpBlock2D}, \texttt{UpBlock2D}, \\
        \texttt{UpBlock2D}, \texttt{UpBlock2D}, \texttt{UpBlock2D}
      \end{tabular}   &
      \begin{tabular}[t]{@{}l@{}}
        \texttt{UpBlock2D}, \texttt{AttnUpBlock2D}, \\
        \texttt{UpBlock2D}, \texttt{UpBlock2D}
      \end{tabular}                                                                                  \\
      \texttt{num\_class\_embeds}                                          & \texttt{None}                  & 10                   \\
      \bottomrule
    \end{tabular}}
\end{table}

As prior models for DiffuseVAE and RPD, we used $\beta$-VAEs \citep{higgins2017betavae}. The encoder and decoder of $\beta$-VAE were implemented using a ResNet-based architecture \citep{he2016deep}. Note that the $\beta$-VAE used in our experiments is considerably smaller than the \texttt{UNet2DModel} employed in the diffusion models---its parameter count is less than one-tenth that of the UNet (see Section~\ref{sec:model_size} for details). $\beta$-VAE includes a hyperparameter $\beta$ that controls the strength of the regularization. For each dataset, we trained models with $\beta \in \left\{8, 16, 24, 32, 40\right\}$ using the training split and selected the $\beta$ value that yielded the best performance on the validation split (see Section~\ref{sec:metric} for the evaluation metrics).

\paragraph{Diffusion schedule configuration}
In training the diffusion models, we set the maximum diffusion timestep to $T = 1000$ and set the noise schedule $\beta_t$ by following the one of Stable Diffusion v1.5.\footnote{\url{https://huggingface.co/stable-diffusion-v1-5/stable-diffusion-v1-5}
  (Accessed 2025-12-09)}
For IMM, the schedule was configured according to the procedure described in the original paper \citep{zhouinductive}.

\paragraph{Training configurations}
For each dataset and each model, we trained the networks following the settings summarized in Table~\ref{tab:training_config}. Note that DiffuseVAE and RPD require pretraining of a $\beta$-VAE, and therefore the number of training iterations for these methods was set lower than that of the other baselines. As shown in the table, even when combining the training iterations of $\beta$-VAE and those of DiffuseVAE/RPD/RPD\_vpred, the total remained smaller than that of the other methods. Furthermore, because $\beta$-VAE had a relatively small model size in our experiments, the time required for its training was comparatively short.
\begin{table}[tbp]
  \centering
  \caption{Training configurations for each dataset and each model in the image generation tasks.}
  \label{tab:training_config}
  \scalebox{0.9}{
    \begin{tabular}{lcccccc}
      \toprule
                               & \multicolumn{3}{c}{\textbf{Butterflies}}                           & \multicolumn{3}{c}{\textbf{FMNIST}}                                                                                     \\
      \cmidrule(lr){2-4} \cmidrule(lr){5-7}
                               & $\beta$-VAE
                               & \begin{tabular}{@{}c@{}}DiffuseVAE/\\RPD/\\RPD\_vpred\end{tabular}
                               & Other models
                               & $\beta$-VAE
                               & \begin{tabular}{@{}c@{}}DiffuseVAE/\\RPD/\\RPD\_vpred\end{tabular}
                               & Other models                                                                                                                                                                                 \\
      \midrule
      \textbf{Optimizer}       & Adam\cite{adam}                                                    & Adam                                & Adam               & Adam               & Adam               & Adam               \\
      \textbf{Learning rate}   & $1.0\times10^{-4}$                                                 & $1.0\times10^{-4}$                  & $1.0\times10^{-4}$ & $1.0\times10^{-4}$ & $1.0\times10^{-3}$ & $1.0\times10^{-3}$ \\
      \textbf{Gradient clip}   & 0.1                                                                & --                                  & --                 & 0.1                & --                 & --                 \\
      \textbf{Batch size}      & 32                                                                 & 32                                  & 32                 & 128                & 128                & 128                \\
      \textbf{dtype}           & bfloat16                                                           & bfloat16                            & bfloat16           & bfloat16           & bfloat16           & bfloat16           \\
      \textbf{Training iters.} & 10000                                                              & 60000                               & 120000             & 10000              & 30000              & 60000              \\
      \bottomrule
    \end{tabular}
  }
\end{table}

\subsubsection{Evaluation metrics for image generation} \label{sec:metric}
To evaluate the quality of the generated image distributions, we used the kernel inception distance \citep{binkowski2018demystifying} (KID) and 1WD for the Butterflies dataset. For the FMNIST dataset, we computed 1WD and maximum mean discrepancy \citep{gretton2012kernel} (MMD) conditioned on each class label, and report the average across all classes as the final metric. We refer to these metrics as class-wise 1WD (CW-1WD) and class-wise MMD (CW-MMD), respectively. As a reference, Table~\ref{tab:train_test_metrics} reports the same metrics computed between the training and test splits of each dataset.

\begin{table}[tbph]
  \centering
  \caption{Distances between the training and test splits of each dataset, measured by the corresponding evaluation metrics.}
  \label{tab:train_test_metrics}
  \scalebox{0.9}{
    \begin{tabular}{lcccc}
      \toprule
                                     & \multicolumn{2}{c}{\textbf{Butterflies}} & \multicolumn{2}{c}{\textbf{FMNIST}}                   \\
      \cmidrule(lr){2-3} \cmidrule(lr){4-5}
                                     & KID                                      & 1WD                                 & CW-1WD & CW-MMD \\
      \midrule
      training split vs.\ test split & 0.0014                                   & 97.89                               & 8.94   & 0.0027 \\
      \bottomrule
    \end{tabular}
  }
\end{table}

In the evaluation, each trained model generated the same number of samples as contained in the corresponding test split, and we computed the distances between the generated samples and the test split by using the metrics described above. We also varied the number of inference steps in the inference procedure of the diffusion models, over ${3, 10, 50}$, and evaluated each configuration separately. In addition, we repeated the model training five times with different random seeds and computed the metrics for each run; we report the mean and standard deviation over these five runs.

\subsubsection{Qualitative evaluation of generated images}
\paragraph{Butterflies dataset results}
Figures~\ref{fig:butterflies_inf_steps_50} and \ref{fig:butterflies_inf_steps_3} present examples of images generated by each model trained on the Butterflies dataset (100 samples per model). Figure~\ref{fig:butterflies_inf_steps_50} shows the results obtained with 50 inference steps, whereas Fig.~\ref{fig:butterflies_inf_steps_3} shows those obtained with three inference steps.

With 50 inference steps, all models produce diverse butterfly images, although differences in the characteristics of the generated images can be observed (Fig.~\ref{fig:butterflies_inf_steps_50}).
IMM, however, tends to generate noticeably distorted shapes.\footnote{One plausible factor contributing to the degraded quality of images generated by IMM is the mini-batch size used during training. Whereas the IMM paper \citep{zhouinductive} employed a batch size of 4096, our experiments used the smaller batch size listed in Table~\ref{tab:training_config} due to single-GPU training constraints.}
With only three inference steps (Fig.~\ref{fig:butterflies_inf_steps_3}), the differences among models become more pronounced. DDPM, DDIM, $v$-prediction, and DiffuseVAE yield incomplete images with noticeably darker color tones. Rectified Flow produces sharper images but exhibits very low diversity, with many samples appearing nearly identical. IMM generates more diverse outputs, but shape distortions remain, similar to the 50-step results.
In contrast, both RPD and RPD\_vpred generate diverse and high-quality butterfly images even with only three inference steps, demonstrating that their generation performance remains robust under substantial reductions in the number of inference steps.

\begin{figure}[tbp]
  \centering
  \begin{minipage}[b]{0.24\linewidth}
    \centering
    \includegraphics[width=\linewidth]{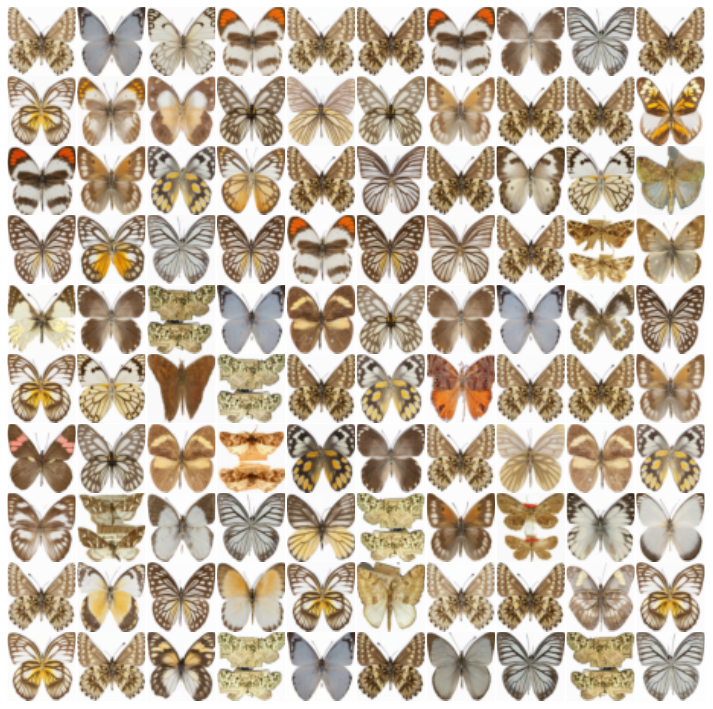}
    \subcaption*{DDPM}
  \end{minipage}%
  \hfill
  \begin{minipage}[b]{0.24\linewidth}
    \centering
    \includegraphics[width=\linewidth]{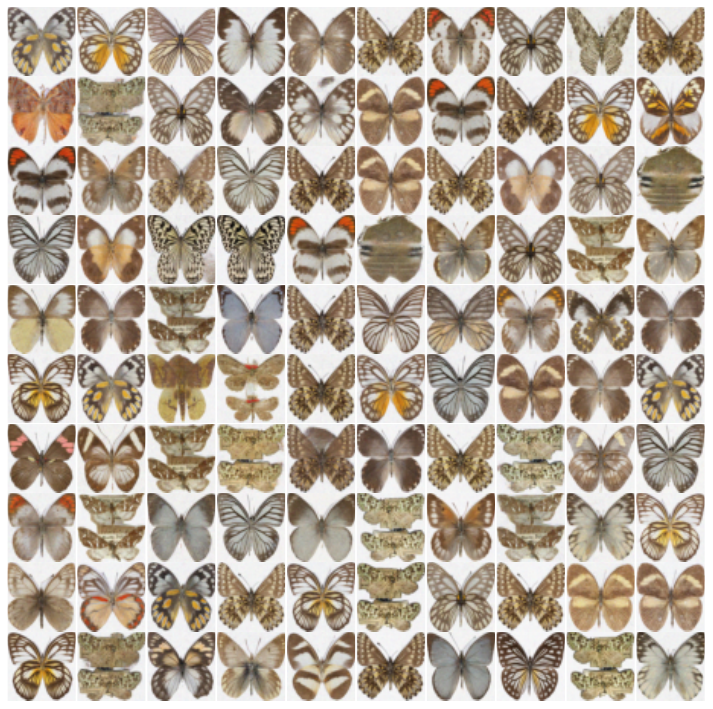}
    \subcaption*{DDIM}
  \end{minipage}%
  \hfill
  \begin{minipage}[b]{0.24\linewidth}
    \centering
    \includegraphics[width=\linewidth]{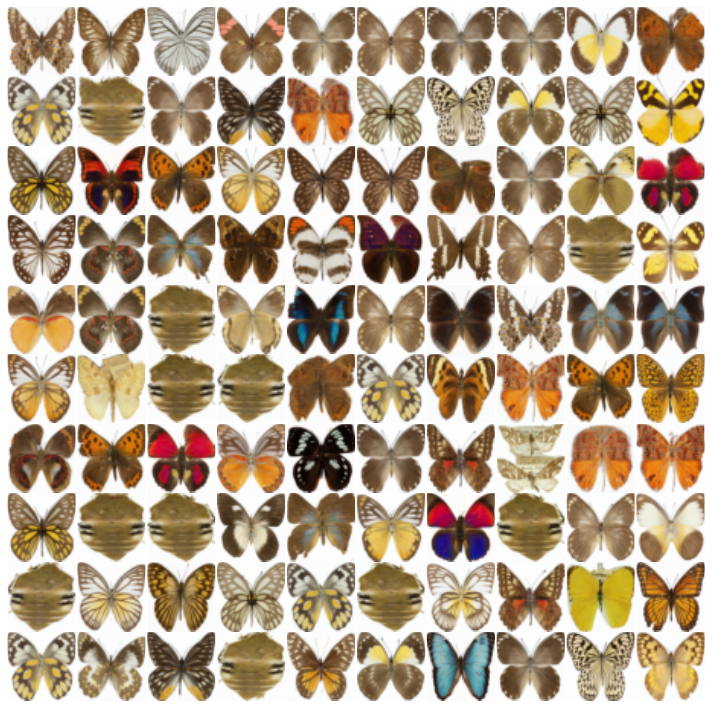}
    \subcaption*{$v$-prediction}
  \end{minipage}%
  \hfill
  \begin{minipage}[b]{0.24\linewidth}
    \centering
    \includegraphics[width=\linewidth]{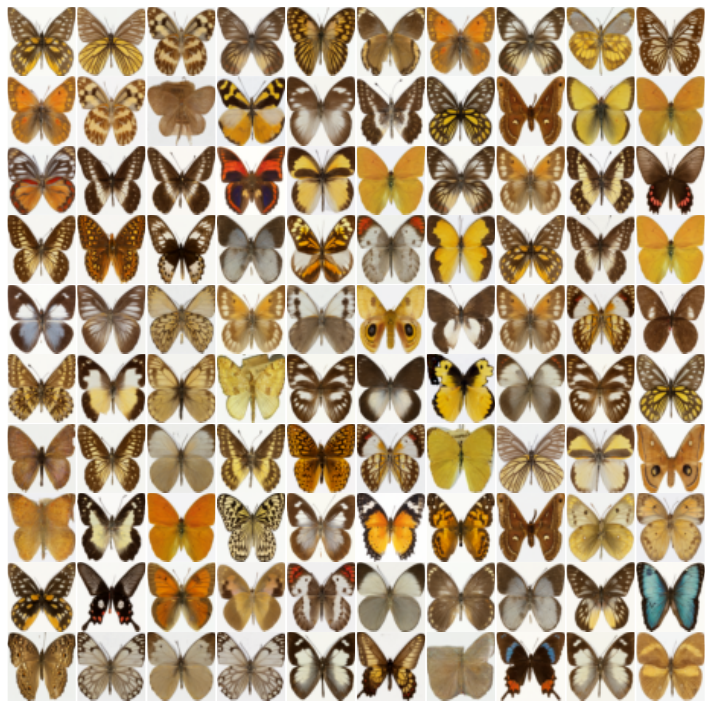}
    \subcaption*{DiffuseVAE}
  \end{minipage}%
  \vspace{2mm}
  \begin{minipage}[b]{0.24\linewidth}
    \centering
    \includegraphics[width=\linewidth]{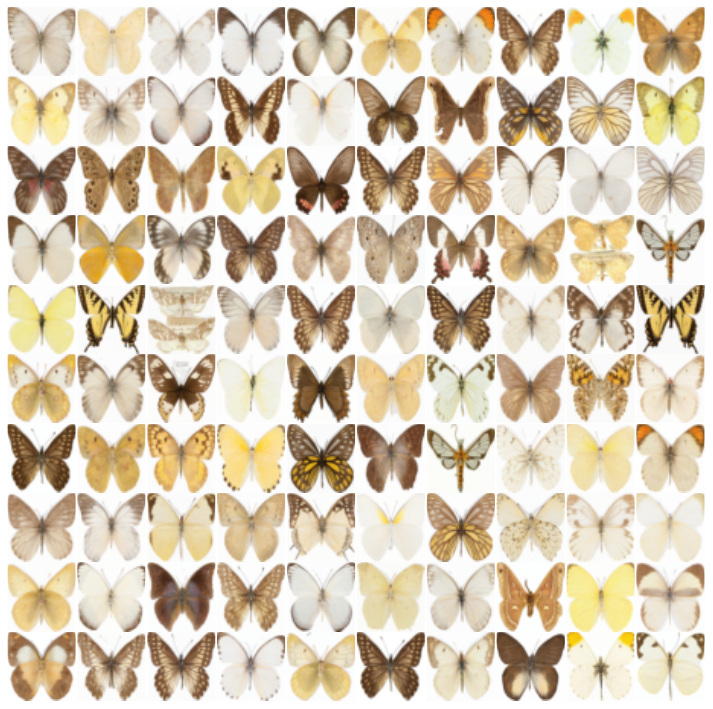}
    \subcaption*{Rectified Flow}
  \end{minipage}%
  \hfill
  \begin{minipage}[b]{0.24\linewidth}
    \centering
    \includegraphics[width=\linewidth]{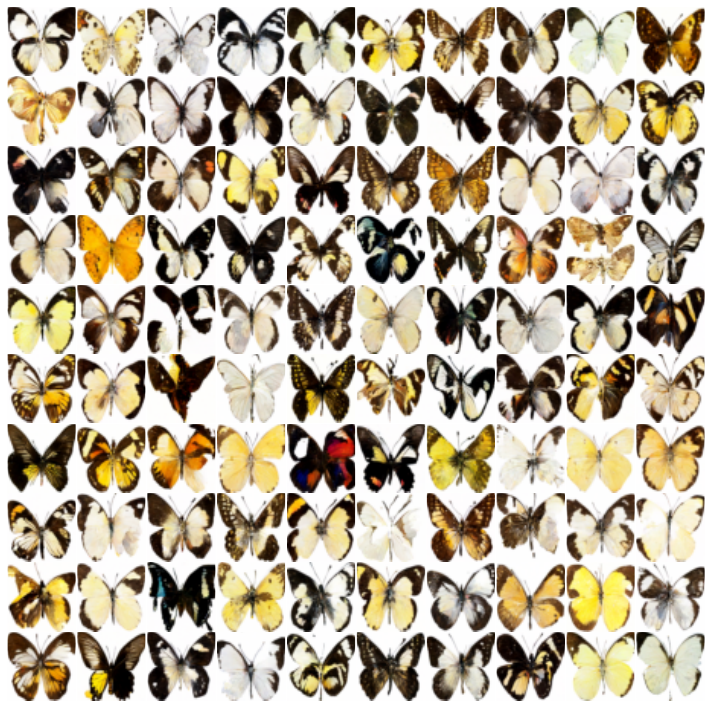}
    \subcaption*{IMM}
  \end{minipage}%
  \hfill
  \begin{minipage}[b]{0.24\linewidth}
    \centering
    \includegraphics[width=\linewidth]{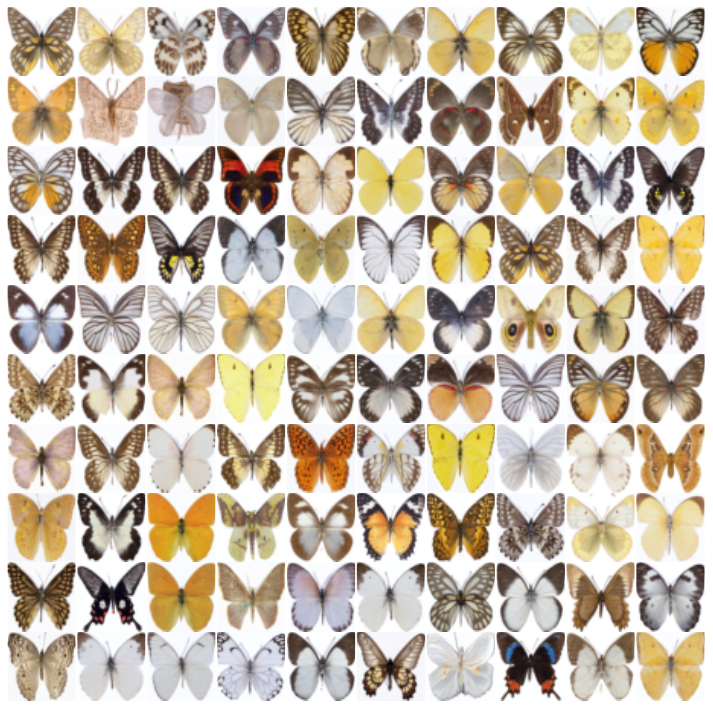}
    \subcaption*{RPD}
  \end{minipage}%
  \hfill
  \begin{minipage}[b]{0.24\linewidth}
    \centering
    \includegraphics[width=\linewidth]{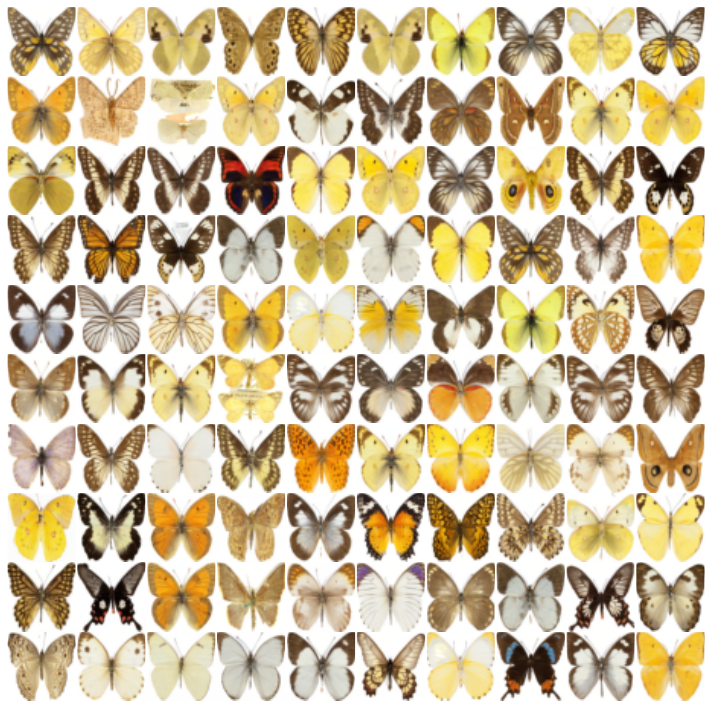}
    \subcaption*{RPD\_vpred}
  \end{minipage}%
  \caption{Samples generated with 50 inference steps on the Butterflies dataset.}
  \label{fig:butterflies_inf_steps_50}
\end{figure}
\begin{figure}[tbp]
  \centering
  \begin{minipage}[b]{0.24\linewidth}
    \centering
    \includegraphics[width=\linewidth]{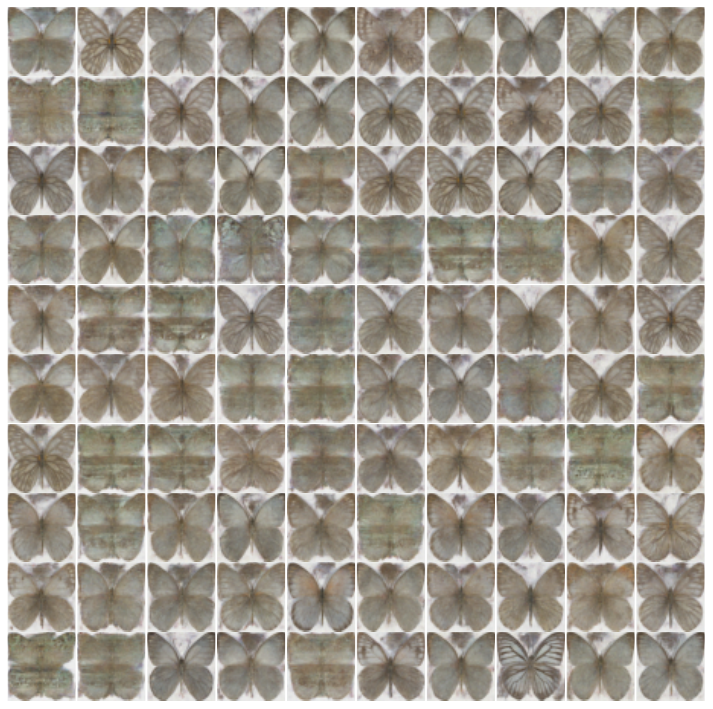}
    \subcaption*{DDPM}
  \end{minipage}%
  \hfill
  \begin{minipage}[b]{0.24\linewidth}
    \centering
    \includegraphics[width=\linewidth]{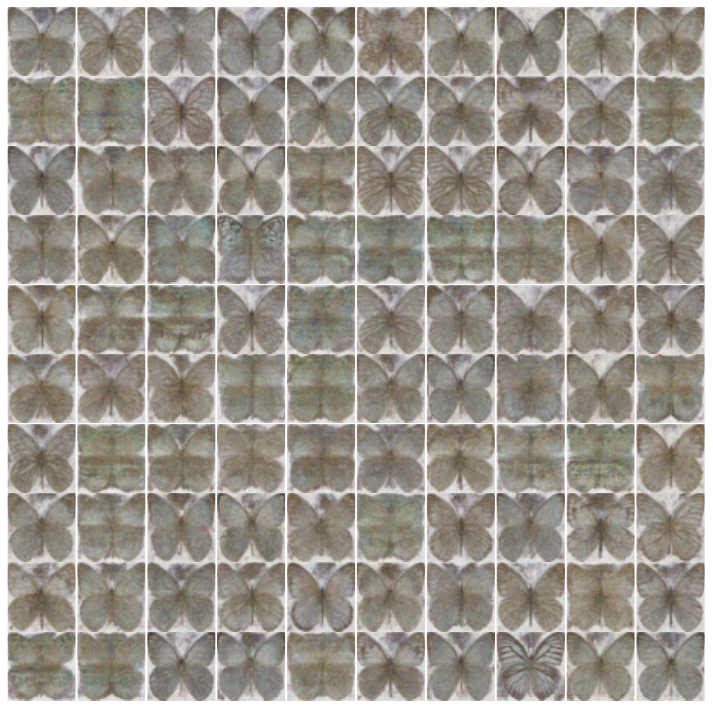}
    \subcaption*{DDIM}
  \end{minipage}%
  \hfill
  \begin{minipage}[b]{0.24\linewidth}
    \centering
    \includegraphics[width=\linewidth]{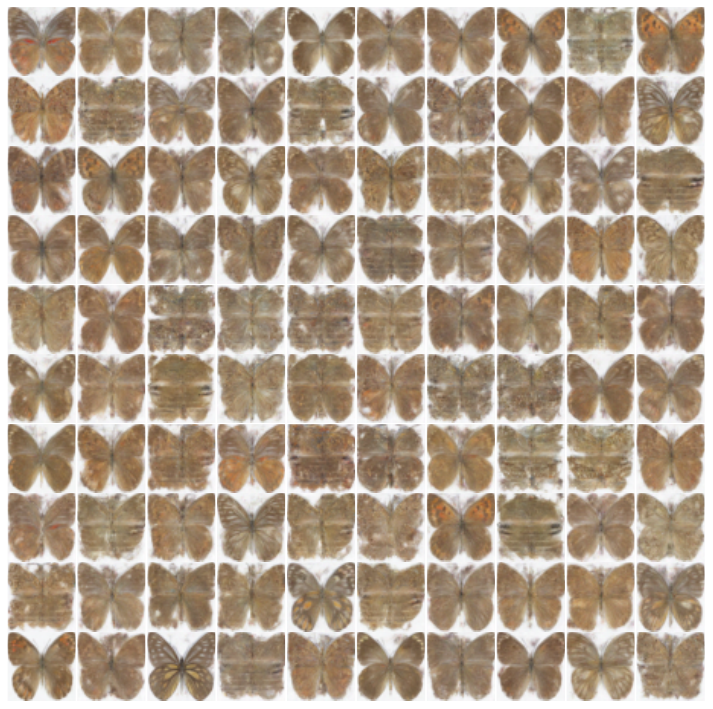}
    \subcaption*{$v$-prediction}
  \end{minipage}%
  \hfill
  \begin{minipage}[b]{0.24\linewidth}
    \centering
    \includegraphics[width=\linewidth]{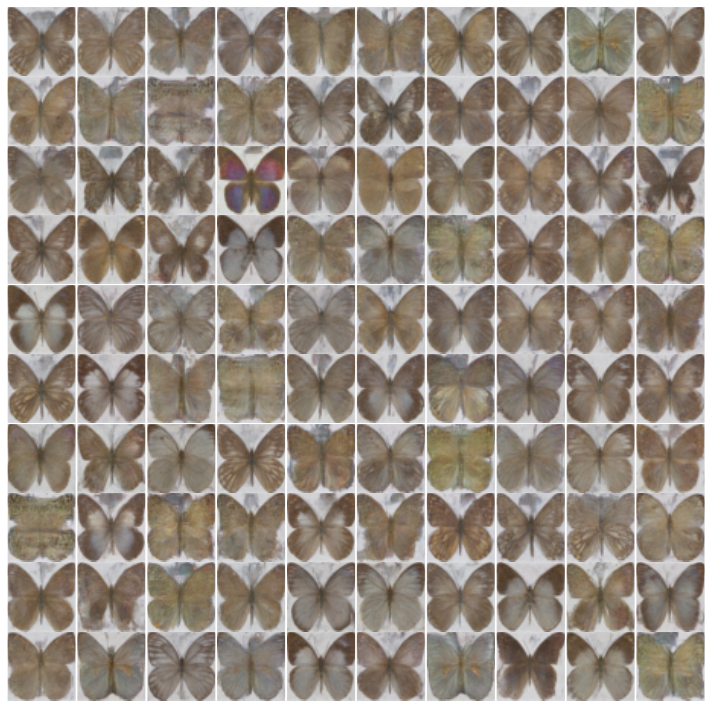}
    \subcaption*{DiffuseVAE}
  \end{minipage}%
  \vspace{2mm}
  \begin{minipage}[b]{0.24\linewidth}
    \centering
    \includegraphics[width=\linewidth]{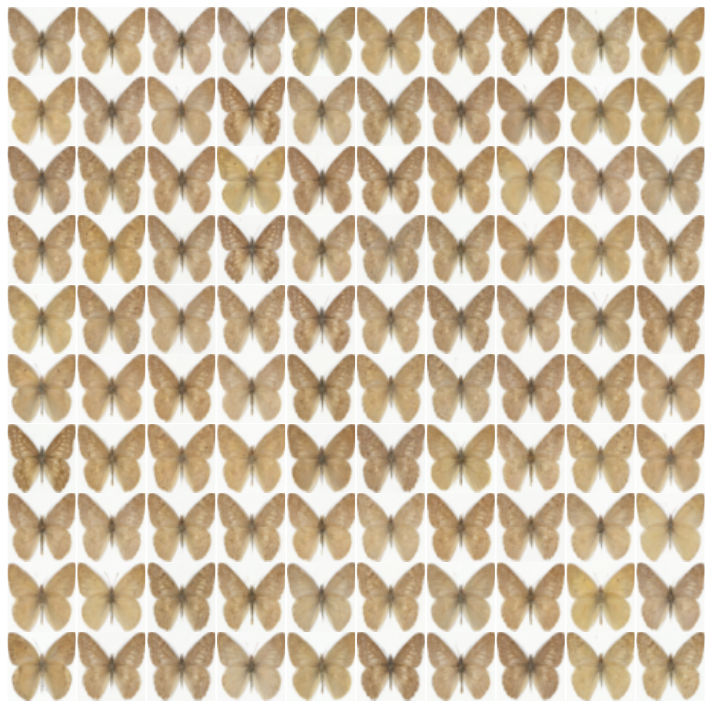}
    \subcaption*{Rectified Flow}
  \end{minipage}%
  \hfill
  \begin{minipage}[b]{0.24\linewidth}
    \centering
    \includegraphics[width=\linewidth]{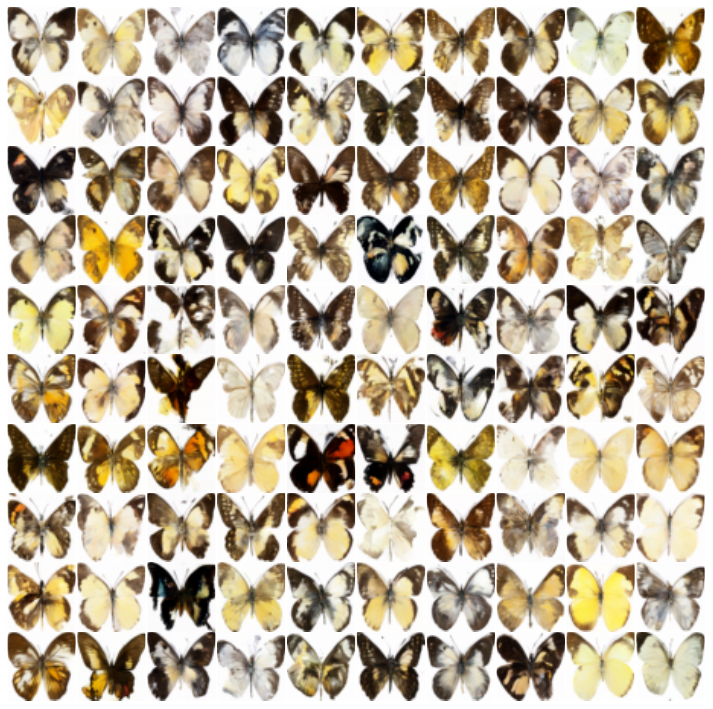}
    \subcaption*{IMM}
  \end{minipage}%
  \hfill
  \begin{minipage}[b]{0.24\linewidth}
    \centering
    \includegraphics[width=\linewidth]{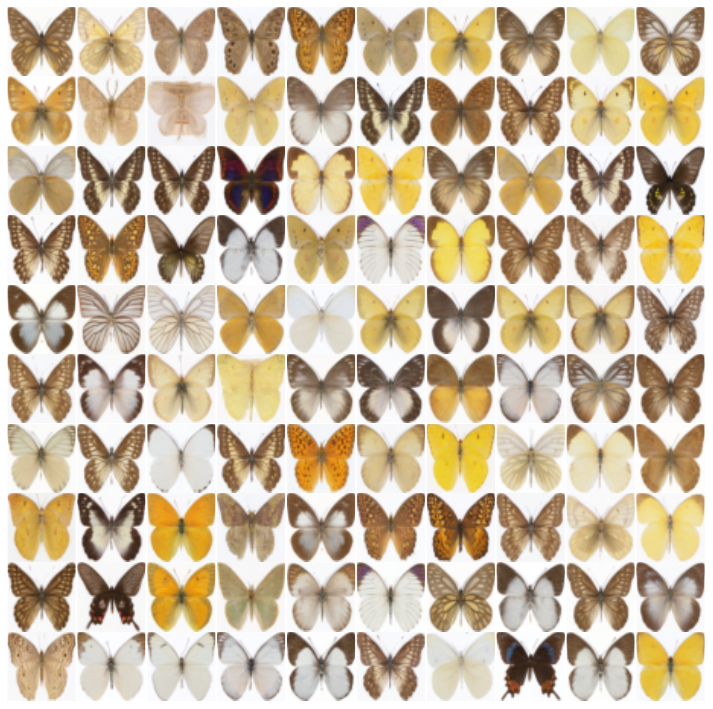}
    \subcaption*{RPD}
  \end{minipage}%
  \hfill
  \begin{minipage}[b]{0.24\linewidth}
    \centering
    \includegraphics[width=\linewidth]{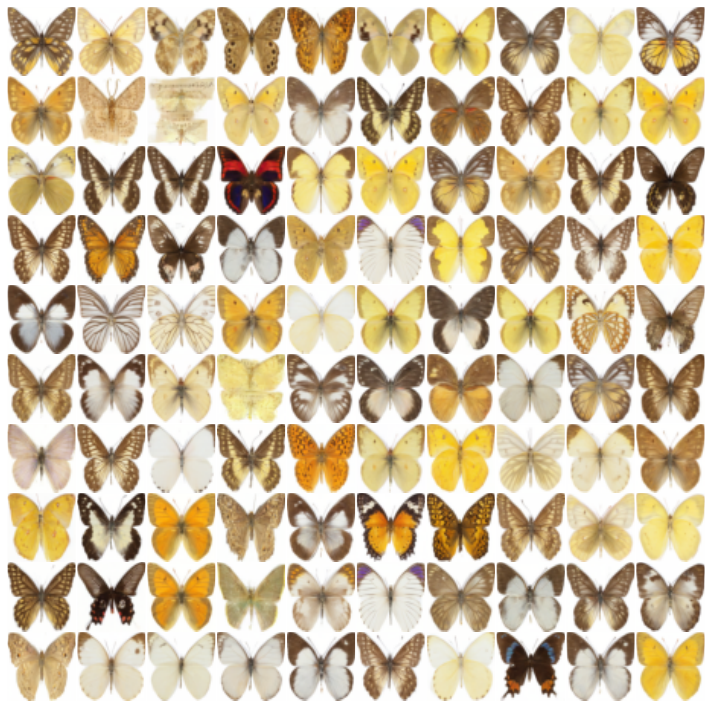}
    \subcaption*{RPD\_vpred}
  \end{minipage}%
  \caption{Samples generated with three inference steps on the Butterflies dataset.}
  \label{fig:butterflies_inf_steps_3}
\end{figure}

\paragraph{FMNIST dataset results}
Figures~\ref{fig:fmnist_inf_steps_50} and \ref{fig:fmnist_inf_steps_3} present the samples generated by each model trained on the FMNIST dataset (five samples per class for each model).
Figure~\ref{fig:fmnist_inf_steps_50} shows the results obtained with 50 inference steps, whereas Fig.~\ref{fig:fmnist_inf_steps_3} shows those obtained with three inference steps.

As in the case of the Butterflies dataset, most models (except IMM) produce diverse and high-fidelity images when the number of inference steps is set to 50, whereas with only three inference steps, DDPM, DDIM, $v$-prediction, and DiffuseVAE yield incomplete generations. Although the images generated by Rectified Flow are relatively sharp, the method exhibits reduced intra-class diversity, particularly in attributes such as clothing brightness. IMM appears to produce higher-quality images with three inference steps than with 50, exhibiting an inverse trend compared with the other models. By comparison, both RPD and RPD\_vpred generate diverse, high-quality images even with a very small number of inference steps.

\begin{figure}[tbp]
  \centering
  \begin{minipage}[b]{0.20\linewidth}
    \centering
    \includegraphics[width=\linewidth]{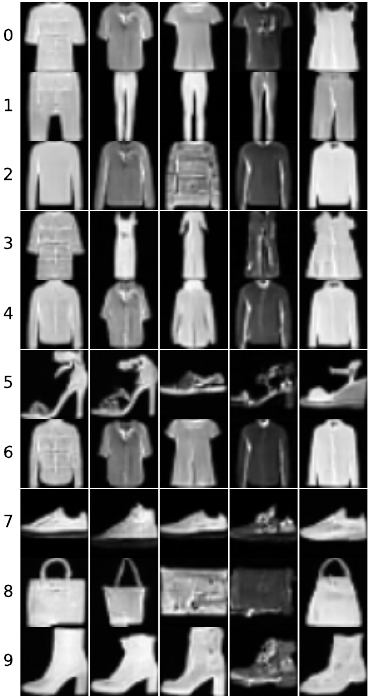}
    \subcaption*{DDPM}
  \end{minipage}%
  \hfill
  \begin{minipage}[b]{0.20\linewidth}
    \centering
    \includegraphics[width=\linewidth]{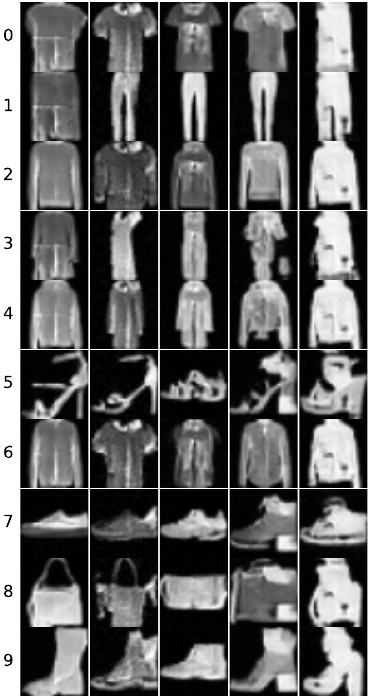}
    \subcaption*{DDIM}
  \end{minipage}%
  \hfill
  \begin{minipage}[b]{0.20\linewidth}
    \centering
    \includegraphics[width=\linewidth]{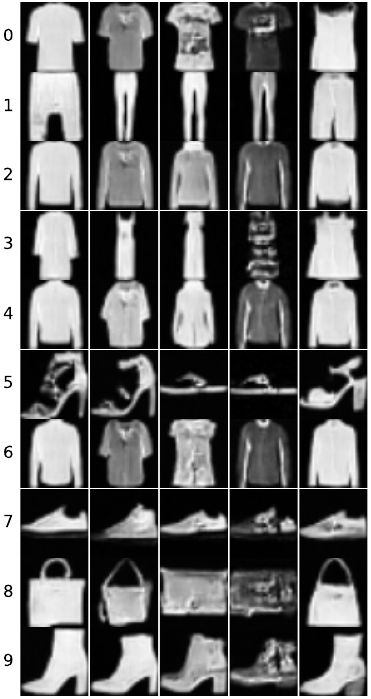}
    \subcaption*{$v$-prediction}
  \end{minipage}%
  \hfill
  \begin{minipage}[b]{0.20\linewidth}
    \centering
    \includegraphics[width=\linewidth]{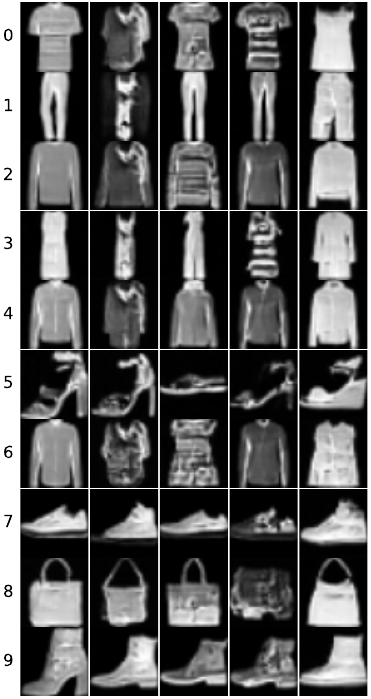}
    \subcaption*{DiffuseVAE}
  \end{minipage}%
  \vspace{2mm}
  \begin{minipage}[b]{0.20\linewidth}
    \centering
    \includegraphics[width=\linewidth]{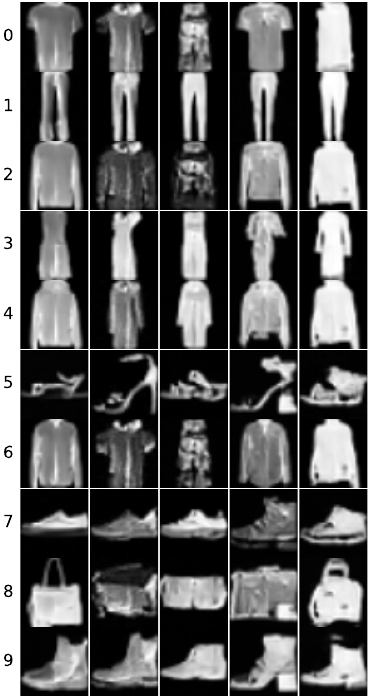}
    \subcaption*{Rectified Flow}
  \end{minipage}%
  \hfill
  \begin{minipage}[b]{0.20\linewidth}
    \centering
    \includegraphics[width=\linewidth]{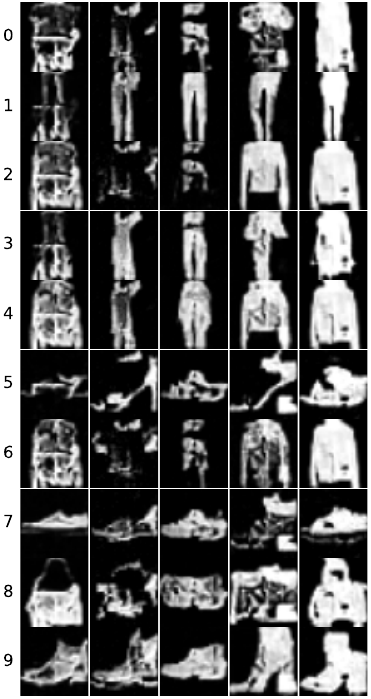}
    \subcaption*{IMM}
  \end{minipage}%
  \hfill
  \begin{minipage}[b]{0.20\linewidth}
    \centering
    \includegraphics[width=\linewidth]{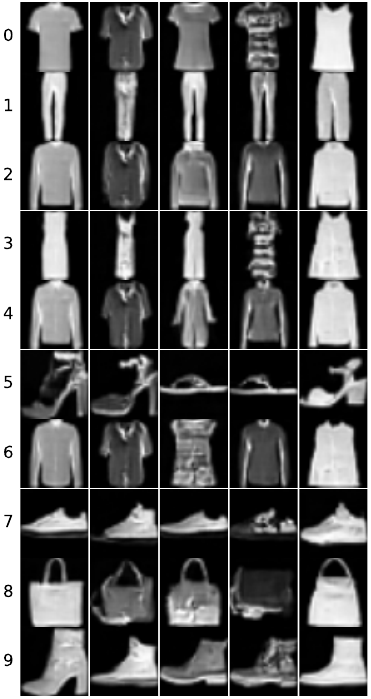}
    \subcaption*{RPD}
  \end{minipage}%
  \hfill
  \begin{minipage}[b]{0.20\linewidth}
    \centering
    \includegraphics[width=\linewidth]{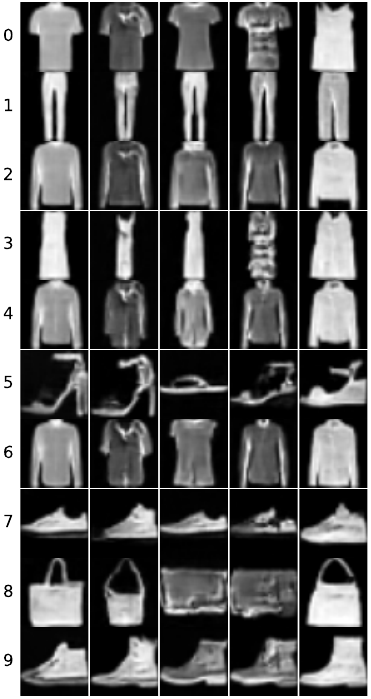}
    \subcaption*{RPD\_vpred}
  \end{minipage}%
  \caption{Samples generated with 50 inference steps on the FMNIST dataset.}
  \label{fig:fmnist_inf_steps_50}
\end{figure}
\begin{figure}[tbp]
  \centering
  \begin{minipage}[b]{0.20\linewidth}
    \centering
    \includegraphics[width=\linewidth]{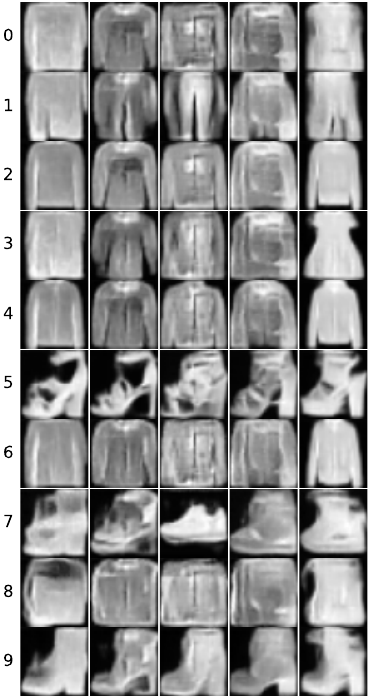}
    \subcaption*{DDPM}
  \end{minipage}%
  \hfill
  \begin{minipage}[b]{0.20\linewidth}
    \centering
    \includegraphics[width=\linewidth]{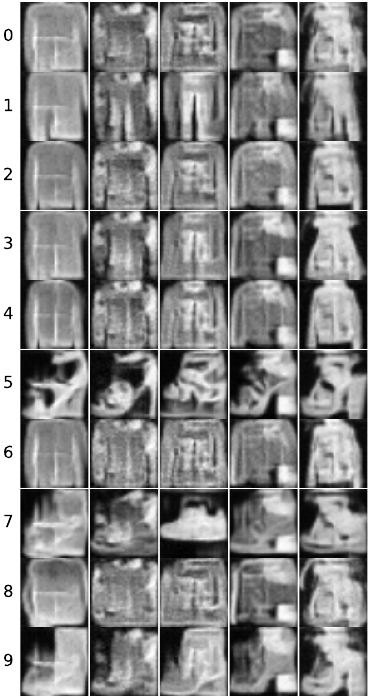}
    \subcaption*{DDIM}
  \end{minipage}%
  \hfill
  \begin{minipage}[b]{0.20\linewidth}
    \centering
    \includegraphics[width=\linewidth]{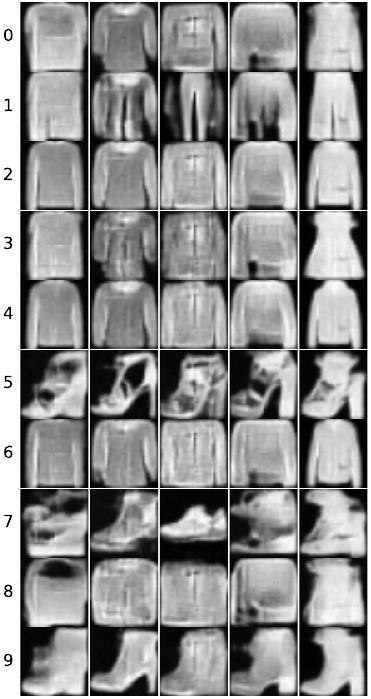}
    \subcaption*{$v$-prediction}
  \end{minipage}%
  \hfill
  \begin{minipage}[b]{0.20\linewidth}
    \centering
    \includegraphics[width=\linewidth]{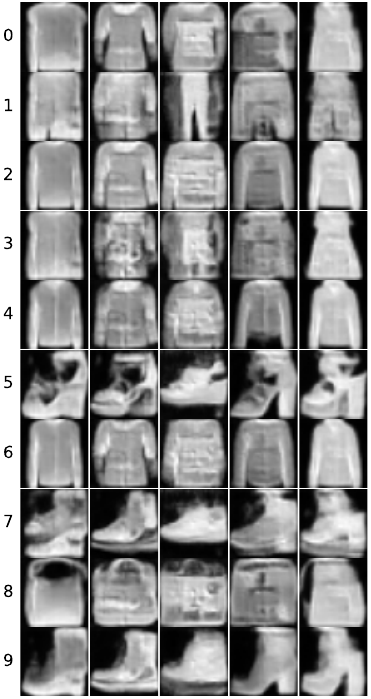}
    \subcaption*{DiffuseVAE}
  \end{minipage}%
  \vspace{2mm}
  \begin{minipage}[b]{0.20\linewidth}
    \centering
    \includegraphics[width=\linewidth]{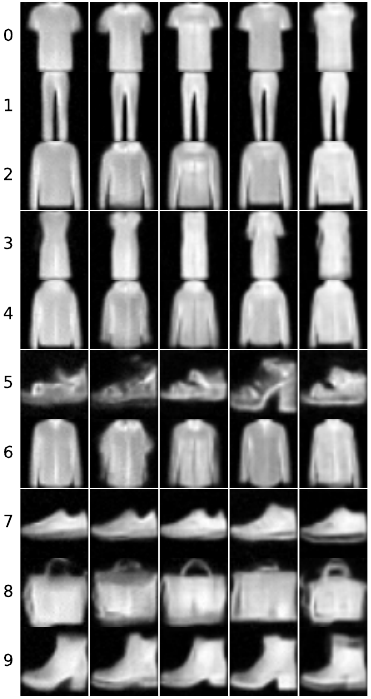}
    \subcaption*{Rectified Flow}
  \end{minipage}%
  \hfill
  \begin{minipage}[b]{0.20\linewidth}
    \centering
    \includegraphics[width=\linewidth]{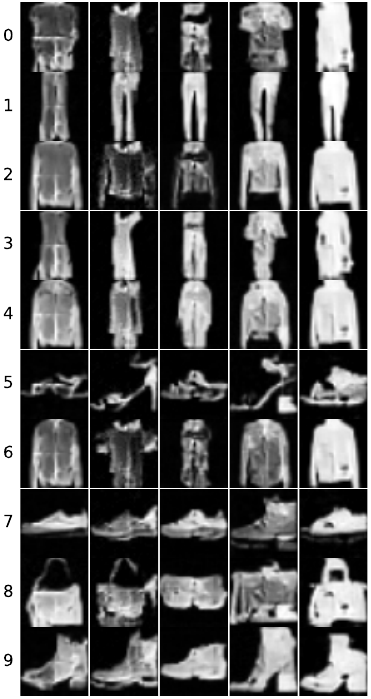}
    \subcaption*{IMM}
  \end{minipage}%
  \hfill
  \begin{minipage}[b]{0.20\linewidth}
    \centering
    \includegraphics[width=\linewidth]{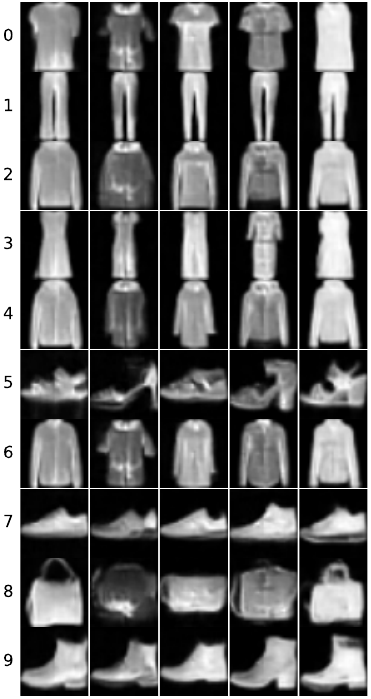}
    \subcaption*{RPD}
  \end{minipage}%
  \hfill
  \begin{minipage}[b]{0.20\linewidth}
    \centering
    \includegraphics[width=\linewidth]{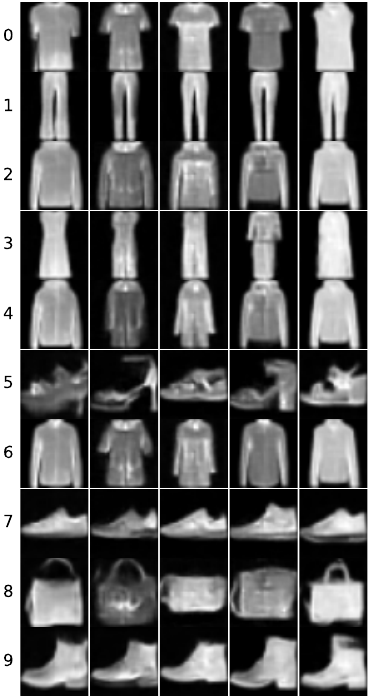}
    \subcaption*{RPD\_vpred}
  \end{minipage}%
  \caption{Samples generated with three inference steps on the FMNIST dataset.}
  \label{fig:fmnist_inf_steps_3}
\end{figure}

\paragraph{$\beta$-VAE results}
Figure~\ref{fig:vae_comparison} shows examples generated by the $\beta$-VAE models that were used as prior models in RPD. For each dataset, the left column displays the mean images produced by the $\beta$-VAE corresponding to $\hat{\mu}(z)$ in \eqref{eq:prior_x0_z}, while the right column shows samples obtained by drawing from the decoded distribution, including the estimated variance. As the figure illustrates, $\beta$-VAE captures coarse structures of the images but fails to express fine details. This limitation arises partly from the relatively small model size used in our experiments and partly from the well-known tendency of VAE-based models trained with pixel-wise squared-error objectives to produce blurry images \citep{dosovitskiy2016generating}.

\begin{figure}[tbp]
  \centering
  \begin{minipage}[b]{0.6\linewidth}
    \centering
    \begin{minipage}[b]{0.48\linewidth}
      \centering
      \includegraphics[width=\linewidth,clip]{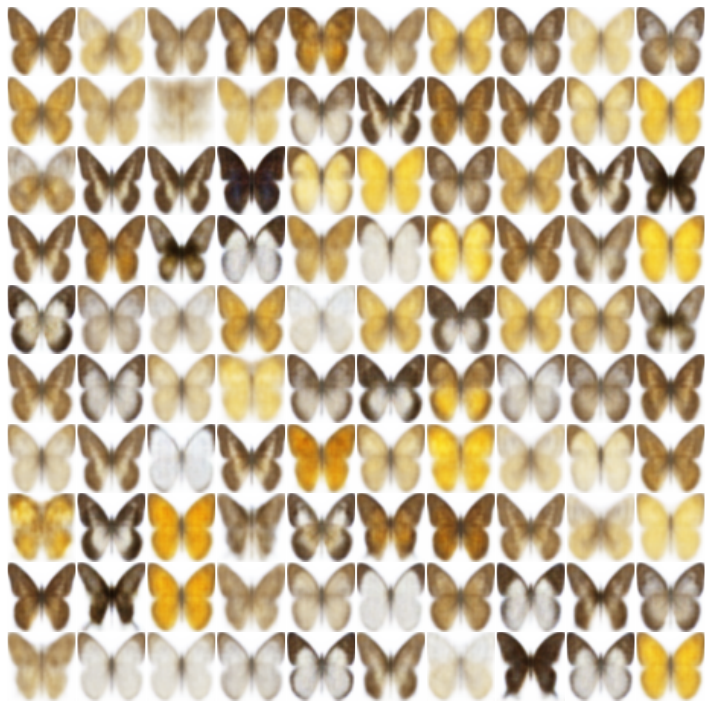}
      \caption*{Average}
      \label{fig:butterflies_VAE_mu}
    \end{minipage}%
    \begin{minipage}[b]{0.48\linewidth}
      \centering
      \includegraphics[width=\linewidth,clip]{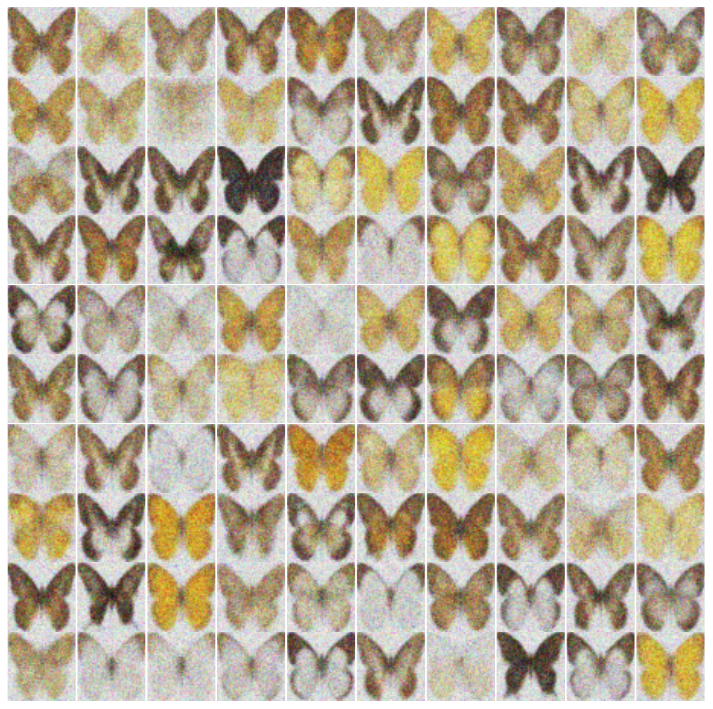}
      \caption*{Sample}
      \label{fig:butterflies_VAE_sample}
    \end{minipage}%
    \subcaption{Butterflies}
  \end{minipage}%
  \hfill
  \begin{minipage}[b]{0.4\linewidth}
    \centering
    \begin{minipage}[b]{0.45\linewidth}
      \centering
      \includegraphics[width=\linewidth,clip]{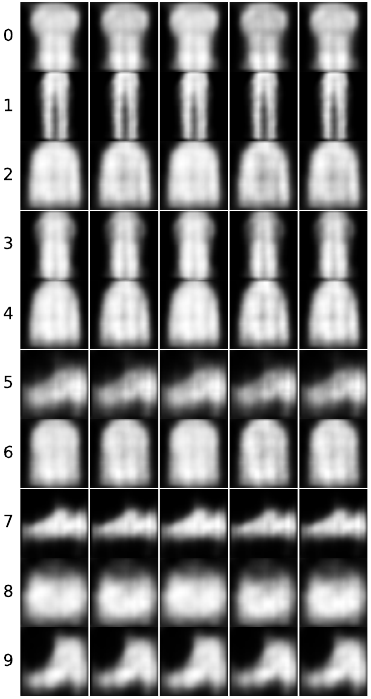}
      \caption*{Average}
      \label{fig:fmnist_VAE_mu}
    \end{minipage}%
    \begin{minipage}[b]{0.45\linewidth}
      \centering
      \includegraphics[width=\linewidth,clip]{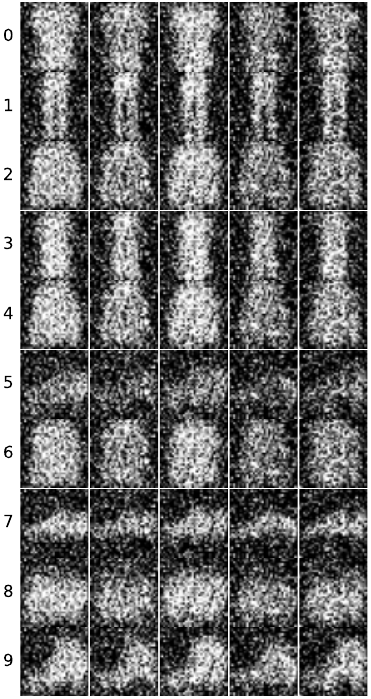}
      \caption*{Sample}
      \label{fig:fmnist_VAE_sample}
    \end{minipage}%
    \subcaption{FMNIST}
  \end{minipage}
  \caption{Samples generated by $\beta$-VAEs.}
  \label{fig:vae_comparison}
\end{figure}

\subsubsection{Quantitative evaluation of image generation}
\paragraph{Butterflies dataset results}
We quantitatively evaluated the generated image distributions using the metrics described in Section~\ref{sec:metric}. The results for the Butterflies dataset are summarized in Table~\ref{tab:butterflies_metric}. In the table, the ``$^\dagger$'' symbol denotes the model achieving the best mean score for each metric, and the ``$^\ddagger$'' symbol indicates models whose mean score lies within twice the standard deviation of the best-scoring model.

The table shows that, across all metrics, either RPD or RPD\_vpred attains the best performance or lies within a range where the difference from the best model is not substantial. Notably, whereas the performance of the other models (except IMM) deteriorates significantly as the number of inference steps decreases, RPD and RPD\_vpred exhibit only minor degradation in KID and maintain relatively strong performance even with as few as three inference steps. Indeed, the KID@3 scores of RPD and RPD\_vpred are comparable to or better than the KID@50 score of the DDPM model.
Both RPD and RPD\_vpred also show a tendency for 1WD to improve as the number of inference steps decreases, possibly because few-step sampling yields only partial corrections by the diffusion process; this result is further analyzed in Section~\ref{sec:visualize_rpd_gen_process}.

\begin{table}[tbp]
  \centering
  \caption{Evaluation results on the Butterflies dataset. For each metric, the number following ``@'' indicates the number of inference steps, and the values in parentheses denote standard deviations.}
  \label{tab:butterflies_metric}
  \scalebox{0.9}{
    \begin{tabular}{lr@{}lr@{}lr@{}lr@{}lr@{}lr@{}l}
      \toprule
      Model          & \multicolumn{1}{c}{KID@3} &             & \multicolumn{1}{c}{KID@10} &            & \multicolumn{1}{c}{KID@50} &            & \multicolumn{1}{c}{1WD@3} &             & \multicolumn{1}{c}{1WD@10} &            & \multicolumn{1}{c}{1WD@50} &             \\
      \midrule
      DDPM           & 0.179 (0.013)             &             & 0.036 (0.007)              &            & 0.018 (0.003)              &            & 141.86 (2.79)             &             & 126.75 (3.00)              &            & 123.64 (3.06)              &             \\
      DDIM           & 0.252 (0.011)             &             & 0.082 (0.013)              &            & 0.033 (0.007)              &            & 147.33 (2.76)             &             & 128.89 (3.32)              &            & 125.45 (2.80)              &             \\
      $v$-prediction & 0.135 (0.016)             &             & 0.021 (0.003)              &            & 0.016 (0.002)              &            & 124.50 (3.31)             &             & 120.51 (4.88)              &            & 120.77 (5.04)              &             \\
      DiffuseVAE     & 0.163 (0.013)             &             & 0.011 (0.002)              &            & 0.010 (0.002)              &            & 142.99 (2.46)             &             & 124.74 (4.68)              &            & 122.49 (5.75)              &             \\
      Rectified Flow & 0.123 (0.020)             &             & 0.025 (0.003)              &            & 0.011 (0.004)              &            & 109.14 (0.90)             &             & 105.81 (1.25)              &            & 95.50 (3.59)               & $^\dagger$  \\
      IMM            & 0.035 (0.011)             &             & 0.024 (0.006)              &            & 0.025 (0.005)              &            & 117.33 (9.07)             &             & 125.83 (9.14)              &            & 138.23 (9.17)              &             \\
      \addlinespace
      RPD            & 0.015 (0.002)             & $^\ddagger$ & 0.005 (0.001)              & $^\dagger$ & 0.004 (0.001)              & $^\dagger$ & 86.99 (1.01)              & $^\dagger$  & 92.81 (0.87)               & $^\dagger$ & 96.03 (1.57)               & $^\ddagger$ \\
      RPD\_vpred     & 0.013 (0.002)             & $^\dagger$  & 0.007 (0.001)              &            & 0.006 (0.001)              &            & 88.70 (0.75)              & $^\ddagger$ & 94.59 (1.23)               &            & 97.28 (2.43)               & $^\ddagger$ \\
      \bottomrule
    \end{tabular}
  }
\end{table}

\paragraph{FMNIST dataset results}
The evaluation results for the FMNIST dataset are summarized in Table~\ref{tab:fmnist_metric}. The meanings of the symbols are identical to those in Table~\ref{tab:butterflies_metric}.
Consistent with the observations on the Butterflies dataset, most of the existing methods achieve good performance when the number of inference steps is set to 50, but their performance deteriorates sharply when the number of steps is reduced to three. IMM shows relatively strong performance even with three steps, and on the FMNIST dataset its evaluation scores tend to improve as the number of inference steps decreases. However, IMM does not necessarily achieve scores comparable to those of the other baselines even when the number of inference steps is increased to 50, a tendency that aligns with the qualitative results shown in Fig.~\ref{fig:fmnist_inf_steps_50}.
Either RPD or RPD\_vpred receives the ``$^\dagger$'' or ``$^\ddagger$'' mark for most metrics, with the exception of CW-1WD@50 and CW-MMD@3. Even for these two metrics, although no mark is assigned, the scores of RPD and RPD\_vpred remain relatively favorable compared with the other models.

\begin{table}[tbp]
  \centering
  \caption{Evaluation results on the FMNIST dataset. For each metric, the number following ``@'' indicates the number of inference steps, and the values in parentheses denote standard deviations.}
  \label{tab:fmnist_metric}
  \scalebox{0.8}{
    \begin{tabular}{lr@{}lr@{}lr@{}lr@{}lr@{}lr@{}l}
      \toprule
      Model          & \multicolumn{1}{c}{CW-1WD@3} &             & \multicolumn{1}{c}{CW-1WD@10} &             & \multicolumn{1}{c}{CW-1WD@50} &            & \multicolumn{1}{c}{CW-MMD@3} &            & \multicolumn{1}{c}{CW-MMD@10} &             & \multicolumn{1}{c}{CW-MMD@50} &             \\
      \midrule
      DDPM           & 18.90 (0.42)                 &             & 9.95 (0.31)                   &             & 9.26 (0.33)                   &            & 0.495 (0.034)                &            & 0.046 (0.007)                 &             & 0.018 (0.005)                 & $^\ddagger$ \\
      DDIM           & 20.10 (0.31)                 &             & 12.90 (0.38)                  &             & 11.69 (0.40)                  &            & 0.470 (0.029)                &            & 0.087 (0.012)                 &             & 0.042 (0.007)                 &             \\
      $v$-prediction & 18.67 (0.33)                 &             & 9.60 (0.18)                   &             & 9.06 (0.09)                   &            & 0.437 (0.024)                &            & 0.033 (0.006)                 & $^\ddagger$ & 0.013 (0.002)                 & $^\ddagger$ \\
      DiffuseVAE     & 18.66 (0.80)                 &             & 10.41 (0.89)                  &             & 9.83 (0.55)                   &            & 0.490 (0.041)                &            & 0.056 (0.017)                 &             & 0.025 (0.006)                 &             \\
      Rectified Flow & 10.96 (0.10)                 &             & 8.88 (0.09)                   & $^\ddagger$ & 8.95 (0.05)                   & $^\dagger$ & 0.246 (0.008)                &            & 0.036 (0.007)                 & $^\ddagger$ & 0.013 (0.004)                 & $^\dagger$  \\
      IMM            & 10.80 (0.74)                 &             & 12.74 (1.09)                  &             & 13.96 (1.24)                  &            & 0.018 (0.007)                & $^\dagger$ & 0.044 (0.014)                 &             & 0.068 (0.017)                 &             \\
      \addlinespace
      RPD            & 9.24 (0.12)                  & $^\ddagger$ & 8.92 (0.26)                   & $^\ddagger$ & 9.38 (0.36)                   &            & 0.080 (0.003)                &            & 0.033 (0.004)                 & $^\ddagger$ & 0.022 (0.005)                 &             \\
      RPD\_vpred     & 9.19 (0.06)                  & $^\dagger$  & 8.82 (0.11)                   & $^\dagger$  & 9.29 (0.23)                   &            & 0.077 (0.004)                &            & 0.026 (0.005)                 & $^\dagger$  & 0.018 (0.004)                 & $^\ddagger$ \\
      \bottomrule
    \end{tabular}
  }
\end{table}

\subsubsection{Visualization of the RPD generation process} \label{sec:visualize_rpd_gen_process}
Figure~\ref{fig:butterflies_rpd_gen_process} illustrates the generation process of RPD trained on the Butterflies dataset. The figure shows results obtained with three and 50 inference steps. Each row corresponds to the progression of inference steps, while each column reflects variability induced by different random seeds. For the case of 50 inference steps, only the outputs at steps 0, 20, 40, and 50 are displayed.
Because RPD uses the output of the prior model as the initial state of the reverse diffusion process, the first row corresponds to samples generated by the $\beta$-VAE prior model used in this experiment. As the figure suggests, the prior model determines the coarse structure and color tone of the images, and the diffusion process progressively refines these initial samples, generating increasingly detailed images.

\begin{figure}[tbp]
  \begin{minipage}[b]{0.49\linewidth}
    \centering
    \includegraphics[width=\linewidth]{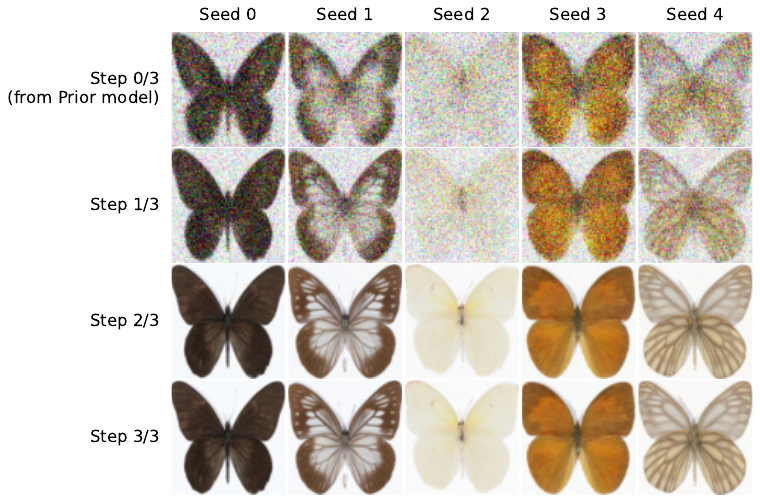}
    \subcaption{Results with three inference steps}
  \end{minipage}%
  \hfill
  \begin{minipage}[b]{0.49\linewidth}
    \centering
    \includegraphics[width=\linewidth]{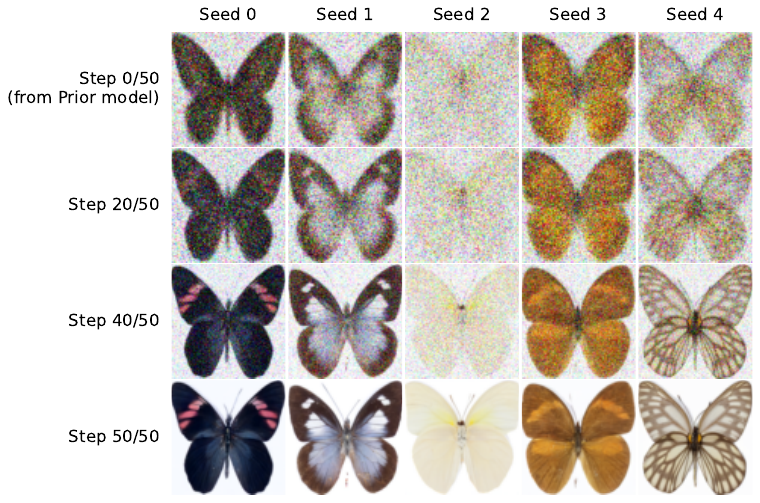}
    \subcaption{Results with 50 inference steps}
  \end{minipage}%
  \caption{Visualization of the RPD generation process on the Butterflies dataset. Panel (a) shows results with three inference steps, and panel (b) shows results with 50 inference steps (displaying only steps 0, 20, 40, and 50).}
  \label{fig:butterflies_rpd_gen_process}
\end{figure}

To further investigate how the RPD generation process depends on the number of inference steps, we fixed the latent variable $z$ sampled from $\hat{p}(z)$ in \eqref{eq:rev_1} and varied only the random seed used in the subsequent generation procedure. The resulting samples are shown in Fig.~\ref{fig:butterflies_rpd_gen_process_same_prior_sample}. From Fig.~\ref{fig:butterflies_rpd_gen_process_same_prior_sample}(a), we observe that with three inference steps, the final generated images exhibit only minor variations. In contrast, as shown in Fig.~\ref{fig:butterflies_rpd_gen_process_same_prior_sample}(b), when the number of inference steps is increased to 50, the generated images become more diverse while still preserving the global structure present at step 0. These results indicate that although RPD can already produce high-fidelity images with as few as three inference steps, increasing the number of steps to 50 enables the model to generate images that are both more diverse and of higher quality.

\begin{figure}[tbp]
  \begin{minipage}[b]{0.49\linewidth}
    \centering
    \includegraphics[width=\linewidth]{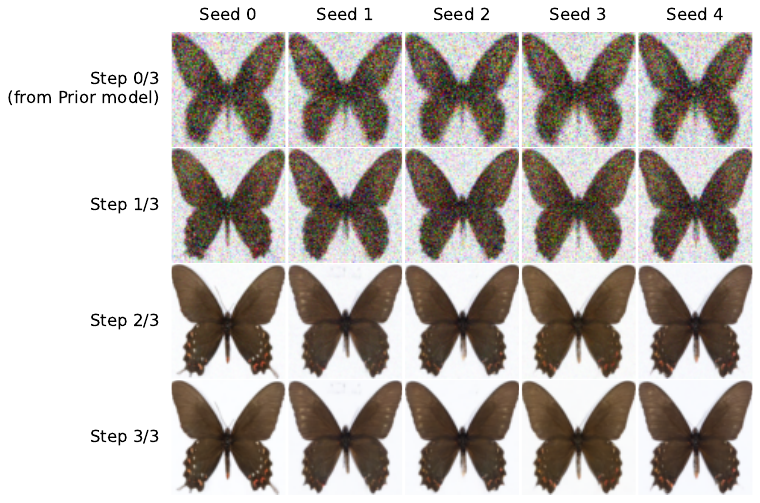}
    \subcaption{Results with three inference steps}
  \end{minipage}%
  \hfill
  \begin{minipage}[b]{0.49\linewidth}
    \centering
    \includegraphics[width=\linewidth]{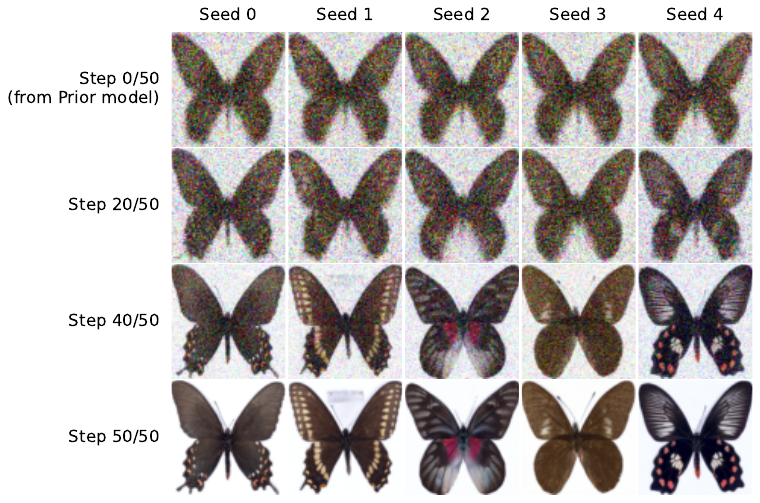}
    \subcaption{Results with 50 inference steps}
  \end{minipage}%
  \caption{RPD generation process on the Butterflies dataset using a common latent variable $z$ sampled from $\hat{p}(z)$. Panel (a) presents the results obtained with three inference steps, and panel (b) presents those obtained with 50 inference steps (showing only steps 0, 20, 40, and 50).}
  \label{fig:butterflies_rpd_gen_process_same_prior_sample}
\end{figure}

\subsubsection{Investigating the impact of model sizes} \label{sec:model_size}
This section investigates how changes in the model size of the diffusion networks and the prior model used in RPD (a $\beta$-VAE in our experiments) affect generation performance. Using the Butterflies dataset, we constructed a smaller diffusion model by modifying the \texttt{UNet2DModel} configuration. Specifically, we updated the model settings from those in Table~\ref{tab:unet_configs} as follows: we reduced \texttt{block\_out\_channels} from (128, 128, 256, 256, 512, 512) to (96, 96, 192, 192, 384) and removed unnecessary \texttt{DownBlock2D} and \texttt{UpBlock2D} layers from the corresponding \texttt{down\_block\_types} and \texttt{up\_block\_types}. We refer to DDPM and RPD using this smaller \texttt{UNet2DModel} as DDPM\_s and RPD\_s, respectively.
In addition, we constructed a reduced-size version of $\beta$-VAE by modifying its ResNet components to decrease the number of parameters. We refer to this smaller prior model as $\beta$-VAE\_s.

Table~\ref{tab:butterflies_compare_small_models} compares the standard and reduced-size models for both DDPM and RPD. For each of RPD and RPD\_s, the table reports performance when using either $\beta$-VAE or $\beta$-VAE\_s as the prior model. In addition, Table~\ref{tab:butterflies_vae} summarizes the performance of $\beta$-VAE and $\beta$-VAE\_s.
The total number of model parameters is also listed in these tables (for RPD models, the total includes the parameters of the prior model). From these parameter counts, we observe that DDPM\_s reduces the number of parameters to roughly one third of that of DDPM, and that $\beta$-VAE\_s contains approximately 20\% of the parameters of $\beta$-VAE.

From Table~\ref{tab:butterflies_compare_small_models}, we observe that DDPM\_s exhibits degradation across all metrics compared with DDPM, indicating that reducing the model size directly leads to decreased performance.
In contrast, for RPD\_s (w/ $\beta$-VAE), where the prior model is kept fixed while the diffusion network is downsized, modest degradation is observed in a few metrics, but the overall performance remains largely comparable to that of the full-sized model.
Notably, when comparing DDPM with RPD\_s (w/ $\beta$-VAE), the latter achieves superior results across all metrics, suggesting that RPD can maintain performance comparable to or better than standard DDPM even under significantly reduced model capacity.
Furthermore, as discussed above, RPD maintains stable performance even when the number of inference steps is reduced to 3, and this behavior is consistently observed even for smaller RPD models. Since smaller models generally require less inference time, RPD has the potential to further reduce overall generation time by jointly reducing the model size and the number of inference steps.

\begin{table}[t]
  \centering
  \caption{Performance comparison under different model sizes on the Butterflies dataset}
  \label{tab:butterflies_compare_small_models}
  \scalebox{0.8}{
    \begin{tabular}{lcrrrrrr}
      \toprule
      Model                      & \#Params [M] & \multicolumn{1}{c}{KID@3} & \multicolumn{1}{c}{KID@10} & \multicolumn{1}{c}{KID@50} & \multicolumn{1}{c}{1WD@3} & \multicolumn{1}{c}{1WD@10} & \multicolumn{1}{c}{1WD@50} \\
      \midrule
      DDPM                       & 113.7        & 0.179 (0.013)             & 0.036 (0.007)              & 0.018 (0.003)              & 141.86 (2.79)             & 126.75 (3.00)              & 123.64 (3.06)              \\
      DDPM\_s                    & 40.2         & 0.268 (0.089)             & 0.062 (0.025)              & 0.026 (0.006)              & 150.29 (3.50)             & 138.14 (1.61)              & 137.06 (5.71)              \\
      \addlinespace
      RPD (w/ $\beta$-VAE)       & 122.9        & 0.015 (0.002)             & 0.005 (0.001)              & 0.004 (0.001)              & 86.99 (1.01)              & 92.81 (0.87)               & 96.03 (1.57)               \\
      RPD (w/ $\beta$-VAE\_s)    & 115.4        & 0.025 (0.003)             & 0.008 (0.001)              & 0.007 (0.002)              & 86.85 (0.82)              & 91.99 (0.95)               & 95.81 (1.83)               \\
      RPD\_s (w/ $\beta$-VAE)    & 49.4         & 0.020 (0.003)             & 0.009 (0.001)              & 0.007 (0.001)              & 86.77 (0.97)              & 90.44 (0.50)               & 93.88 (1.37)               \\
      RPD\_s (w/ $\beta$-VAE\_s) & 41.9         & 0.035 (0.005)             & 0.017 (0.004)              & 0.014 (0.003)              & 86.59 (0.63)              & 89.60 (0.97)               & 93.07 (2.39)               \\
      \bottomrule
    \end{tabular}
  }
\end{table}
\begin{table}[t]
  \centering
  \caption{Evaluation results for the VAE prior models on the Butterflies dataset}
  \label{tab:butterflies_vae}
  \scalebox{0.8}{
    \begin{tabular}{lcrr}
      \toprule
      \addlinespace
      Model          & \#Params [M] & \multicolumn{1}{c}{KID} & \multicolumn{1}{c}{1WD} \\
      \midrule
      $\beta$-VAE    & 9.2          & 0.338 (0.013)           & 178.64 (0.77)           \\
      $\beta$-VAE\_s & 1.7          & 0.384 (0.010)           & 179.03 (0.18)           \\
      \bottomrule
    \end{tabular}
  }
\end{table}

\paragraph{Assessing the impact of prior model size}
As shown in Table~\ref{tab:butterflies_compare_small_models}, when the capacity of the prior model is reduced while keeping the size of the RPD diffusion network fixed (RPD (w/ $\beta$-VAE\_s)), KID@3 exhibits more significant degradation than in RPD\_s (w/ $\beta$-VAE), despite the latter having a substantially smaller overall model size. This indicates that reducing the capacity of the prior model negatively affects generation quality in few-step sampling regimes.

To investigate this behavior in detail, Fig.~\ref{fig:butterflies_rpd_small_prior} presents generation results obtained using $\beta$-VAE or $\beta$-VAE\_s as prior models in RPD.
As can be seen in the figure, the images generated with $\beta$-VAE\_s exhibit reduced color diversity compared with those generated with $\beta$-VAE, suggesting that the smaller prior model’s limited capacity affects the quality of its outputs.
In the lower row, where $\beta$-VAE\_s is used as the prior model, the images generated by RPD also show reduced color diversity, particularly when the number of inference steps is three. As shown in Fig.~\ref{fig:butterflies_rpd_gen_process_same_prior_sample}, RPD tends to produce samples close to those of the prior model when the number of inference steps is small. Consequently, when a prior model with limited representational capacity is used, its influence is expected to persist strongly in the RPD outputs under few-step sampling.
However, even when $\beta$-VAE\_s is used as the prior model, increasing the number of inference steps to 50 improves the diversity of the generated images, including color variation. This is consistent with Fig.~\ref{fig:butterflies_rpd_gen_process_same_prior_sample}, which shows that a larger number of inference steps enables RPD to generate images that deviate from the expressiveness of the prior model.
These observations indicate that when RPD is used under few-step sampling regimes, it is important to employ a prior model with sufficient representational capacity.

\begin{figure}[htbp]
  \centering
  \begin{minipage}[b]{0.33\linewidth}
    \centering
    \includegraphics[width=\linewidth]{butterflies_VAE_sample.pdf}
    \subcaption*{$\beta$-VAE}
  \end{minipage}%
  \hfill
  \begin{minipage}[b]{0.33\linewidth}
    \centering
    \includegraphics[width=\linewidth]{butterflies_RPD_inf_steps_3.pdf}
    \subcaption*{RPD (w/ $\beta$-VAE)@3}
  \end{minipage}%
  \hfill
  \begin{minipage}[b]{0.33\linewidth}
    \centering
    \includegraphics[width=\linewidth]{butterflies_RPD_inf_steps_50.pdf}
    \subcaption*{RPD (w/ $\beta$-VAE)@50}
  \end{minipage}%
  \vspace{2mm}
  \begin{minipage}[b]{0.33\linewidth}
    \centering
    \includegraphics[width=\linewidth]{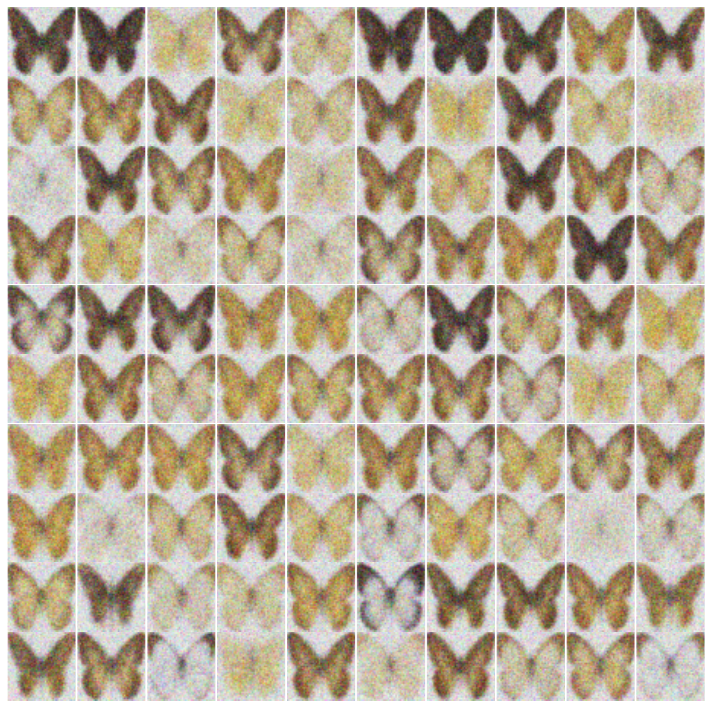}
    \subcaption*{$\beta$-VAE\_s}
  \end{minipage}%
  \hfill
  \begin{minipage}[b]{0.33\linewidth}
    \centering
    \includegraphics[width=\linewidth]{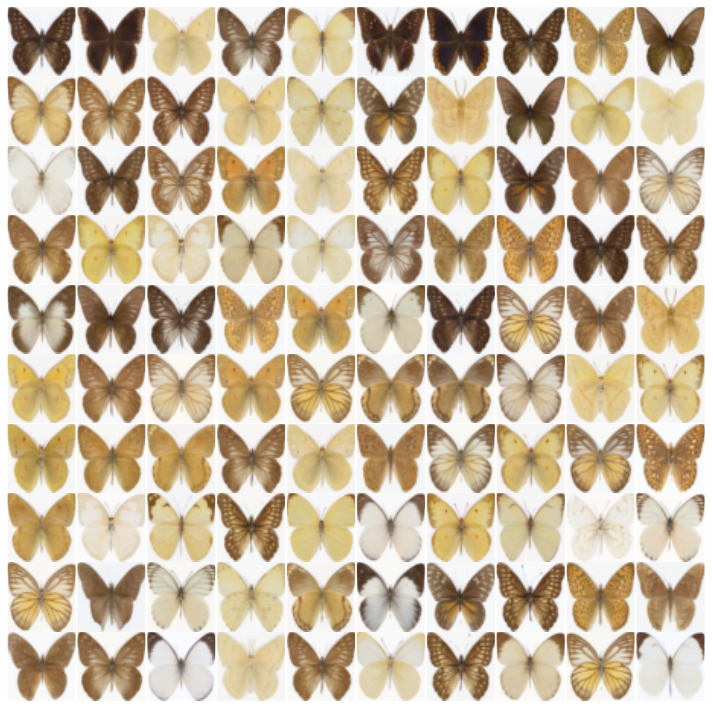}
    \subcaption*{RPD (w/ $\beta$-VAE\_s)@3}
  \end{minipage}%
  \hfill
  \begin{minipage}[b]{0.33\linewidth}
    \centering
    \includegraphics[width=\linewidth]{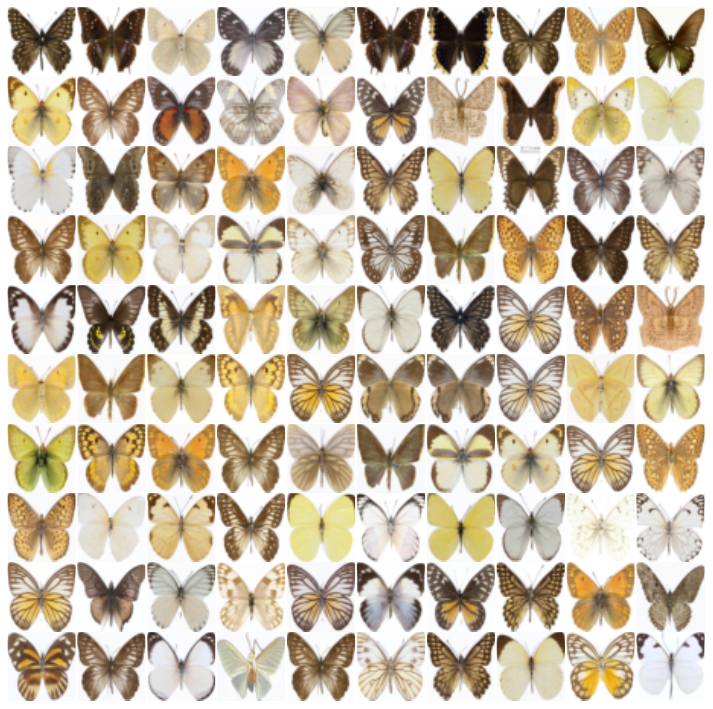}
    \subcaption*{RPD (w/ $\beta$-VAE\_s)@50}
  \end{minipage}%
  \caption{Generation results obtained with RPD and its prior models. The upper and lower rows show the results when $\beta$-VAE and $\beta$-VAE\_s are used as prior models, respectively. The number following ``@'' indicates the number of inference steps used in the RPD generation process.}
  \label{fig:butterflies_rpd_small_prior}
\end{figure}

\section{Conclusion and future work} \label{sec:conclusion}

We introduced RPD, a new probabilistic generative framework that represents the coarse structure of a data distribution using a prior model and learns the residual discrepancy between the prior and the true data distribution through a diffusion model. We formulated RPD as a variational inference framework based on the ELBO and derived training and inference algorithms under both noise- and velocity-prediction parameterizations. We also introduced auxiliary variables that support prediction in RPD and demonstrated, both theoretically and empirically, that these variables facilitate the learning process.
Through numerical experiments, we verified the advantages of RPD over existing models across two types of task: two-dimensional hetero-scale data and natural image generation. For hetero-scale data, we confirmed that RPD can capture fine-grained local structures, particularly when the scales of local distributions are much smaller than the global scale. In natural image generation, we showed that incorporating a prior model improves image quality and, notably, that RPD can generate high-fidelity images even when the number of inference steps is drastically reduced.

Several directions remain for future work.
A first avenue is to deepen the theoretical understanding of RPD. For example, one may analyze RPD from the perspective of the diffusion model regimes discussed in \citep{biroli2024dynamical}. \citet{biroli2024dynamical} theoretically showed that the generation dynamics of diffusion models can be divided into three distinct regimes. In RPD, the use of a prior model may effectively shorten the initial regime (i.e., the regime before the speciation transition), suggesting that the regime structure of RPD could differ meaningfully from that of standard diffusion models. A more precise characterization of how the prior model reshapes these dynamical regimes would contribute to a deeper theoretical understanding of the proposed framework.
From an empirical standpoint, the present study evaluated RPD primarily on relatively small-scale datasets. While these experiments are sufficient to demonstrate the core advantages of residual prior modeling---particularly in regimes with strong scale mismatch---assessing the scalability of RPD on substantially larger and more diverse datasets remains an important direction for future work. Such evaluations are necessary to clarify how the interaction between the prior model and the diffusion process behaves in high-capacity settings and to understand potential trade-offs between prior expressiveness, diffusion capacity, and computational efficiency at scale.
From a more application-oriented perspective, another promising direction is to further advance conditional generation with RPD. As demonstrated in our experiments with the FMNIST dataset, RPD naturally supports conditioning in both the prior model and the diffusion process. In image generation, for instance, conditioning the prior model could control coarse global attributes such as scene layout or high-level semantic categories, while conditioning the diffusion process could refine fine-grained details such as textures and local appearance. This separation suggests a hierarchical form of conditional control that is difficult to realize within standard diffusion models.
Moreover, models such as $\beta$-VAE are known to yield interpretable, factorized latent representations as a consequence of training \citep{higgins2017betavae}. Using such models as priors for RPD would allow these factorized representations to be directly inherited by the diffusion process, potentially enabling structured and interpretable control over generation.

\appendix

\section{Formula derivations} \label{apdx:proofs}

\subsection{Derivation of ELBO for RPD} \label{apdx:elbo}
\begin{align*}
  \log p_\theta(x_0) & \geq \mathbb{E}_{q(x_{1:T}, z \mid x_0)}\left[\log \frac{p_\theta(x_{0:T}, z)}{q(x_{1:T}, z \mid x_0)}\right]                                                                                                                                                                                                        \\
                     & = \mathbb{E}_{q(x_{1:T}, z \mid x_0)}\left[\log \frac{\hat{p}(z) \hat{p}(x_T \mid z) \prod_{t=1}^T p_\theta(x_{t-1}|x_t, z)}{\hat{q}(z \mid x_0) \prod_{t=1}^T q(x_{t} \mid x_{t-1}, z)}\right]                                                                                                                      \\
                     & = \mathbb{E}_{q(x_{1:T}, z \mid x_0)}\left[\log \frac{\hat{p}(z) \hat{p}(x_T \mid z) p_\theta(x_0 \mid x_1, z) \prod_{t=2}^T p_\theta(x_{t-1}|x_t, z)}{\hat{q}(z \mid x_0) q(x_1 \mid x_0, z) \prod_{t=2}^T q(x_{t} \mid x_{t-1}, z)}\right]                                                                         \\
                     & = \mathbb{E}_{q(x_{1:T}, z \mid x_0)}\left[\log \frac{\hat{p}(z) \hat{p}(x_T \mid z) p_\theta(x_0 \mid x_1, z) }{\hat{q}(z \mid x_0) q(x_1 \mid x_0, z)} + \log \prod_{t=2}^T  \frac{p_\theta(x_{t-1}|x_t, z)}{q(x_{t} \mid x_{t-1}, z)}\right]                                                                      \\
                     & = \mathbb{E}_{q(x_{1:T}, z \mid x_0)}\left[\log \frac{\hat{p}(z) \hat{p}(x_T \mid z) p_\theta(x_0 \mid x_1, z) }{\hat{q}(z \mid x_0) q(x_1 \mid x_0, z)} + \log \prod_{t=2}^T  \frac{p_\theta(x_{t-1}|x_t, z)}{q(x_{t} \mid x_{t-1}, x_0, z)}\right]  \quad \text{(Conditioning trick)}                              \\
                     & = \mathbb{E}_{q(x_{1:T}, z \mid x_0)}\left[\log \frac{\hat{p}(z) \hat{p}(x_T \mid z) p_\theta(x_0 \mid x_1, z) }{\hat{q}(z \mid x_0) q(x_1 \mid x_0, z)} + \log \prod_{t=2}^T  \frac{p_\theta(x_{t-1}|x_t, z)}{\frac{q(x_{t-1} \mid x_{t}, x_0, z) \cancel{q(x_t \mid x_0, z)}}{\cancel{q(x_{t-1}|x_0,z)}}}\right]   \\
                     & = \mathbb{E}_{q(x_{1:T}, z \mid x_0)}\left[\log \frac{\hat{p}(z) \hat{p}(x_T \mid z) p_\theta(x_0 \mid x_1, z) }{\hat{q}(z \mid x_0) \cancel{q(x_1 \mid x_0, z)}} + \log \frac{\cancel{q(x_1|x_0, z)}}{q(x_T|x_0, z)} + \log \prod_{t=2}^T  \frac{p_\theta(x_{t-1}|x_t, z)}{q(x_{t-1} \mid x_{t}, x_0, z)}\right]    \\
                     & = \mathbb{E}_{q(x_{1:T}, z \mid x_0)}\left[\log \frac{\hat{p}(z) \hat{p}(x_T \mid z) p_\theta(x_0 \mid x_1, z) }{\hat{q}(z \mid x_0) q(x_T|x_0, z)} + \sum_{t=2}^T \log \frac{p_\theta(x_{t-1}|x_t, z)}{q(x_{t-1} \mid x_{t}, x_0, z)}\right]                                                                        \\
                     & = \mathbb{E}_{\hat{q}(z \mid x_0)}\left[\log \frac{\hat{p}(z)}{\hat{q}(z \mid x_0)}\right] + \mathbb{E}_{q(x_T, z \mid x_0)}\left[\log \frac{\hat{p}(x_T \mid z)}{q(x_T|x_0, x)}\right] + \mathbb{E}_{q(x_1, z \mid x_0)}\left[\log p_\theta(x_0 \mid x_1, z)\right] \notag                                          \\
                     & \qquad  + \sum_{t=2}^T \mathbb{E}_{q(x_t, x_{t-1}, z \mid x_0)}\left[\log \frac{p_\theta(x_{t-1}|x_t, z)}{q(x_{t-1} \mid x_{t}, x_0, z)}\right]                                                                                                                                                                      \\
                     & = \mathbb{E}_{\hat{q}(z \mid x_0)}\left[\log \frac{\hat{p}(z)}{\hat{q}(z \mid x_0)}\right] + \mathbb{E}_{\hat{q}(z \mid x_0) q(x_T \mid x_0, z)}\left[\log \frac{\hat{p}(x_T \mid z)}{q(x_T|x_0, x)}\right] + \mathbb{E}_{\hat{q}(z \mid x_0) q(x_1 \mid x_0, z)}\left[\log p_\theta(x_0 \mid x_1, z)\right]  \notag \\
                     & \qquad  + \sum_{t=2}^T \mathbb{E}_{\hat{q}(z \mid x_0) q(x_t \mid x_0, z)} \mathbb{E}_{q(x_{t-1} \mid x_t, x_0, z)}\left[\log \frac{p_\theta(x_{t-1}|x_t, z)}{q(x_{t-1} \mid x_{t}, x_0, z)}\right]                                                                                                                  \\
                     & = - D_{\mathrm{KL}} \left(\hat{q}(z \mid x_0) \ || \ \hat{p}(z)\right) - \mathbb{E}_{\hat{q}(z \mid x_0)}\left[ D_{\mathrm{KL}} \left(q(x_T \mid x_0, z) \ || \ \hat{p}(x_T \mid z)\right) \right] \notag                                                                                                            \\
                     & \qquad  + \mathbb{E}_{\hat{q}(z \mid x_0) q(x_1 \mid x_0, z)}\left[\log p_\theta(x_0 \mid x_1, z)\right]   \notag                                                                                                                                                                                                    \\
                     & \qquad \qquad - \sum_{t=2}^T \mathbb{E}_{\hat{q}(z \mid x_0) q(x_{t} \mid x_0, z)} \left[ D_{\mathrm{KL}}\left(q(x_{t-1} \mid x_{t}, x_0, z) \ || \ p_\theta(x_{t-1}|x_t, z) \right)\right].
\end{align*}

\subsection{Derivation of $q(x_t \mid x_0, z)$} \label{apdx:q_xt_x0_z}
Let $\{\epsilon_t^\ast, \epsilon_t\}_{t=0}^{T-1} \overset{\mathrm{iid}}{\sim} \mathcal{N}(\epsilon \mid 0, I)$ and $z \sim \hat{q}(z \mid x_0)$.
Then, for $t = 1, \ldots, T$, the following relations hold:
\begin{align}
  x_t & = \sqrt{\alpha_t} x_{t-1} + \left(1 - \sqrt{\alpha_t}\right) \hat{\mu}(z) +\sqrt{\beta_t \hat{\sigma}^2(z)} \epsilon_{t-1}                                                                                                                                                \notag                                    \\
      & = \sqrt{\alpha_t} x_{t-1} + \sqrt{\left(1-\alpha_t\right) \hat{\sigma}^2(z)} \epsilon_{t-1} + \left(1 -\sqrt{\alpha_t}\right) \hat{\mu}(z)                                                                                                     \notag                                                               \\
      & = \sqrt{\alpha_t} \left(\sqrt{\alpha_{t-1}} x_{t-2} +\sqrt{\left(1 - \alpha_{t-1}\right) \hat{\sigma}^2(z)} \epsilon_{t-2}^\ast + \left(1 - \sqrt{\alpha_{t-1}}\right) \hat{\mu}(z)\right) + \sqrt{\left(1-\alpha_t\right) \hat{\sigma}^2(z)} \epsilon_{t-1} + \left(1 - \sqrt{\alpha_t}\right) \hat{\mu}(z) \notag \\
      & = \sqrt{\alpha_t \alpha_{t-1}} x_{t-2} + \sqrt{\alpha_t - \alpha_t\alpha_{t-1}} \sqrt{\hat{\sigma}^2(z)} \epsilon_{t-2}^\ast + \sqrt{1-\alpha_t} \sqrt{\hat{\sigma}^2(z)} \epsilon_{t-1} + \left(1 - \sqrt{\alpha_t \alpha_{t-1}} \right) \hat{\mu}(z) \label{eq:merge_epsilon1}                                    \\
      & = \sqrt{\alpha_t \alpha_{t-1}} x_{t-2} + \sqrt{1 - \alpha_t \alpha_{t-1}} \sqrt{\hat{\sigma}^2(z)} \epsilon_{t-2}  + \left(1 - \sqrt{\alpha_t \alpha_{t-1}} \right) \hat{\mu}(z)  \label{eq:merge_epsilon2}                                                                                                         \\
      & = \ldots                                                                                                                                                                                                                      \notag                                                                                \\
      & = \sqrt{\prod_{i=1}^t \alpha_i} x_0 + \sqrt{1 - \prod_{i=1}^t \alpha_i} \sqrt{\hat{\sigma}^2(z)} \epsilon_0 + \left(1 - \sqrt{\prod_{i=1}^t \alpha_i}\right) \hat{\mu}(z)                   \notag                                                                                                                  \\
      & = \sqrt{\bar{\alpha}_t} x_0 + \left(1 - \sqrt{\bar{\alpha}_t}\right) \hat{\mu}(z) + \sqrt{\left(1 - \bar{\alpha}_t\right) \hat{\sigma}^2(z)} \epsilon_0. \notag
\end{align}
The transformation from \eqref{eq:merge_epsilon1} to \eqref{eq:merge_epsilon2} follows the derivation presented in \citep{luo2022understanding}.

\subsection{Derivation of $q(x_{t-1} \mid x_t, x_0, z)$} \label{apdx:q_xt-1_xt_x0_z}
For $t = 2, \ldots, T$, the following holds.
\begin{align*}
  q(x_{t-1} \mid x_{t}, x_0, z) & = \frac{q(x_t \mid x_{t-1}, x_0, z)q(x_{t-1} \mid x_0, z)}{q(x_t \mid x_0, z)}                                                                                                                          \\
                                & \propto q(x_t \mid x_{t-1}, x_0, z)q(x_{t-1} \mid x_0, z)                                                                                                                                               \\
                                & = \mathcal{N}\left(x_t \ \middle| \ \sqrt{\alpha_t}x_{t-1} + \left(1 - \sqrt{\alpha_t}\right) \hat{\mu}(z), (1 - \alpha_t) \hat{\sigma}^2(z) I_n)\right) \notag                                         \\
                                & \qquad \mathcal{N}\left(x_{t-1} \ \middle| \ \sqrt{\bar{\alpha}_{t-1}} x_0  + \left(1 - \sqrt{\bar{\alpha}_{t-1}}\right) \hat{\mu}(z), \left(1 - \bar{\alpha}_{t-1}\right) \hat{\sigma}^2(z) I_n\right) \\
                                & = \mathcal{N}\left(x_{t-1} \ \middle| \ \frac{1}{\sqrt{\alpha_t}} x_{t} - \frac{1 - \sqrt{\alpha_t}}{\sqrt{\alpha_t}} \hat{\mu}(z), \frac{1 - \alpha_t}{\alpha_t} \hat{\sigma}^2(z) I)\right)  \notag   \\
                                & \qquad \mathcal{N}\left(x_{t-1} \ \middle| \ \sqrt{\bar{\alpha}_{t-1}} x_0  + \left(1 - \sqrt{\bar{\alpha}_{t-1}}\right) \hat{\mu}(z), \left(1 - \bar{\alpha}_{t-1}\right) \hat{\sigma}^2(z) I_n\right) \\
                                & = \mathcal{N}\left(x_{t-1} \ \middle| \ \tilde{\mu}(x_t, x_0, z), \tilde{\beta}_t(z) I \right),
\end{align*}
where $\tilde{\mu}(x_t, x_0, z)$ and $\tilde{\beta}_t(z)$ are obtained by applying the standard formula for the product of two Gaussian distributions, yielding the following expressions.
\begin{align*}
  \tilde{\mu}(x_t, x_0, z) & = \frac{\left((1-\bar{\alpha}_{t-1})\left(\frac{1}{\sqrt{\alpha}_t}x_t - \frac{1 - \sqrt{\alpha_t}}{\sqrt{\alpha_t}} \hat{\mu}(z)\right) + \frac{1 - \alpha_t}{\alpha_t}\left(\sqrt{\bar{\alpha}_{t-1}} x_0  + \left(1 - \sqrt{\bar{\alpha}_{t-1}}\right) \hat{\mu}(z)\right)\right)}{\frac{1}{\frac{1 - \alpha_t}{\alpha_t} + (1 - \bar{\alpha}_{t-1})}} \\
                           & = \frac{\sqrt{\alpha_t} (1-\bar{\alpha}_{t-1})}{1 - \bar{\alpha}_t} x_t + \frac{(1 - \alpha_t)\sqrt{\bar{\alpha}_{t-1}}}{1 - \bar{\alpha}_t} x_0 + \nu_t \hat{\mu}(z),                                                                                                                                                                                    \\
  \tilde{\beta}_t(z)       & = \frac{\frac{1-\alpha_t}{\alpha_t} \hat{\sigma}^2(z) \left(1-\bar{\alpha}_{t-1}\right) \hat{\sigma}^2(z)}{\frac{1-\alpha_t}{\alpha_t} \hat{\sigma}^2(z) + \left(1-\bar{\alpha}_{t-1}\right) \hat{\sigma}^2(z)}
  = \frac{(1 - \alpha_t)(1 - \bar{\alpha}_{t-1})}{1 - \bar{\alpha}_t} \hat{\sigma}^2(z).
\end{align*}

\subsection{Derivation of $\mathcal{J}_3$ under noise-prediction modeling} \label{apdx:J3_epsilon}
Using \eqref{eq:xt_x0_ep0_f} with $t = 1$, $\mathcal{J}_3$ can be rewritten as follows.
\begin{align*}
  \mathcal{J}_3 & = \mathbb{E}_{\hat{q}(z \mid x_0) q(x_{1} \mid x_0, z)}\left[\log p_\theta(x_0 \mid x_1, z) \right]                                                                                                                                                                                                                                                                             \\
                & = \mathbb{E}_{\hat{q}(z \mid x_0) q(x_{1} \mid x_0, z)}\left[ \log \mathcal{N}\left(x_0 \ \middle| \ \mu_\theta(x_1, 1, z), \sigma_1^2(z) I\right) \right]                                                                                                                                                                                                                      \\
                & = \mathbb{E}_{\hat{q}(z \mid x_0) q(x_{1} \mid x_0, z)}\left[ - \frac{1}{2 \sigma_1^2(z)} \left\|x_0 - \mu_\theta(x_1, 1, z)\right\|^2 \right] + C_3                                                                                                                                                                                                                            \\
                & = \mathbb{E}_{\hat{q}(z \mid x_0) q(x_{1} \mid x_0, z)}\left[ - \frac{1}{2 \sigma_1^2(z)} \left\|x_0 - \left(\frac{1}{\sqrt{\alpha_1}} x_1 - \frac{1 - \sqrt{\alpha_1}}{\sqrt{\alpha_1}} \hat{\mu}(z) - \frac{\left(1 - \alpha_1\right) \sqrt{\hat{\sigma}^2(z)}}{\sqrt{\left(1 - \bar{\alpha}_1 \right) \alpha_1}} \epsilon_\theta (x_1, 1, z) \right)\right\|^2 \right] + C_3 \\
                & = \mathbb{E}_{\hat{q}(z \mid x_0) \mathcal{N}(\epsilon_0 \mid 0, I)}\biggl[ - \frac{1}{2 \sigma_1^2(z)} \biggl\|x_0 - \biggl(\frac{1}{\sqrt{\alpha_1}} \left(\sqrt{\alpha}_1 x_0 + \left(1 - \sqrt{\alpha_1}\right) \hat{\mu}(z) + \sqrt{\left(1 - \alpha_1\right) \hat{\sigma}^2(z)} \epsilon_0\right) \notag                                                                  \\
                & \hspace{50mm} - \frac{1 - \sqrt{\alpha_1}}{\sqrt{\alpha_1}} \hat{\mu}(z) - \frac{\sqrt{\left(1 - \alpha_1\right) \hat{\sigma}^2(z)}}{\sqrt{\alpha_1}} \epsilon_\theta (x_1, 1, z) \biggr)\biggr\|^2 \biggr] + C_3                                                                                                                                                               \\
                & = \mathbb{E}_{\hat{q}(z \mid x_0) \mathcal{N}(\epsilon_0 \mid 0, I)}\left[ - \frac{1}{2 \sigma_1^2(z)} \left\|\frac{\sqrt{\left(1 - \alpha_1\right) \hat{\sigma}^2(z)} }{\sqrt{\alpha_1}}\left(\epsilon_0 - \epsilon_\theta (x_1, 1, z)\right) \right\|^2 \right] + C_3                                                                                                         \\
                & = \mathbb{E}_{\hat{q}(z \mid x_0) \mathcal{N}(\epsilon_0 \mid 0, I)}\left[- \frac{\left(1 - \alpha_1\right) \hat{\sigma}^2(z)}{2 \tilde{\beta}_1(z) \alpha_1} \left\|\epsilon_0 - \epsilon_\theta (x_1, 1, z) \right\|^2 \right] + C_3.
\end{align*}

\section{Additional experimental details and results on the Datasaurus-Grid datasets}

\subsection{VQVAE training details} \label{apdx:vqvae}
In experiments with the Datasaurus-Grid datasets, we set the VQVAE codebook size to $16$ and the embedding dimension to $2$. Codebook updates followed the EMA-based method introduced in VQ-VAE-2 \citep{razavi2019vqvae2}, and we additionally utilized the empirical usage frequencies of the learned clusters during sampling.
To stabilize training with variance estimation, the decoder’s variance parameter (i.e., $\hat{\sigma}$ in \eqref{eq:prior_x0_z}) was treated as a scalar independent of the latent variable $z$.
We impose a lower bound of $0.1$ on the estimated variance by using the softplus function \citep{dugas2000incorporating}.
The model was trained for $15000$ steps, with the variance fixed at $1.0$ for the first $10000$ steps and learned during the remaining $5000$ steps.

\subsection{Additional qualitative results of 2D generated samples} \label{apdx:2d_sample}
Fig.~\ref{fig:datasaurus-grid-diffusevae} presents the samples generated by DiffuseVAE, and
Fig.~\ref{fig:datasaurus-grid-rf} shows those produced by Rectified Flow trained on the Datasaurus-Grid datasets.

\begin{figure}[tbp]
  \centering
  \begin{minipage}[t]{0.39\linewidth}
    \centering
    \includegraphics[width=30mm,clip]{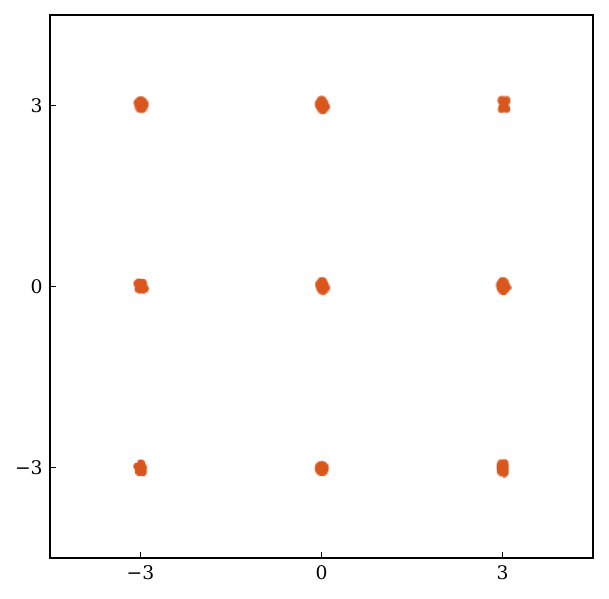}
    \includegraphics[width=30mm,clip]{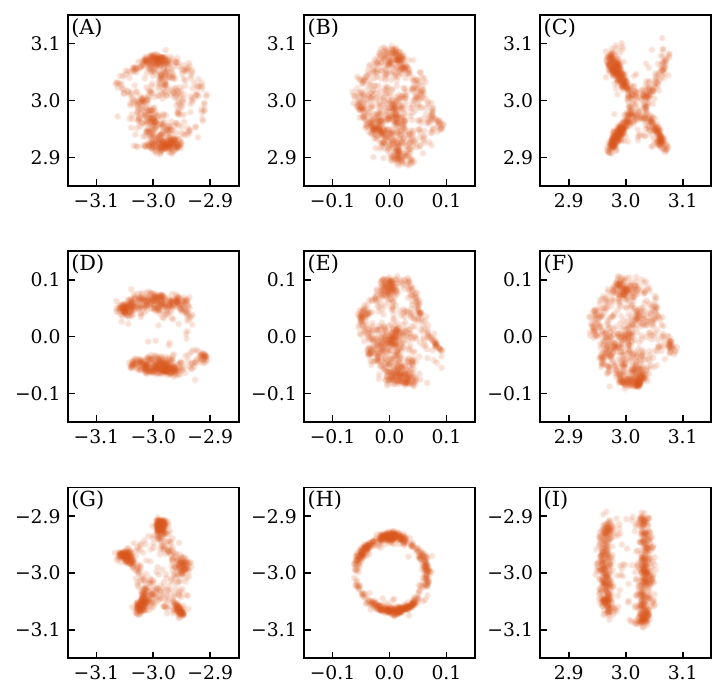}
    \subcaption{scale=0.1 (overview / zoomed)}
    \label{fig:datasaurus-grid-0.1-diffusevae}
  \end{minipage} \hfill
  \begin{minipage}[t]{0.2\linewidth}
    \centering
    \includegraphics[width=30mm,clip]{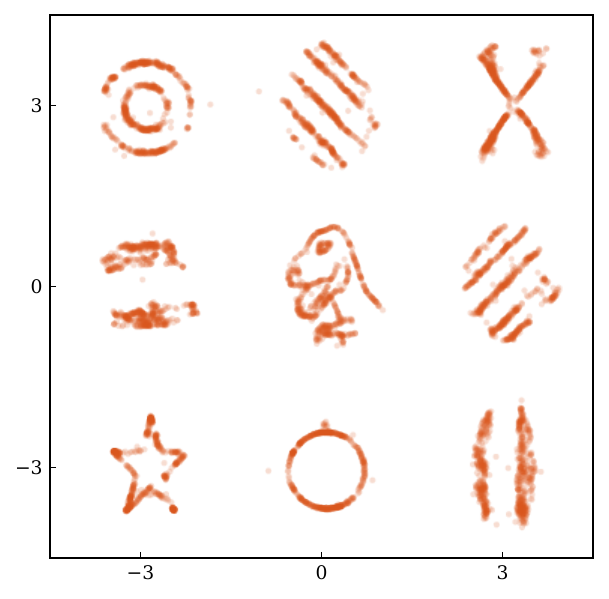}
    \subcaption{scale=1.0}
    \label{fig:datasaurus-grid-1.0-diffusevae}
  \end{minipage} \hfill
  \begin{minipage}[t]{0.39\linewidth}
    \centering
    \includegraphics[width=30mm,clip]{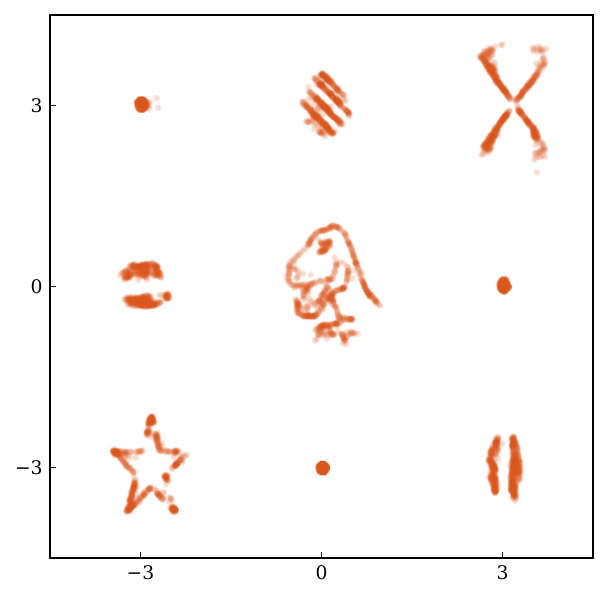}
    \includegraphics[width=30mm,clip]{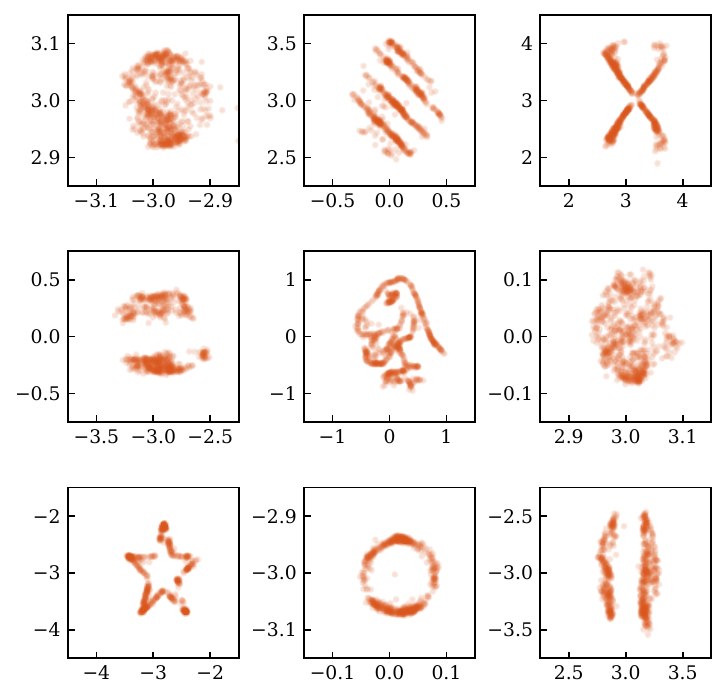}
    \subcaption{hetero-scale (overview / zoomed)}
    \label{fig:datasaurus-grid-heteroscale-diffusevae}
  \end{minipage}
  \caption{Samples generated by DiffuseVAE on the Datasaurus-Grid datasets: (a) scale=0.1, (b) scale=1.0, and (c) hetero-scale.}
  \label{fig:datasaurus-grid-diffusevae}
\end{figure}

\begin{figure}[tbp]
  \centering
  \begin{minipage}[t]{0.39\linewidth}
    \centering
    \includegraphics[width=30mm,clip]{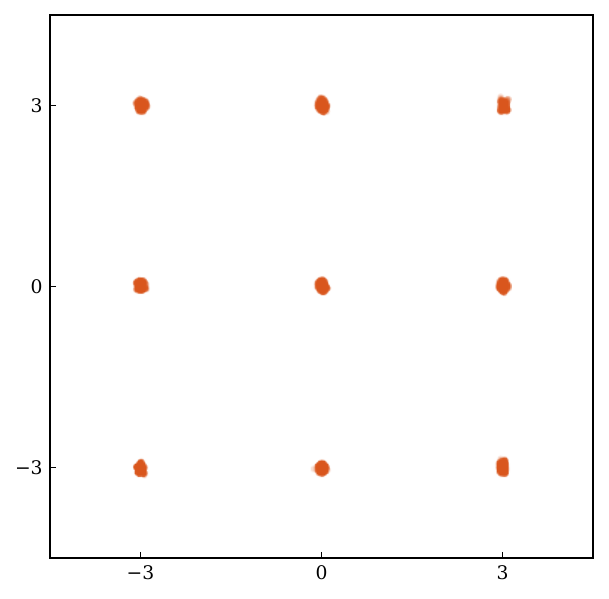}
    \includegraphics[width=30mm,clip]{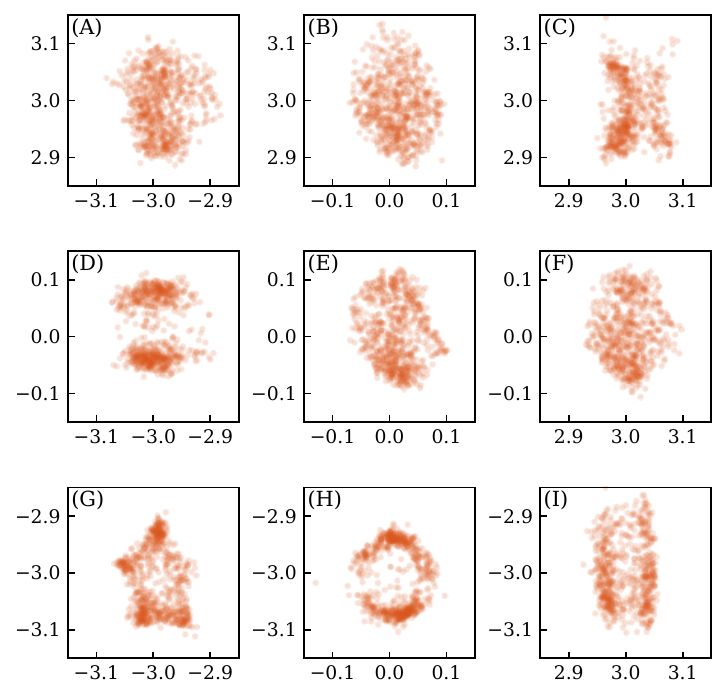}
    \subcaption{scale=0.1 (overview / zoomed)}
    \label{fig:datasaurus-grid-0.1-rf}
  \end{minipage} \hfill
  \begin{minipage}[t]{0.2\linewidth}
    \centering
    \includegraphics[width=30mm,clip]{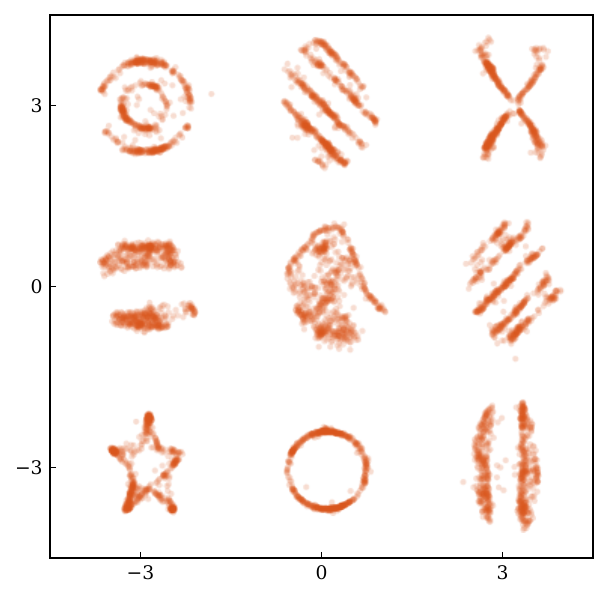}
    \subcaption{scale=1.0}
    \label{fig:datasaurus-grid-1.0-rf}
  \end{minipage} \hfill
  \begin{minipage}[t]{0.39\linewidth}
    \centering
    \includegraphics[width=30mm,clip]{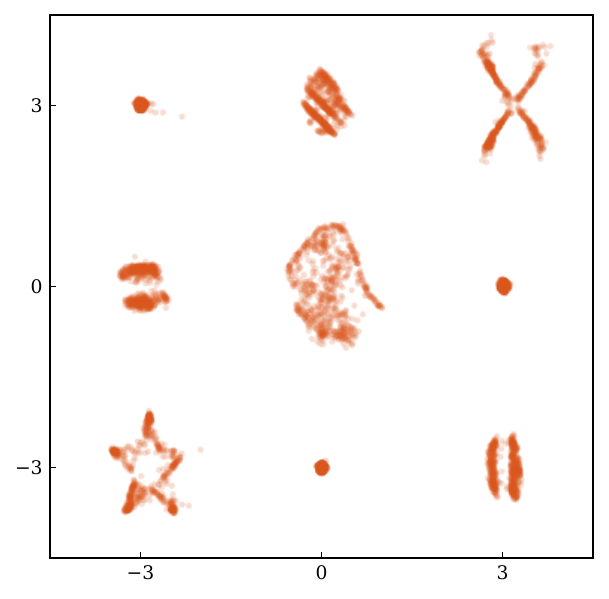}
    \includegraphics[width=30mm,clip]{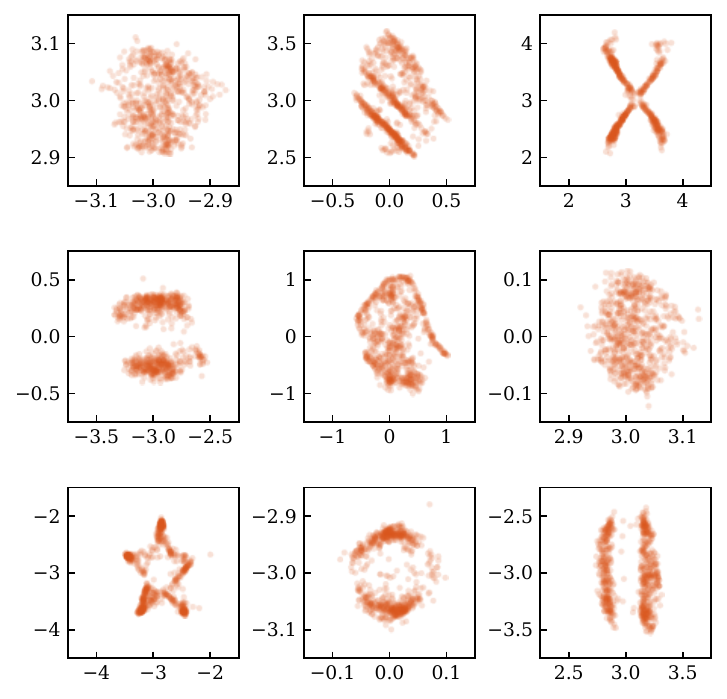}
    \subcaption{hetero-scale (overview / zoomed)}
    \label{fig:datasaurus-grid-heteroscale-rf}
  \end{minipage}
  \caption{Samples generated by Rectified Flow on the Datasaurus-Grid datasets: (a) scale=0.1, (b) scale=1.0, and (c) hetero-scale.}
  \label{fig:datasaurus-grid-rf}
\end{figure}

\subsection{Comparison with DDPM under different hyperparameter configurations} \label{apdx:ddpm_various_hparams}
Diffusion models involve a number of hyperparameters. To examine how such design choices affect performance,
we compared RPD with several DDPM variants obtained by modifying key hyperparameters.
Specifically, we prepared three DDPM models with the following configurations:
\begin{itemize}
  \item A model with the maximum number of diffusion steps set to $T=1000$.
        The same value was used during inference.
        In general, increasing $T$ is known to improve generation quality at the cost of higher computational overhead.
  \item A model in which the log-linear schedule for $\beta_t$ was modified by reducing \texttt{sigma\_min} to $0.001$
        (while keeping \texttt{sigma\_max} at $100$).
        Increasing the resolution in the region of small $\beta_t$ is expected to produce a more precise diffusion process near $x_0$,
        which in turn may improve the reconstruction of fine-scale structure.
  \item A model with increased network capacity, implemented by expanding the hidden-layer widths of the DNN to
        $128, 256, 512, 256,$ and $128$.
        Enhancing the expressive power of the DNN is generally expected to improve generative performance.
\end{itemize}

The comparison between these DDPM variants and RPD is presented in Fig.~\ref{fig:datasaurus-grid-plot-ddpm-config}.
The RPD configuration follows the one described in Section~\ref{sec:compared_models}.
All experiments were conducted using the Datasaurus-Grid (scale=0.1) dataset.

As shown in the figure, modifying the DDPM hyperparameters results in minor variations in 1WD and RW-1WD,
but none of the configurations yield substantial improvements.
RPD consistently maintains superior performance.
These results suggest that, for datasets such as Datasaurus-Grid (scale=0.1), where local structures occur at scales
much smaller than the global distribution, adjusting DDPM hyperparameters alone is insufficient to accurately capture
fine-grained local structure.

\begin{figure}[htb]
  \begin{center}
    \includegraphics[width=0.6\linewidth,clip]{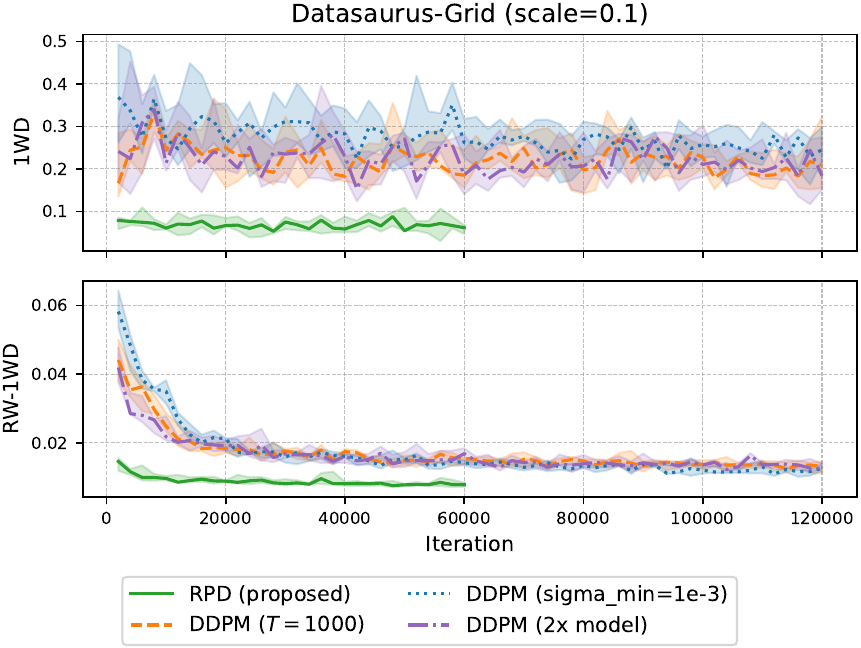}
    \caption{RPD vs.\ DDPM with different hyperparameter settings.}
    \label{fig:datasaurus-grid-plot-ddpm-config}
  \end{center}
\end{figure}

\bibliographystyle{plainnat}
\bibliography{references}

\end{document}